\documentclass[11pt]{article} 

\usepackage{amsmath,amsfonts,amsthm,mathrsfs, mathtools,amssymb}
\usepackage{hyperref}
\usepackage{hyperref}
\hypersetup{       
    unicode=false,          
    pdftoolbar=true,       
    pdfmenubar=true,      
    pdffitwindow=false,     
    pdfstartview={FitH},   
    pdfnewwindow=true,      
    colorlinks=true,      
    linkcolor=blue,          
    citecolor=blue,       
    filecolor=magenta,      
    urlcolor=blue,           
    breaklinks=true
}
\usepackage{xurl}
\usepackage[capitalize]{cleveref}
\usepackage{nicefrac}   
\usepackage{enumitem}
\usepackage[dvipsnames]{xcolor}
\usepackage{wrapfig}

\usepackage{comment}
\usepackage[rightcaption]{sidecap}
\usepackage{bm,bold-extra} 
\usepackage[innerleftmargin=0em, innerrightmargin=0em,skipabove=0.4em, innertopmargin=0em, innerbottommargin=0.4em,skipbelow=0.4em,outerlinecolor=white,linecolor=white]{mdframed} 

\usepackage{makecell}
\usepackage{booktabs} 
\usepackage{pifont}

\newcommand{\xmark}{\ding{55}}%
\usepackage{pifont}
\usepackage[dvipsnames,table]{xcolor}
\usepackage{authblk}
\usepackage{selectp}
\usepackage[normalem]{ulem}
\usepackage{scalerel,stackengine}
\usepackage{fullpage}
\usepackage{breqn}

\usepackage{tikz}
\usetikzlibrary{hobby}

\usepackage[toc,page,header]{appendix}
\usepackage{minitoc}

\usepackage[natbib=true,backend=biber,doi=false, isbn=false, eprint=false, url=false,date=year,maxbibnames=99,defernumbers=true,giveninits,style=authoryear,useprefix=true]{biblatex}
\AtEveryBibitem{\clearlist{language}}
\AtEveryBibitem{\clearlist{editor}}
\DeclareSortingScheme{noneyear}{
 \sort{\citeorder}
 \sort{\field{year}}
}

\newcommand{\F}{\mathscr{F}} 
\newcommand{\Fstar}{\F_{\star}} 
\newcommand{\fstar}{f_{\star}} 

\newcommand{\Freg}{\F^{\rho}} 
\newcommand{\Fregstar}{\Freg_{\star}} 
\newcommand{\balltwo}[1]{B(#1)} 
\newcommand{\Fr}{\F_r} 
\newcommand{\sumt}{\sum_{t=0}^{T-1}} 
\newcommand{\Sobo}{\mathscr{H}^{s}}
\newcommand{\Soboper}{\Sobo_{\text{per}}} 
\newcommand{\smallcube}{Q_L} 
\newcommand{\mediumcube}{Q_{2L}} 
\newcommand{\bigcube}{Q_{4L}} 
\newcommand{\Prob}{\mathbb{P}} 
\newcommand{\inputseq}{X} 
\newcommand{\jointProb}{\Prob_{\inputseq}} 
\newcommand{\X}{\Omega} 
\newcommand{\reg}[1]{\Psi(#1)}
\newcommand{\Elltwo}{\mathscr{L}^2} 
\newcommand{\N}[3]{\mathcal{N}_{#1}\left( #2, #3 \right)} 
\newcommand{\Eset}{\mathcal{E}} 
\newcommand{\MOC}[1]{\bm{\mathrm{M}}_T \left[ #1 \right] } 
\newcommand{\Alo}{\mathcal{A}} 
\newcommand{\Dataset}{\mathcal{D}} 
\newcommand{\R}{\mathbb{R}} 
\newcommand{\Z}{\mathbb{Z}_{\geq 0}} 
\newcommand{\Filt}{\mathcal{X}} 
\newcommand{\sigmaw}{\sigma_W^2} 
\newcommand{\NA}{\textbf{(\texttt{N.A})}} 
\newcommand{\NB}{\textbf{(\texttt{N.B})}} 
\newcommand{\NC}{\textbf{(\texttt{N.C})}} 
\newcommand{\D}{D^{\alpha}} 
\newcommand{\Der}{\mathscr{D}} 
\newcommand{\Cont}{\mathscr{C}} 
\newcommand{\HH}{\mathcal{H}} 
\newcommand{\fhat}{\hat{f}} 
\newcommand{\dx}{d_X} 
\newcommand{\dy}{d_Y} 
\newcommand{\innerprod}[3]{\left \langle #1,\, #2 \right \rangle_{#3}} 
\newcommand{\norm}[2]{\left \| #1 \right \|_{#2}} 
\newcommand{\E}[2]{\mathbb{E}_{#2} \left[ #1 \right]} 
\newcommand{\PeriodExt}[2]{E_{\text{per}}(#1)(#2)} 
\newcommand{\transker}[3]{\mathcal{K}_{#3}\left(#1,#2\right)} 
\newcommand{\Lebmeas}{\mu_{\lambda}} 
\newcommand{\pt}[2]{p_{#1}\left(#2 \right)} 
\newcommand{\kappal}{\underline{\kappa}} 
\newcommand{\kappau}{\overline{\kappa}} 
\renewcommand{\i}{\iota} 
\newcommand{\pdec}[1]{p_{#1}} 
\newcommand{\RSob}{\rho_f} 
\newcommand{\ExtOp}[2]{\tilde{E}(#1)(#2)}
\newcommand{\Ftmp}{\F_{\Lebmeas}} 
\newcommand{\Fe}{\F_{\epsilon}} 
\newcommand{\me}{m_{\epsilon}} 
\newcommand{\Ellinf}{\mathscr{L}^{\infty}} 
\newcommand{\p}{\mathfrak{p}} 
\newcommand{\xx}{\mathrm{x}} 
\newcommand{\tiRSob}{\widetilde{\RSob}} 
\newcommand{\tir}{\widetilde{r}} 
\newcommand{\eps}{\epsilon} 
\newcommand{\bigO}{\mathcal{O}} 
\newcommand{\Kl}{\underline{K}} 
\newcommand{\Ku}{\overline{K}} 
\newcommand{\Fku}{\F_{\Ku}} 
\newcommand{\Fk}{\widetilde{F}_{k}} 
\newcommand{\Cc}{C_c} 
\newcommand{\Cslow}{C_{\texttt{slow}}} 
\newcommand{\Cfast}{C_{\texttt{fast}}} 

\newcommand\polygon[2][]{
  \pgfmathsetmacro{\angle}{360/#2}
  \pgfmathsetmacro{\startangle}{-90 + \angle/2}
  \pgfmathsetmacro{\y}{cos(\angle/2)}
  \begin{scope}[#1]
    \foreach \i in {1,2,...,#2} {
      \pgfmathsetmacro{\x}{\startangle + \angle*\i}
      \draw[fill=ForestGreen,color=ForestGreen] (0, 0) -- (\x:1 cm) -- (\x + \angle/2:\y cm) -- cycle;
      \draw[fill=ForestGreen,color=ForestGreen] (0, 0) -- (\x + \angle/2:\y cm) -- (\x + \angle:1 cm) -- cycle;
    }
  \end{scope}
}

\newmdenv[backgroundcolor=blue!5]{repthm}

\newmdtheoremenv[%
ntheorem = true,backgroundcolor=blue!5
]{theorem}{Theorem}
\newmdtheoremenv[%
ntheorem = true,backgroundcolor=blue!5%
]{proposition}[theorem]{Proposition}
\newmdtheoremenv[%
ntheorem = true,backgroundcolor=blue!5%
]{lemma}[theorem]{Lemma}
\newmdtheoremenv[%
ntheorem = true,backgroundcolor=blue!5,innerbottommargin=0em%
]{corollary}[theorem]{Corollary}

\theoremstyle{remark}
\newmdtheoremenv[%
ntheorem = true,backgroundcolor=white%
]{problem}{Problem}
\newmdtheoremenv[%
ntheorem = true,backgroundcolor=red!5%
]{definition}{Definition}
\newmdtheoremenv[%
ntheorem = true,backgroundcolor=yellow!20%
]{assumption}{Assumption}
\newmdtheoremenv[ntheorem = false,backgroundcolor=gray!20]{remark}{Remark}

\numberwithin{theorem}{section}

\crefname{theorem}{Theorem}{Theorems}
\crefname{proposition}{Proposition}{Propositions}
\crefname{corollary}{Corollary}{Corollaries}
\crefname{lemma}{Lemma}{Lemmas}
\crefname{assumption}{Assumption}{Assumptions}
\crefname{equation}{Equation}{Equations}
\crefname{section}{Section}{Sections}
\crefname{appendix}{Appendix}{Appendices}
\crefname{figure}{Figure}{Figures}

\allowdisplaybreaks

\DeclareMathOperator*{\argmin}{arg\,min}

\numberwithin{definition}{section}
\numberwithin{equation}{section}
\numberwithin{theorem}{section}

\title{\LARGE \bf
Physics-informed learning under mixing:\\ How physical knowledge speeds up learning 
}

\author[1]{Anna Scampicchio*}
\author[2]{Leonardo F.~Toso*}
\author[1]{Rahel Rickenbach}
\author[2]{James Anderson}
\author[1]{Melanie N.~Zeilinger}
\affil[1]{Institute for Dynamic Systems and Control, ETH Z\"urich}
\affil[2]{Department of Electrical Engineering, Columbia University}
\date{September 2025}
\begin{document}

\doparttoc 
\faketableofcontents

\maketitle
\allowdisplaybreaks
\begin{abstract}
A major challenge in physics-informed machine learning is to understand how the incorporation of prior domain knowledge affects learning rates when data are dependent. Focusing on empirical risk minimization with physics-informed regularization, we derive complexity-dependent bounds on the excess risk in probability and in expectation. We prove that, when the physical prior information is aligned, the learning rate improves from the (slow) Sobolev minimax rate to the (fast) optimal i.i.d.~one without any sample-size deflation due to data dependence.
\end{abstract}

\footnote{\hspace{-0.6cm}*A.~Scampicchio and L.~F.~Toso share first-authorship. \\ Correspondence to: \{ascampicc@ethz.ch, lt2879@columbia.edu\}.}

\section{Introduction}

Physics-informed machine learning encompasses a wide taxonomy of approaches that combine physical knowledge and learning algorithms to address two main tasks: (i) enhancing physical models (given, e.g., by systems of partial differential equations) through data-driven methods to improve their accuracy and numerical solvability; (ii) improve the learning algorithms' performance by including physical information, e.g., as additional constraint~\citep{karniadakis_physics-informed_2021,meng_when_2025}. Focusing on the second class of methods, surveyed in~\cite{rai_driven_2020,von_rueden_informed_2023}, the resulting approaches turn out to be practically effective in terms of data efficiency, generalization capability and interpretability, especially in view of downstream tasks such as safe learning-based control~\citep{nghiem_physics-informed_2023,drgona_safe_2025}.  However, theoretically quantifying the beneficial impact of physical information into learning algorithms is challenging and still an active research question
~(see \cite{von_rueden_how_2023} and references therein). \\

In this paper, we tackle this question by considering a statistical learning set-up and focusing on regularized empirical risk minimization problems of the following form:
\vspace{-0.3em}
\begin{equation}
\begin{tikzpicture}

\node (fhatEqArgmin) {$\fhat = \argmin$};

\node[xshift = 1em,yshift=-1.1em, align=center,scale=0.7] (fspec) at (fhatEqArgmin)  {$f \in $ ball in \\ \vspace{-1.1em} Sobolev space};

\node[align=center,draw=black,rounded corners,xshift=7.2em,yshift=-0.2em,execute at begin node = \linespread{0.8}\selectfont] (dataFit) at (fhatEqArgmin) {data-fit \\ squared loss$(f)$};

\node[xshift=4.3em] (plus) at (dataFit) {$+$};

\node[xshift=1.2em] (lambda) at (plus) {$\lambda_T$};

\node[align=center,draw=black,rounded corners,xshift=4.7em,execute at begin node = \linespread{0.8}\selectfont] (reg) at (lambda) {physics-informed\\ regularizer$(f)$};

\node[xshift=4.2em] (comma) at (reg) {,};

\end{tikzpicture}\label{eq:abstract_RERM}
\end{equation}
where data entering the fit term are \emph{dependent},  derived from observations of a ground-truth nonlinear dynamical system $X_{t+1} = \fstar(X_t) + W_t$, with $W_t$ being a sub-Gaussian noise martingale  difference sequence.  
The regularizer in~\eqref{eq:abstract_RERM} encodes the information that the true function to be estimated, $\fstar$, approximately satisfies a known partial differential equation induced by a linear operator $\Der$ --- i.e., we have that the regularizer takes the form $\norm{\Der(f)}{\Elltwo}^2$, and we say that \emph{knowledge alignment} occurs if it holds that  $\norm{\Der(\fstar)}{\Elltwo}^2 \simeq 0$. \\

The main results of this paper are \emph{complexity-dependent} bounds  --- i.e., bounds that depend on $||\Der(\fstar)||_{\Elltwo}$~\citep{lecue_regularization_2017} --- for the \emph{excess risk} $||\fhat - \fstar||^2_{\Elltwo}$ in physics-informed and non-parametric learning with dependent data. Informally, our results (both in high probability and expectation) will look like this:
\begin{repthm}
    {\bfseries Theorem (Informal).} For a suitable choice of the regularization parameter $\lambda_T$, for a sufficiently large number of samples $T$, and letting $d < 1$ be the \emph{Sobolev minimax rate}~\citep{ibragimov_statistical_1981,nussbaum_minimax_2006}, it holds that
    \begin{equation}
    \text{(Excess risk)}  \quad ||\fhat - \fstar||_{\Elltwo}^2 \leq \Cslow \frac{\norm{\Der(\fstar)}{\Elltwo}^{\text{some power}}}{T^{d}} + \Cfast \frac{\text{noise level}}{T}.\notag
\end{equation}
\end{repthm}
Thanks to this we show that, under knowledge alignment, the regularized estimate $\fhat$ converges to the true, unknown function $\fstar$ at the i.i.d.~rate of $\bigO(1/T)$: in other words, it behaves like classic optimal rates for i.i.d.~learning \emph{even if the data are dependent} after a suitable burn-in time.\\

The remainder of the paper unfolds as follows: \cref{sec:setup} provides the set-up of the learning problem, introducing the \emph{weighted, vector-valued} function spaces that will be used throughout the paper. Next, the learning problem is stated in~\cref{sec:problem_statement}, and in~\cref{sec:bounds} we provide the general statement for the excess risk bounds, both in probability and in expectation. Our analysis culminates in~\cref{sec:convergence_rates}, where we prove how knowledge alignment leads to optimal i.i.d.~rates even if data are dependent. We discuss our results in juxtaposition with related works in~\cref{sec:discussion}, and present some concluding remarks in~\cref{sec:conclusions}.

\section{Problem set-up}\label{sec:setup}
This section collects preliminary concepts, defining the probability set-up of the data-generation mechanism (\cref{sec:setupXt}) and the involved weighted, vector-valued function spaces (\cref{sec:function_spaces}).

\subsection{Input domain and trajectory distribution}\label{sec:setupXt} Let $\X \subseteq [-L,L]^{\dx} \subset \R^{\dx}$ be the input domain whose boundary is locally Lipschitz~\citep[Definition~4.9]{adams_sobolev_2003}. Suppose we have a horizon length $T$, the input trajectories denoted by $\inputseq \doteq (X_0,\, X_1,\,\cdots,\allowbreak \, X_{T-1})$ belong to the metric space \mbox{$(\X^T, \left\{ \Filt _t\right\}_{t=0}^{T-1}, \jointProb)$}, where \mbox{$\X^T \doteq {\bigtimes}_{t=0}^{T-1} \X$} is the Cartesian product of the single-component input domains~$\X$; \mbox{$\left\{ \Filt _t\right\}_{t=0}^{T-1}$} is the \textit{filtration} given by a sequence of increasing~$\sigma$-algebras \mbox{$\Filt_{t+1} \subset \Filt_t$} with respect to which~$X$ is \textit{adapted}~\citep[Chapter~II.45]{rogers_diffusions_2000}; and~$\jointProb$ is the joint probability distribution of the input trajectory. As detailed in~\cref{sec:mut}, there exists a probability distribution associated with every component of $\inputseq$ --- we denote it by $\mu_t$ for each $t=0,\cdots,T-1$, and we mostly work with a \textit{known} initial distribution $\mu_0$ for $X_0$ (typically, a Dirac measure centered at the observed initial state $X_0$). Overall, we make use of the following:
\begin{assumption}\label{ass:densities}
    Let $\Lebmeas$ be the Lebesgue measure defined on $\X \subset \R^{\dx}$. For all $t=0,\cdots, T-1$, each measure $\mu_t\colon \Filt_t \to \R_{\geq 0}$ is assumed to admit a density with respect to $\Lebmeas$. We denote such density by $p_t(\cdot)$, and we assume that  there exist $0 < \kappal < \kappau < \infty$ such that, for all $t=0,\cdots, T-1$, $\kappal \leq \pt{t}{\cdot} \leq \kappau$. 
\end{assumption}
    Note that~\cref{ass:densities} accounts for many cases of practical relevance, such as the uniform, the truncated Gaussian and the beta 
    distributions~\citep{krishnamoorthy_handbook_2016}.

\subsection{Spaces of functions} \label{sec:function_spaces}
\paragraph{Space of square-integrable functions $\Elltwo$.} We will focus on the Hilbert space~$\Elltwo(\X^T,\jointProb;\R^{\dy})$ of vector-valued, square-integrable functions that consist of multiple evaluations of a function $f\colon \X \to \R^{\dy}$ along the input trajectory $\inputseq$. Such a space allows us to consider the  trajectory $\inputseq$ and is endowed with the inner product defined as follows: given $f,g \colon \X \to \R^{\dy}$, we have
\begin{align}\label{eq:innerprod}
    \innerprod{f}{g}{\Elltwo(\X^T,\jointProb;\R^{\dy})} &\doteq \frac{1}{T}\sumt \E{\innerprod{f(X_t)}{g(X_t)}{2}}{\jointProb} = \frac{1}{T}\sumt \int_{\X^T} \innerprod{f(X_t)}{g(X_t)}{2}d\jointProb \notag\\
    & =\frac{1}{T}\sumt \int_{\X} \innerprod{f(X_t)}{g(X_t)}{2} \mu_t(dX_t),\end{align}
where $\innerprod{\cdot}{\cdot}{2}$ is the standard inner product defined in the Euclidean space $\R^{\dy}$, and $\mu_t$ is the probability measure of the $t$-th component of $\inputseq$ introduced in~\cref{sec:setupXt}. The inner product~\eqref{eq:innerprod} induces the trajectory norm $\norm{f}{\Elltwo(\X^T,\jointProb;\R^{\dy})}$ such that $\norm{f}{\Elltwo(\X^T,\jointProb;\R^{\dy})}^2 = \innerprod{f}{f}{\Elltwo(\X^T,\jointProb;\R^{\dy})}$. Furthermore, it follows by construction that one can write $\norm{f}{\Elltwo(\X^T,\jointProb;\R^{\dy})}^2 = \frac{1}{T}\sumt \mathbb{E}_{\jointProb}[\norm{f(X_t)}{2}^2]$. 
Note in addition that, thanks to the separability of $\R^{\dy}$, the vector-valued space $\Elltwo(\X^T,\jointProb;\R^{\dy}) = \{ f\colon \X \to \R^{\dy} \mid \norm{f}{\Elltwo(\X^T,\jointProb;\R^{\dy})} < \infty \}$ can be written as the direct sum 
$\bigoplus_{i=1}^{\dy} \Elltwo(\X^T,\jointProb;\R)$~\citep[Chapter~I.6]{conway_course_2007}: indeed, following~\eqref{eq:innerprod}, we can write 
\begin{align}
  \norm{f}{\Elltwo(\X^T,\jointProb;\R^{\dy})}^2 = \sum_{i=1}^{\dy}\frac{1}{T} \sumt \E{f_i(X_t)^2}{\jointProb} = \sum_{i=1}^{\dy} \norm{f_i}{\Elltwo(\X^T,\jointProb;\R)}^2. \notag
\end{align}

\paragraph{General $\mathscr{L}^p$ spaces.}
In general, one can define the space $\mathscr{L}^p(\X^T,\jointProb;\R^{\dy})$ for any $p \in \Z$ endowed with the norm $\norm{f}{\mathscr{L}^p(\X^T,\jointProb;\R^{\dy})}^p = \frac{1}{T}\sumt \E{\norm{f(X_t)}{2}^p}{\jointProb}$. Of particular interest will be the Banach space of bounded functions $\Ellinf(\X^T;\R^{\dy})$ equipped with the norm \begin{equation}
\norm{f}{\Ellinf(\X^T;\R^{\dy})} \allowbreak \doteq \sup_{x \in \X} \norm{f(x)}{2}.\notag
\end{equation} 

\paragraph{Sobolev space $\Sobo$.} Another fundamental function space derived from $\Elltwo(\X^T,\jointProb;\R^{\dy})$ is
the multi-output, weighted \textit{Sobolev space} of order $s \in \Z$, which is defined as follows:
\begin{align}\label{eq:SobolevSpace}
\Sobo(\X^T,\jointProb;\R^{\dy}) \doteq \left\{ f \in \Elltwo(\X^T,\jointProb;\R^{\dy}) \mid \norm{f}{\Sobo(\X^T,\jointProb;\R^{\dy})} < \infty \right\}, \notag
\end{align}
where the norm is induced by the inner product
\begin{align}
    \innerprod{f}{g}{\Sobo(\X,\jointProb;\R^{\dy})} \doteq \sum_{|\alpha|\leq s} \innerprod{\D f}{\D g}{\Elltwo(\X^T,\jointProb;\R^{\dy})}, \notag
\end{align}
with $\D f$ being the differential given by the multi-index $\alpha \doteq (\alpha_1, \cdots, \alpha_{\dx})$ of non-negative integers with order $|\alpha| \doteq \sum_{i=1}^{\dx} \alpha_i$, i.e., $\D  \doteq \nicefrac{\partial^{|\alpha|}f}{\partial x_1^{\alpha_1}\cdots \partial x_{\dx}^{\alpha_{\dx}}}$. Regarding the order of the Sobolev spaces we will consider, we will rely on the following:
\begin{assumption}\label{ass:SobolevOrder}
    The order $s$ of $\Sobo(\X^T,\jointProb;\R^{\dy})$ is a non-negative integer that satisfies $s \geq 2\dx$. 
\end{assumption}

Finally, note that also the space $\Sobo(\X^T,\jointProb;\R^{\dy})$ admits the representation as the direct sum $\bigoplus_{i=1}^{\dy} \Sobo(\X^T,\jointProb;\R)$ thanks to the separability of $\R^{\dy}$. This allows us to extend key results of scalar Sobolev spaces to our vector-valued ones, as detailed in~\cref{sec:sobolevProperties}. In particular, we show that the Sobolev Imbedding Theorem~\citep[Theorem~4.12]{adams_sobolev_2003} holds in our set-up, which will provide the necessary structure for the hypothesis space involved in the learning problem.

\section{Problem statement}\label{sec:problem_statement}
\paragraph{Measurement model.} Assume to collect $T$ data points, $\Dataset \doteq \{X_t, Y_t\}_{t=0}^{T-1}$, generated according to the measurement model
\begin{equation}
    Y_t \doteq X_{t+1} = \fstar(X_t) + W_t, \label{eq:measurement_model}
\end{equation}
where the noise sequence satisfies the following:
\begin{assumption}\label{ass:noise}
    The additive noise $\{W_t\}_{t \in \Z}$ is a martingale difference sequence with respect to the filtration $\{\Filt_t\}_{t\in\Z}$: thus, $\E{W_t | \Filt_{t-1}}{W_t} =0$ for all $t=0,\cdots,T-1$. Moreover, each $W_t$ is also assumed to be $\sigmaw$-conditionally sub-Gaussian given $\Filt_{t-1}$: i.e., it holds that, for every $\xi \in \R$ and every $u$ in the unit sphere in $(\R^{\dy}, \norm{\cdot}{2})$, \begin{equation}\label{eq:sub-Gaussian}
        \mathbb{E}\left[\exp\left\{\xi \innerprod{W_t}{u}{2}\right\} \mid \Filt_{t-1} \right]  \leq \exp \left\{ \frac{\xi^2 \sigmaw}{2} \right \}.
    \end{equation}
\end{assumption}

\paragraph{The learning problem.} 
In general, the learning problem can be stated as that of minimizing the \emph{excess risk} $||\fhat-\fstar||_{\Elltwo(\X^T,\jointProb; \R^{\dy})}^2$, searching for the estimate $\fhat$ within a chosen \emph{hypothesis space} $\F$ (which we specify later). However, since the underlying probability measures are unknown, the amount of data in $\Dataset$ is finite and the hypothesis space $\F$ might be large, the estimate $\fhat$ is typically computed through \emph{(regularized) empirical risk minimization}:

\begin{align}\label{eq:RERM}
    \fhat     \doteq \argmin_{f\in\F} \frac{1}{T}\sumt   \norm{Y_t - f(X_t)}{2}^2 + \lambda_T  \reg{f}. 
\end{align}
\paragraph{Focus on the physics-informed regularizer.} In the set-up of our interest, the regularizer $\reg{\cdot} \colon \F \to \R_{\geq 0}$ encodes available prior physical information on the ``true" function $\fstar$ --- in other words, $\reg{f}$ penalizes the physical inconsistency of $f$ with respect to the prior on $\fstar$. Such physical information is conveyed by the fact that $\fstar$ is assumed to approximately satisfy a known partial differential equation given by the linear operator \mbox{$\Der \colon \Sobo(\X,\Lebmeas;\R^{\dy}) \to \Elltwo(\X,\Lebmeas;\R^{\dy})$}. 
Such an operator is defined component-wise as
\begin{align}\label{eq:diff_op}
    [\Der(f)]_i \doteq \sum_{|\alpha|\leq s}  \pdec{i,\alpha} \D f_i \; \text{for all } i=1,\cdots,\dy,
\end{align}
where each $\pdec{i,\alpha} \colon \X \to \R$ is a bounded function --- therefore, if we denote by $p$ the collection of all $\pdec{i,\alpha}$, then we have that $\norm{p}{\infty}$ is finite. To describe the regularity of the differential operator in~\eqref{eq:diff_op}, we make use of the following:
\begin{assumption}\label{ass:elliptic}

The differential operator $\Der(f)$ is \emph{elliptic} --- that is, for all $i=1,\cdots,\dy$ and any $\xi \in \R^{\dx}\setminus \lbrace 0 \rbrace$, it holds that $\sum_{|\alpha|=s} p_{i,\alpha}\xi_1^{\alpha_1}\cdots \xi_{\dx}^{\alpha_{\dx}} \neq 0$.
\end{assumption}
Elliptic partial differential equations abound in practical applications, as they can be seen as generalizations of the Laplace and Poisson operators~\citep[Chapter 6]{evans_partial_2010}. The differential operator $\Der$ enters the definition of the regularizer in~\eqref{eq:RERM}, where we have
\begin{equation}\label{eq:pderegu}
    \reg{f} \doteq \norm{\Der(f)}{\Elltwo(\X^T,\jointProb; \R^{\dy})}^2,
\end{equation}
which is a 2-proper regularizer~\citep[Assumption~1.1]{lecue_regularization_2017} --- see~\cref{sec:regularizer_properties} for the definition and further insights.
\paragraph{Hypothesis space.} Let us now focus on the hypothesis space $\F$. We consider it as the ball of radius $\RSob$ in the Sobolev space, i.e.,
\begin{equation}
    \F \doteq \left\{f \in \Sobo(\X^T, \jointProb; \R^{\dy} ) \mid \|f\|_{\Sobo(\X^T, \jointProb; \R^{\dy} )} \leq \rho_f  \right\}.\label{eq:hypothesisSpace}
\end{equation}

    Alternatively, as pointed out in~\cite[Theorem~8.21]{cucker_learning_2007}, one could write the cost in~\eqref{eq:RERM} as \mbox{$\frac{1}{T}\sumt   (Y_t - f(X_t))^2 + \tilde{\lambda}_T \|f\|_{\Sobo(\X^T,\jointProb;\R^{\dy})}^2 + \lambda_T \reg{f}$}, and the minimization would be performed for $f \in \Sobo(\X^T,\jointProb;\R^{\dy})$, thanks to the equivalence yielding $\RSob = \RSob(\tilde{\lambda}_T)$.\\ 
    In this paper, we will rely on the following:
    \begin{assumption}\label{ass:containment}
    The hypothesis space $\F$ contains the unknown function to be estimated, $\fstar$.    
    \end{assumption}
    The case in which such an assumption is violated is dealt with in the literature on  \emph{approximation theory} -- see, e.g.,~\cite{cucker_mathematical_2002,cucker_learning_2007}; however, these discussions are beyond the scope of this paper. 

Additionally, we will also consider the \emph{effective hypothesis space} induced by the regularizer, namely
\begin{equation}\label{eq:effectiveHypothesis}
    \Freg = \left\{ f\in \F \, \big\vert\, \reg{f-\fstar} \leq \rho \right\}.
\end{equation}
For a visualization of these hypothesis spaces, please refer to~\cref{fig:spaces_visualization}.
Finally, we will sometimes simplify notation by considering the shifted hypothesis space $\HH_{\star} \doteq \HH - \fstar = \{ f - \fstar \mid f \in \HH \}$, with $\HH$ being for instance $\F$ or $\Freg$.

\paragraph{Modelling sample dependence in trajectories.} Finally, we assume regularity in the trajectory $\inputseq$ given by the following one-sided exponential inequality~\citep{samson_concentration_2000}:
\begin{assumption}\label{ass:S-persistence}
    The trajectory $\inputseq$ governed by the law $\jointProb$ in the hypothesis class $\F$ is $S$-persistent for some $S \in [1,\infty)$. Specifically, for every $\xi \geq 0$ and every $f\in \F$, we have that 
    \begin{align}
    \mathbb{E}\left[\exp\left(-\xi \sumt   \norm{f(X_t)}{2}^2 \right) \right] \leq \exp\left( -\xi \sumt   \mathbb{E}\left[\norm{f(X_t)}{2}^2 \right] + \frac{\xi^2 S}{2}\sumt   \mathbb{E}\left[ \norm{f(X_t)}{2}^4\right] \right). \notag
\end{align}
\end{assumption}
Typically, $S$ is expressed in terms of the \emph{dependence matrix} of $\inputseq$ (see~\cref{sec:Spersistence} for its definition), and such a parameter attains higher values the more dependent $X_t$ is on its past. In general, $S$ might depend on $T$; however, in this paper we will focus on the case in which $S$ is a constant: as pointed out in~\cite[Section~2]{samson_concentration_2000}, this is a rather weak condition satisfied by a large class of 
Markov chains and of $\phi$-mixing processes --- see~\cref{sec:Spersistence} for more details.

\paragraph{Contribution.} 
Our results demonstrate that the physics-informed regularization in the 
empirical risk minimization problem~\eqref{eq:RERM} can speed-up the learning even in presence of dependent data. In particular, we derive complexity-dependent bounds for the excess risk \mbox{$||\fhat - \fstar||_{\Elltwo(\X^T,\jointProb; \R^{\dy})}^2$}, both in probability and in expectation, for learning under mixing, and prove that the rate of the excess risk matches the one from i.i.d~learning in presence of knowledge alignment. Therefore, our results theoretically quantify the beneficial impact of physical knowledge in learning algorithms, even in the challenging scenario of learning with dependent data.

\section{Error bounds}\label{sec:bounds}
We now present the bounds for the excess risk, both in probability and in expectation. We start in~\cref{sec:ideaBounds} by conveying the underlying ideas that lead to those results, and then provide the result in probability (\cref{sec:bound_prob}) and in expectation (\cref{sec:bound_exp}). These results will be further analyzed in~\cref{sec:convergence_rates} to obtain our main claims on the convergence rate of learning with physics-informed regularization. Before proceeding, we emphasize that the excess risk is a random quantity depending on the distribution of the input sequence $\inputseq$ and of the noises $\{W_t\}_{t=0}^{T-1}$: therefore, often we will simply write $\Prob$ and $\mathbb{E}$ instead of $\Prob_{\jointProb,W}$ and $\mathbb{E}_{\jointProb,W}$ to streamline notation.

\begin{SCfigure}[2.3]
    \centering
    \begin{tikzpicture} 
\begin{scope}[transform shape, scale=0.43,xshift=-80em]

\draw [line width = 2pt] (1.5,3) circle [radius=4.5];
\node at (1.5,3) {$\bullet$};
\node at (6,5.5) {\huge $\F$};
\draw[draw=gray] (1.5,3) -- (5.5,1);
\node at (3.5,2.5) {\color{gray} \huge $\RSob$};

\begin{scope}[transform shape,scale=1.8,yshift=2.5cm,xshift=-0.5cm]
\polygon[xshift=1em]{5};
\end{scope}
\node (fstar) at (-0.25,4.5) {\huge $\bullet$};
\node at (-0.3,4) {\huge $\fstar$};
\node at (-0.2,2.7) {\huge \color{ForestGreen}$\Freg$};

\begin{scope}[transform shape, scale=0.9,xshift=0.15cm,yshift=1.6cm]
\draw[line width = 3pt, color=magenta] (-2,2) rectangle (1,5);
\node at (1.3,5.4) {\huge \color{magenta} $\partial B(r)$};
\end{scope}

 \end{scope}
\end{tikzpicture}
    \caption{\small Visualization of the involved hypothesis spaces. Note that the set $\partial B(r) = \{f \in \Fstar \mid \norm{f}{\Elltwo(\X^T,\jointProb;\R^{\dy})}^2=r^2 \}$  introduced in~\cref{sec:ideaBounds} is represented as a square to highlight the fact that the norm therein involved is different to the one defining $\F$~\eqref{eq:hypothesisSpace}. Similarly, we represented $\Freg$~\eqref{eq:effectiveHypothesis} as a convex set that is not necessarily a ball in the Sobolev norm.}
\label{fig:spaces_visualization}
\end{SCfigure}
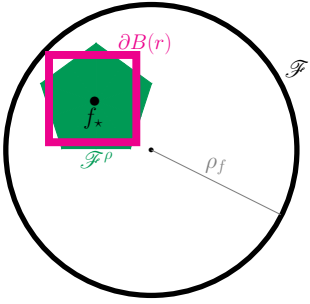

\subsection{The idea}\label{sec:ideaBounds}
The main idea consists of identifying an event according to which, with high probability and for some parameter $\theta$,
\begin{equation}\label{eq:highProbEvent}
    \norm{f-\fstar}{\Elltwo(\X^T,\jointProb;\R^{\dy})}^2 \leq  \frac{\theta}{T}\sumt \norm{f(X_t)-\fstar(X_t)}{2}^2. 
\end{equation}
This kind of one-sided concentration inequality was studied for the i.i.d.~setting in~\cite{mendelson_learning_2014}, to which we defer for a thorough discussion. The proof that~\eqref{eq:highProbEvent} holds with high probability in the i.i.d.~case is given in~\cite{mendelson_learning_2014} thanks to the \emph{small-ball condition}, which is a rather weak assumption from a statistical point of view: see the discussion after Assumption 1.2 in~\cite{lecue_regularization_2017}, together with its interpretation in terms of identifiability. In our data-dependent setting, the small-ball condition will be imposed by $(C,\alpha)$-hypercontractivity with $\alpha=2$ (see~\cref{sec:hypercontractivity}), and we show that it holds in the set \mbox{$\partial B(r) \doteq \{ f \in \F \,\mid\, \norm{f-\fstar}{\Elltwo(\X^T, \jointProb; \R^{\dy})}^2 = r^2 \}$} for any fixed $r>0$. Therefore, the probability level of the event in~\eqref{eq:highProbEvent} will be controlled by the radius $r$. We present a visualization of $B(r)$, together with all the hypothesis spaces, in~\cref{fig:spaces_visualization}.\\ 
Crucially, inequality~\eqref{eq:highProbEvent} allows us to shift the analysis of the excess risk to that of its empirical version. The next step consists then in upper-bounding the latter (i.e., the right-hand side in~\eqref{eq:highProbEvent}) by the \emph{martingale offset complexity} of the effective hypothesis space $\MOC{\Fregstar}$. In particular, for every $f \in \Fregstar$ (i.e., $f = f^{\prime} - \fstar$ for some $f^{\prime} \in \Freg$), one has that
\begin{equation}\label{eq:MOC}
    \frac{1}{T}\sumt \norm{f(X_t)}{2}^2 \leq \sup_{f \in \Fregstar} \frac{1}{T} \sumt 4\innerprod{W_t}{f(X_t)}{2} - \norm{f(X_t)}{2}^2 \doteq \MOC{\Fregstar}.
\end{equation}
We defer to~\cref{lemma:liang} for a derivation of such an inequality. Along the lines of~\cite{liang_learning_2015}, we would like to stress that the term $\norm{f(X_t)}{2}^2$ in the right-hand side introduces a self-normalizing effect that compensates the fluctuations of the term $\innerprod{W_t}{f(X_t)}{2}$. This fact is key in making the martingale offset complexity \emph{not depend on mixing}, as discussed in~\cref{sec:convergence_rates}. One can provide bounds in probability and in expectation for the martingale offset complexity (see~\cref{sec:MOC}), and these will play a key role in the excess risk bounds that we present in the remainder of the section and further discuss in~\cref{sec:convergence_rates}. \\
Before presenting the aforementioned bounds, let us formally introduce the \emph{lower isometry event}, which is the complement of~\eqref{eq:highProbEvent}, whose probability we bound in~\cref{sec:lower_isometry_bound}:
\begin{align}
    \Alo_r \doteq \sup_{f \in \Fregstar \setminus B(r)} \left\{\frac{1}{T}\sumt   \norm{f(X_t)}{2}^2 - \frac{1}{\theta}\norm{f}{\Elltwo(\X^T,\jointProb; \R^{\dy})}^2  \leq 0\right\}.\notag
    \end{align}

\subsection{Result in probability}\label{sec:bound_prob}

\begin{theorem}\label{thm:main_probab}
    Let~\cref{ass:densities,ass:SobolevOrder,ass:noise,ass:containment,ass:S-persistence} hold. Consider a parameter $\theta > 8$, and let $\fhat$ be the solution of the estimation problem~\eqref{eq:RERM} with $\lambda_T > 0$, and let the radius $\rho$ defining the effective hypothesis class $\Freg$ be such that $\rho \geq 10 \reg{\fstar}$. Then, on the event
    \begin{equation}
       \Alo_r^{\complement} \cap \left\{ \lambda_T \geq \frac{40}{3\rho} \MOC{\Freg}\right\} \notag
    \end{equation}
we have that
\begin{equation}
    \norm{\fhat - \fstar}{\Elltwo(\X^T,\jointProb;\R^{\dy})}^2 \leq \theta \MOC{\Freg} + 2\lambda_T \reg{\fstar} + r^2. \label{eq:ER_bound_prob_general}
\end{equation}    
\end{theorem}
\vspace{-0.7em}
\begin{proof}
    (Sketch). The proof follows~\cite{lecue_regularization_2017,ziemann_learning_2022} and it consists in characterizing the scenarios that lead to the event $\Alo_r^{\complement}$, showing that the case for which $\fhat \in \F \setminus \Freg$ cannot occur for $\lambda_T$ sufficiently large. The detailed proof is given in~\cref{sec:proofboundprob}.
\end{proof}

\subsection{Result in expectation}\label{sec:bound_exp}

\begin{theorem}\label{thm:main_exp}
    Let~\cref{ass:densities,ass:SobolevOrder,ass:noise,ass:containment,ass:S-persistence} hold. Consider a parameter $\theta > 8$, a radius $r>0$, and let $\Fr$ be a $r/\sqrt{\theta}$-cover in the infinity norm of $\partial B(r)$ that is $(C(r),2)$-hypercontractive.
    Consider the regularized empirical risk minimization problem in~\eqref{eq:RERM} with regularization parameter satisfying 
    $\lambda_T \geq \frac{40}{3\rho}\E{\MOC{\Freg}}{W}$, where 
    $\rho \geq 10 \reg{\fstar}$. Then, letting $B$ be the positive constant such that $\norm{f}{\Ellinf(\X^T;\R^{\dy})} \leq B$ for all $f \in \F$, the estimate $\fhat$ satisfies
    \begin{align}
        \E{\norm{\fhat - \fstar}{\Elltwo(\X^T, \jointProb; \R^{\dy})}^2}{} &\leq 4B^2\N{\infty}{\partial B(r)}{\frac{r}{\sqrt{\theta}}}\exp\left\{ - \frac{8T}{\theta^2 C_r S}\right\} \notag \\ &+ \theta \E{\MOC{\Freg}}{} + \lambda_T \reg{\fstar} + r^2. \notag 
    \end{align}
\end{theorem}
\vspace{-0.7em}
\begin{proof}(Sketch). The idea consists in decomposing the expected value according to the lower-isometry event $\Alo_r$ and its complement: informally, we would write \begin{equation} \E{\text{excess risk}}{}  = \allowbreak \E{\text{excess risk} \cap \Alo_r}{} + \E{\text{excess risk} \cap \Alo_r^{\complement}}{}.\notag\end{equation} The first term would then be bounded thanks to $S$-persistence, $(C,2)$-hypercontractivity and $B$-boundedness, which allow us to quantify the probability of the lower-isometry event $\Alo_r$ (see~\cref{sec:lower_isometry_bound}). The bound for the second term is derived along the lines of the proof of~\cref{thm:main_probab}. The full details are presented in~\cref{sec:proofboundexp}.
\end{proof}
Overall, our analysis deploys the concepts of $S$-persistence and $(C,\alpha)$-hypercontractivity to adapt the small-ball argument of~\cite{mendelson_learning_2014} to the data-dependent case. Thanks to this construction, we can identify the lower-isometry event, which enables the derivation of our bounds depending on the martingale offset complexity, the ground-truth regularizer $\reg{\fstar}$ and the critical radius~$r$. In the next section, we will characterize the behavior of these terms to obtain the desired convergence rates for physics-informed learning.

\section{Convergence rates}\label{sec:convergence_rates}

We finally provide our main results in terms of convergence rates for the excess risk, whose detailed proofs are deferred to~\cref{sec:proofs_rates}. Throughout this section, we will denote by $d = \nicefrac{2s}{2s+\dx}$ the Sobolev minimax rate, and $d' = \nicefrac{2\dx}{2s+\dx}$.
\newpage

\subsection{Bound in probability}
\begin{theorem}\label{thm:main_rate_probability}
    Let~\cref{ass:densities,ass:SobolevOrder,ass:elliptic,ass:containment,ass:noise,ass:S-persistence} hold, and let $\fhat$ be the solution of~\eqref{eq:RERM}. Fix a probability of failure $\delta \in (0,1)$, and assume the regularization parameter $\lambda_T$ satisfies
    \begin{equation}
        \lambda_T \geq \frac{4}{3T^d} \left[ \frac{C_I \sigma_W^{1+d}}{\reg{\fstar}^{1-\frac{d'}{4}}} + \frac{(C_{II}+C_{IV})\sigma_W^{2d}}{\reg{\fstar}^{1-\frac{d'}{2}}} + \frac{C_{III}\sigmaw \log(1/\delta)}{\reg{\fstar}} \right],\notag
    \end{equation}
    where 
    $C_I$, $C_{II}$, $C_{III}$ and $C_{IV}$ are constants depending only on $s,\dx,\dy$ and $\sqrt{\log(1/\delta)}$.
    If the number of samples $T$ satisfies
    \begin{equation}
T \geq \frac{\theta^2 C_h S}{8}\left[C_M\left(\frac{1}{r}\right)^{\frac{6\dx}{2s-\dx}}\log\left(1 + C_L\left(\frac{1}{r}\right)^{\frac{4s-\dx}{2s-\dx}} \right) + \left(\frac{1}{r}\right)^{\frac{4\dx}{2s-\dx}}\log(1/\delta) \right] \notag        
    \end{equation}
    for $r^2 = \lambda_T \reg{\fstar} + \sigmaw/T$ and $C_h,C_M,C_L$ being uniform constants depending on $\rho_f,\kappau,\theta,s,\dx$ and $\X$, then, with probability at least $1-6\delta$, the excess risk enjoys the following convergence rate:
    \begin{equation}
        \norm{\fhat - \fstar}{\Elltwo(\X^T,\jointProb;\R^{\dy})}^2 \leq \Cslow \frac{\max\left\lbrace \reg{\fstar}^{d'/4}, \reg{\fstar}^{d'/2}\right\rbrace}{T^d} + \Cfast \frac{\sigmaw \log(1/\delta)}{T}, \notag
    \end{equation}
    where $\Cslow$ is a constant that depends on $s, \dx, \dy, \sigmaw, \sqrt{\log(1/\delta)}$, and $\Cfast$ is a constant that depends on $s,\dx,\dy$.
\end{theorem}
\vspace{-0.7em}
\begin{proof}(Sketch).
    The result builds upon the bound in probability on the excess risk of~\cref{thm:main_probab}, and its crux consists in conveniently setting the values for the critical radius $r$, the radius $\rho$ of the effective hypothesis class $\Freg$, and the regularization parameter $\lambda_T$. This allows us to rewrite the excess risk bound~\eqref{eq:ER_bound_prob_general} in terms of the martingale offset complexity, which can in turn be bounded according to~\cite[Theorem~4.2.2]{ziemann_statistical_2022}  (see~\cref{thm:MOC_prob} for its detailed proof). Finally, the characterization of the burn-in follows from the probability of the lower-isometry event. The full proof is reported in~\cref{sec:proof_rate_prob}, where the value of all of the involved constants is given.
\end{proof}

\subsection{Bound in expectation}
\begin{theorem}\label{thm:main_rate_exp}  Let~\cref{ass:densities,ass:SobolevOrder,ass:elliptic,ass:containment,ass:noise,ass:S-persistence} hold, and let $\fhat$ be the solution of~\eqref{eq:RERM} with regularization parameter $\lambda_T$ satisfying
    \begin{equation}
        \lambda_T \geq \frac{4(C_I + C_{II})(\sigmaw)^d}{3 T \reg{\fstar}^{1 - \frac{d'}{2}}}, \notag
    \end{equation}
 where 
 $C_I$ and $C_{II}$ are constants depending only on $s,\dx$ and ~$\dy$.
    If $T$ satisfies
    \begin{equation}
T \geq \frac{\theta^2 C_h S}{8}\left(\frac{1}{r}\right)^{\frac{4\dx}{2s-\dx}} \left[ C_M \left(\frac{1}{r} \right)^{\frac{2\dx}{2s-\dx}}\log\left(4B^2\left(1 + C_L\left(\frac{1}{r}\right)^{\frac{4s-\dx}{2s-\dx}}\right) \right) + \log\left(\frac{\sigmaw}{T}\right)\right],\notag    
    \end{equation}
    where $B$ is such that $\norm{f}{\Ellinf(\X^T;\R^{\dy})}\leq B$ for all $f\in\F$ and $C_M,C_h,C_L$ are constants depending on $\rho_f,\kappau,\theta,s,\dx$ and $\X$, then the excess risk enjoys the following convergence rate:
    \begin{equation}
        \E{\norm{\fhat - \fstar}{\Elltwo(\X^T,\jointProb;\R^{\dy})}^2}{} \leq \Cslow \frac{ \reg{\fstar}^{d'/2}}{T^d} + \Cfast \frac{\sigmaw \log(1/\delta)}{T}, \notag
    \end{equation}
    where $\Cslow$ and $\Cfast$ are constants that depend on $s, \dx, \dy$ and~$\sigmaw$.
\end{theorem}
\vspace{-0.7em}
\begin{proof}(Sketch). Similarly to~\cref{thm:main_rate_probability}, one starts from~\cref{thm:main_exp} to set the values for $\rho$ and $\lambda_T$, and then deploys the bound on the expected martingale offset complexity of~\cite[Theorem~3.2.1]{ziemann_statistical_2022} (see~\cref{thm:MOC_exp} for its detailed proof). Ultimately, the claim is obtained by suitably choosing the critical radius $r$ and accordingly characterizing the lower-isometry event probability, leading to the expression for the burn-in. The detailed proof can be found in~\cref{sec:proof_rate_exp}.
\end{proof}

Overall, our analysis allows us to transfer the contribution of data dependence from the excess risk bound to the burn-in time condition. Moreover, our bounds feature a fast, i.i.d.-like term ($\bigO(T^{-1})$) and a slower Sobolev rate term ($\bigO(T^{-d})$) that becomes annihilated when $\reg{\fstar} \simeq 0$: this proves that, under knowledge alignment, the learning rate speeds up to $\bigO(T^{-1})$ even if data are dependent.

\section{Related work and discussion}\label{sec:discussion}
\paragraph{General statistical learning framework.} 
The general theory of statistical learning rates has developed along two main streams, as identified by \cite{fischer_sobolev_2020}. The first relies on the spectral analysis of integral operators in reproducing kernel Hilbert spaces \citep{smale_learning_2007,caponnetto_optimal_2007,steinwart_optimal_2009}, while the second builds on empirical process techniques and the small-ball method
\citep{mendelson_learning_2014,mendelson_learning_2018,lecue_regularization_2017}. Our work belongs to the latter stream, adapting the small-ball method to the \emph{dependent-data} case along the lines of the localization analysis of~\cite{ziemann_learning_2022}.

\paragraph{Learning rates for dependent data.} A common approach to handle dependence is through \emph{blocking} techniques \citep{yu_rates_1994,sancetta_estimation_2021}, where the trajectory is divided into blocks of length $k$ so that consecutive blocks can be treated as independent. However, this deflates the effective sample size, leading to suboptimal rates. Similar rates appear also in~\cite{steinwart_fast_2009,zou_generalization_2009,agarwal_generalization_2012,kuznetsov_generalization_2017}, and~\cite{nagaraj_least_2020} shows that such a deflation in a worst-case agnostic model set-up is unavoidable. To contrast this phenomenon, a significant line of work has studied learning under dependent data \emph{without regularization}. In the linear setting, \cite{simchowitz_learning_2018} and \cite{nagaraj_least_2020} established sample complexity bounds for system identification and stochastic gradient descent. Moreover, \cite{roy_empirical_2021} extended the small-ball method to dependent processes, but without using one-sided concentration, leading to slower rates. Similar slower-rate phenomena also appear in \cite{ziemann_single_2022}. More recently, \cite{ziemann_learning_2022} proposed an adaptation of the small-ball method and offset complexity technique of~\cite{liang_learning_2015} to obtain optimal rates for nonlinear settings. Our work builds upon this line of thought, extending the analysis to \emph{physics-informed regularization}. However, the results in this paper are not a mere adaptation: the physics-informed regularizer introduces additional challenges, such as characterizing the entropy numbers of the effective hypothesis class (e.g., under ellipticity, non-trivial nullspaces of the operator, and boundary conditions), 
determining trajectory hypercontractivity and working with weighted, vector-valued Sobolev spaces.

\paragraph{Theoretical analysis of physics-informed machine learning.} Our work belongs to the branch of physics-informed machine learning that aims at enhancing learning algorithms with available physical knowledge --- a class of models also known as \emph{hybrid modeling}~\citep{rai_driven_2020,von_rueden_informed_2023}. To the best of the authors' knowledge, results aimed at quantifying the beneficial impact of physical priors in learning algorithms are~\cite{von_rueden_how_2023} and~\cite{doumeche_physics-informed_2024}. The present paper is very similar in spirit to the latter work in the way complexity-dependent rates are derived, but crucially deals with non-i.i.d.~data and presents bounds for the excess risk not only just in expectation, but also in probability. We further summarize related work in~\cref{tb:comp}.

\begin{table}[ht]
    \caption{\small Comparison of convergence rates for non-parametric regression with and without regularization. The rate from~\cite{ziemann_learning_2022} follows from its Corollary~4.1 with $q=\dx/s$ under the metric entropy bound $\log \N{\infty}{\F}{\varepsilon} \sim (1/\varepsilon)^q$.  The rate from~\cite{lecue_regularization_2017} follows from its Lemma~2.1 assuming $r^2(\rho)\sim\sigmaw T^{-1}$, 
with $\lambda_T \sim T^{-d}$.
} 
  \label{tb:comp}%
  \centering%
  \resizebox{\textwidth}{!}{%
      \begin{tabular}{lllllll}
        \thead{Work}                                  &\thead{Hypothesis class} & \thead{Data} & \thead{Regularization} & \thead{Assumption}& \thead{Rate} \\
        \midrule                                                                                         
        
        \cite{nussbaum_minimax_2006} & $\Elltwo$ Sobolev space & i.i.d. & \hspace{1cm}\xmark & $\sigmaw$-Gaussian, $\dx = 1$ & $\sigmaw T^{-2s/(2s+1)}$ \\
        
        \cite{farahmand_regularized_2012} & General Sobolev space   & non-i.i.d.      & \hspace{1cm}\xmark      & Exponential mixing, $\dy = 1$          & $T^{-d}\log(T)$        \\
        
        \cite{lecue_regularization_2017} & General   & i.i.d.      & Proper regularizer      & $\sigmaw$-sub-Gaussian, $\dy = 1$          & $\reg{\fstar} T^{-d} + \sigmaw T^{-1}$\\
        
        \cite{ziemann_learning_2022}  & General (not too large)   & non.i.i.d. & \hspace{1cm} \xmark    & $\sigmaw$-sub-Gaussian  & $\sigmaw T^{-d}$ \\
        
         \cite{doumeche_physics-informed_2024}     & Periodic Sobolev space     & i.i.d.            & Physics-informed   & $\sigmaw$-sub-Gamma, $\dy = 1$         & $\reg{\fstar} T^{-d} + \sigmaw T^{-1}$                         \\
        \rowcolor{blue!30}\bf Our work               & $\Elltwo$ Sobolev space    & non-i.i.d.      & Physics-informed      & $\sigmaw$-sub-Gaussian, $s \geq 2\dx$           & $\reg{\fstar}^{d^\prime/2}T^{-d} + \sigmaw T^{-1}$
      \end{tabular}
}
\end{table}

\paragraph{Quantifying the impact of knowledge alignment.}
We now showcase the impact of knowledge alignment $\reg{\fstar} \simeq 0$ in contrast with the rates of empirical risk minimization \emph{without regularization} --- i.e., considering $\fhat'$ as the solution of~\eqref{eq:RERM} when $\lambda_T =0$. 
As shown in detail in~\cref{sec:results_noreg}, the excess risk for $\fhat'$ behaves, both in probability and in expectation, in the following way (informally):
\begin{equation}
    \text{(Excess risk)} \quad ||\fhat'-\fstar||^2_{\Elltwo(\X^T,\jointProb;\R^{\dy})} \leq \frac{\Cslow'}{T^d} + \Cfast' \frac{\sigmaw}{T}.\notag 
\end{equation}
We can notice how, for the result without regularization, the term decaying according to the Sobolev rate is not modulated by any design parameter (as happened with $\reg{\fstar}$ in~\cref{thm:main_rate_probability,thm:main_rate_exp}), and is thus the dominant term dictating the slow Sobolev convergence rate of the excess risk.

\paragraph{On the behavior of $\lambda_T$.} It is worth emphasizing that, in both the expectation and probability analyses, the condition on the regularization parameter depends on $1/\reg{\fstar}^\beta$ for some $\beta > 1$. This condition reflects the well-known regularization-complexity trade-off: as the hypothesis class is restricted (i.e., as $\rho$ becomes small), one must increase $\lambda_T$ to compensate for the reduced richness of the class and the potentially higher sensitivity to noise or variance, as discussed in~\cite[Section~2]{lecue_regularization_2017} and also displayed in~\cite[Theorem~5.3]{doumeche_physics-informed_2024}. Even if such a phenomenon prevents us from considering the case $\reg{\fstar}=0$, our bounds still capture the (practical) annihilation of the Sobolev rate term in presence of knowledge alignment. Finally, as pointed out in~\cite{doumeche_physics-informed_2024}, even if $\lambda_T$ depends on the unknown $\reg{\fstar}$, it can still be estimated via, e.g., cross-validation~\citep{wahba_spline_1990}.

\paragraph{On the burn-in condition and the Sobolev order $s$.} In~\cref{thm:main_rate_probability,thm:main_rate_exp}, the burn-in time scales as $(1/r)^{\nicefrac{6\dx}{2s-\dx}}$, and $r$ in turn scales as ${T}^{-1/2}$. Therefore, to ensure well-posedness of the burn-in time condition, we have to impose that $\nicefrac{3\dx}{2s-\dx} \leq 1$, which yields~\cref{ass:SobolevOrder}. Thus, our results come at the price of a stronger requirement on $s$ with respect to the standard $s \geq \dx/2$ needed, e.g., for the Sobolev imbedding theorem (\cref{sec:sobolevProperties}).

\paragraph{Numerical experiment.} 
We complement our theoretical analysis with an example showcasing the benefit of prior domain knowledge in learning a nonlinear dynamical system. 
In this experiment, whose full details can be found in~\cref{sec:numerical_experiment}, we consider the dynamics of a unicycle robot described by the differential equations
$\dot{x}_1(t) = \nu(t) \cos \vartheta(t), \;\ \dot{x}_2(t) = \nu(t) \sin \vartheta(t), \;\ \dot{\vartheta}(t) = \omega(t),$ 
where $(x_1,x_2) \in \R^2$ is the position of the robot on the plane, $\vartheta \in [0,\pi/2]$ is the orientation angle, and $(\nu,\omega)$ are the translational and angular velocities, respectively. The physical information we want to incorporate is that the velocity has no lateral component, enforcing the non-slip behavior of the unicycle kinematics. Such a constraint is embedded in the learning problem~\eqref{eq:RERM} as a (discretized) $\Elltwo$-regularization term, and we perform estimation by deploying a multilayer perceptron with two hidden layers featuring 64 nodes and ReLU activation functions. \\
\begin{SCfigure}[][h!]
    \centering
    \includegraphics[width=0.45\textwidth]{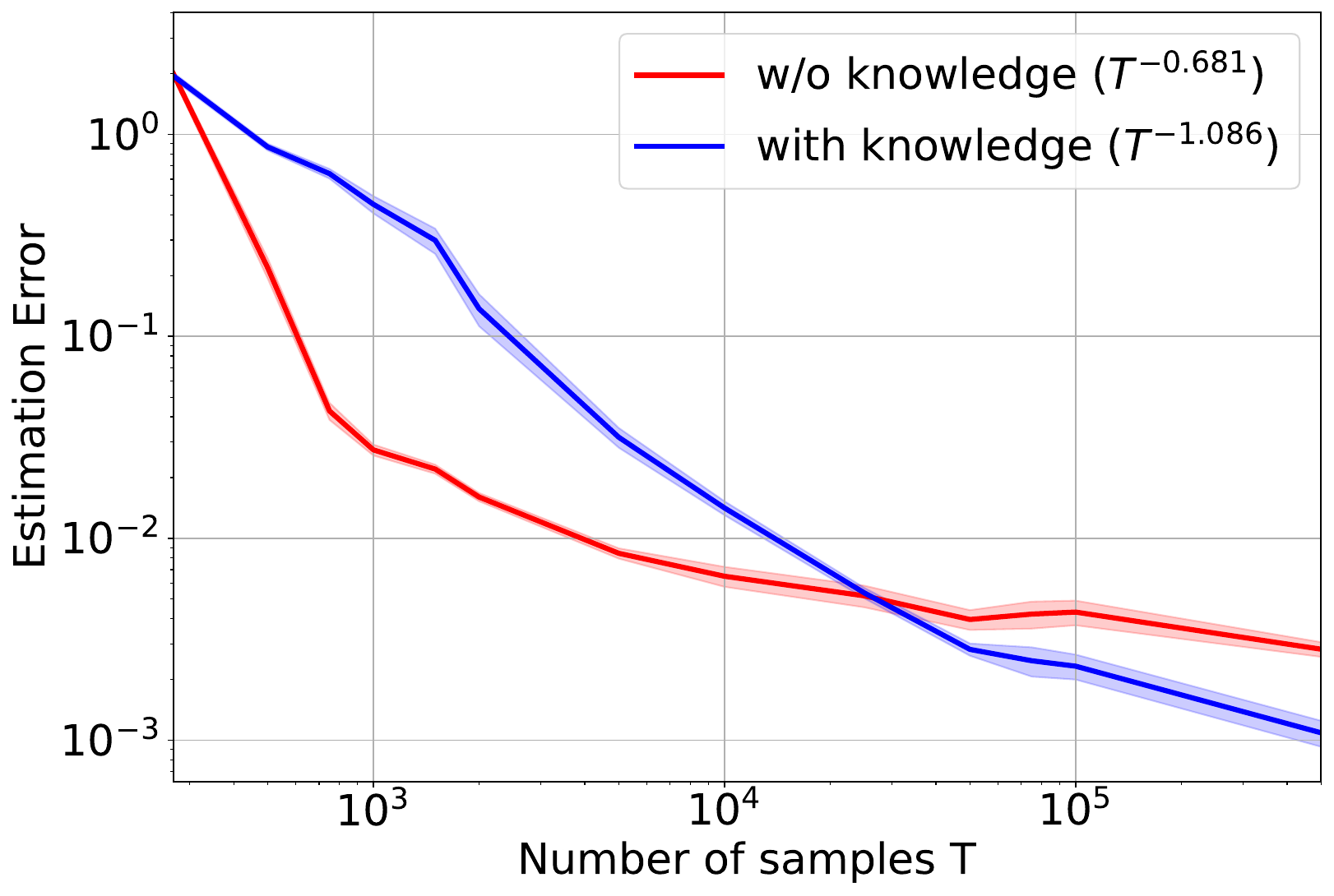}
    \vspace{-0.4cm}
    \caption{\small Log-log plot of the empirical excess risk (estimation error)  with respect to the number of samples $T$ for the unicycle dynamics after the burn-in period. Each curve is obtained by averaging over $20$ independent random realizations of the training data, with solid lines indicating the mean estimation error and shaded regions denoting $95\%$ confidence intervals.}
\label{fig:unicycle_rates}
\end{SCfigure}

The experiment, whose results are displayed in~\cref{fig:unicycle_rates}, compares the empirical rates obtained with and without physics-informed regularization. 
We can notice that both estimators eventually return an accurate model for the ground-truth dynamics. 
However, without physics knowledge the rate of decay of the estimation error is relatively slow, with an empirical slope of approximately $\mathcal{O}(T^{-0.681})$. In contrast, incorporating physics-informed regularization yields a markedly faster decay, with an empirical slope of approximately $\mathcal{O}(T^{-1.086})$, as the model is explicitly constrained by the domain knowledge that unicycle dynamics do not admit lateral velocity. This experiment demonstrates how embedding physics-based operators into the training objective leads to provable improvements in sample efficiency, consistent with our theoretical trends predicted in Section~\ref{sec:convergence_rates} -- especially the result in expectation presented in~\cref{thm:main_rate_exp}.

\section{Conclusions}\label{sec:conclusions}

This work focused on vector-valued function estimation from dependent data, and studied the excess risk of the estimate $\fhat$ obtained through regularized empirical risk minimization, where regularization is induced by physical knowledge (namely, that the unknown function approximately satisfies a partial differential equation). The main message of this work is that knowledge alignment (i.e., the regularizer is approximately zero when evaluated at the ground-truth function $\fstar$) allows to speed up the learning rate from the slow, Sobolev rate $\bigO(T^{-d})$, with $d=\nicefrac{2s}{2s+\dx}<1$, to the fast, optimal i.i.d.~one $\bigO(T^{-1})$.
Taken together, our results provide the first convergence rates for physics-informed learning under dependent data that avoid the sample-size deflation inherent to blocking techniques, and  reveal a transition from Sobolev minimax rates to fast i.i.d.-optimal rates through knowledge alignment. This bridges classical statistical learning theory, physics-informed regularization, and learning with dependent data.

\section*{Acknowledgments} 
Leonardo F. Toso thanks Ingvar Ziemann for an instructive discussion during the early stage of this work. Leonardo F. Toso is funded by the Center for AI and Responsible Financial Innovation (CAIRFI) Fellowship and by the Columbia Presidential Fellowship. James Anderson is partially funded by NSF grants ECCS 2144634 and 2231350.

\printbibliography

\clearpage
\appendix
\begin{center}
{\bfseries \Large $\bullet$  Technical appendix $\bullet$}
\end{center}

In the following sections we provide the derivations of all of the results stated in the paper. This technical appendix is structured as follows: 
\paragraph{\cref{appendix:definitions}} \hspace{-0.9em} provides the necessary results for the probability set-up of the estimation problem. Specifically, we construct the marginal probability measures stated in~\cref{ass:densities} (\cref{sec:mut}); we discuss the meaning of the $S$-persistence given in~\cref{ass:S-persistence} and its relation to data dependence (\cref{sec:Spersistence}); and derive some useful properties of sub-Gaussian random vectors that will be useful in the martingale offset complexity bounds (\cref{sec:usefulsubGaussian}).
    \paragraph{\cref{sec:sobolevProperties}} \hspace{-0.9em} constructs the auxiliary set-up of weighted and vector-valued Sobolev spaces needed to define the hypothesis spaces for the empirical risk minimization problem~\eqref{eq:RERM}. First, we extend the Sobolev imbedding theorem to the weighted, vector-valued case (\cref{sec:SobolevImbedding}). Next, we present definitions and key results of periodic Sobolev spaces (\cref{sec:periodicSobolev}), and show that the Sobolev space of interest, $\Sobo(\X^T,\jointProb;\R^{\dy})$ is imbedded in one of them. Such a construction will be leveraged in the proof of $(C(r),2)$-hyper-contractivity of $\partial \balltwo{r}$  in~\cref{sec:hypercontractivity}.
    \paragraph{\cref{sec:covering_numbers_results}} \hspace{-0.9em} focuses on bounds for covering numbers of convex sets of vector-valued Sobolev spaces. By deploying the direct-sum structure of the Sobolev space $\Sobo(\X^T,\jointProb;\R^{\dy})$ elucidated in~\cref{sec:function_spaces}, we first show how the covering of the multi-dimensional set can be obtained from the covering of the one-dimensional counterpart (\cref{sec:scalar2vec_covering}). We then use such a result to extend classic results on covering numbers, culminating in the bound for the covering number of the effective hypothesis space $\Freg$ (\cref{sec:covering_vector}).
     \paragraph{\cref{sec:properties_hypothesis_spaces}} \hspace{-0.9em} shows key properties of the hypothesis spaces $\F$ and $\Freg$, such as convexity and $B$-boundedness (\cref{sec:convexity_boundedness}), and $(C(r),2)$-hypercontractivity of $\partial \balltwo{r}$ (\cref{sec:hypercontractivity}). We also show how the trajectory hypercontractivity condition enforces the small-ball property.
     \paragraph{\cref{sec:regularizer_properties}} \hspace{-0.9em} focuses on the physics-informed regularizer and its properties. In particular, we show that the physics-informed regularizer is 2-proper (\cref{sec:reg_2proper}), and prove an inequality on the difference $\reg{\fhat} - \reg{\fstar}$ that will be useful in the proofs of~\cref{thm:main_probab,thm:main_exp} (\cref{sec:inequality_LM}).

\paragraph{\cref{sec:lower_isometry_bound}} \hspace{-0.9em} provides the bound for the probability of the lower-isometry event $\Alo_r$ introduced in~\cref{sec:ideaBounds}. We first show an ancillary inequality linking hypercontractivity and $S$-persistence (\cref{sec:combo_CalphaS}) and then derive the main lower-isometry bound result (\cref{sec:bound_Alor}), also presenting its corollary that will be useful in characterizing the burn-ins in~\cref{sec:convergence_rates}. 

\paragraph{\cref{sec:MOC}} \hspace{-0.9em} provides the full derivation of the bounds for the martingale offset complexity, which play a key role in the results of~\cref{sec:bounds,sec:convergence_rates}. We first show the inequality that underpins the definition of martingale offset complexity (\cref{sec:MOC_def}), and then prove its bounds, both in probability (\cref{sec:MOC_prob}) and in expectation (\cref{sec:MOC_exp}).  

\paragraph{\cref{sec:proofs_sec4}} \hspace{-0.9em} contains the proofs of the excess bound rates, namely of~\cref{thm:main_probab} (\cref{sec:proofboundprob}) and of~\cref{thm:main_exp} (\cref{sec:proofboundexp}).

\paragraph{\cref{sec:proofs_rates}}\hspace{-0.9em} collects the proof of the convergence rates results of~\cref{sec:convergence_rates}, specifically of~\cref{thm:main_rate_probability} (\cref{sec:proof_rate_prob}) and of~\cref{thm:main_rate_exp} (\cref{sec:proof_rate_exp}).

\paragraph{\cref{sec:results_noreg}} \hspace{-0.9em} provides the corollaries of the results  given in~\cref{sec:convergence_rates} dealing with empirical risk minimization without regularization, which we use in the comparison performed in~\cref{sec:discussion}. Specifically, we derive the result in probability (\cref{sec:noreg_prob}) and in expectation (\cref{sec:noreg_exp}). 

\paragraph{\cref{sec:numerical_experiment}} \hspace{-0.9em} presents the full details of the numerical experiment set-up outlined in the discussion reported in~\cref{sec:discussion}.

\section{Probability measure and stochastics set-up}\label{appendix:definitions}
This section collects all the ancillary results concerning the probability space $(\X^T,\lbrace \Filt_t\rbrace_{t=0}^{T-1},\jointProb)$, the inter-sample dependence in trajectories $\inputseq$ belonging to such a space, and the noise sequence we are considering. In particular, in~\cref{sec:mut} we specify the marginal distributions presented in~\cref{ass:densities}; then, in~\cref{sec:Spersistence} we discuss the $S$-persistence condition in~\cref{ass:S-persistence}, showing how $S$ quantifies the degree of dependence between data samples separated in time; finally, in~\cref{sec:usefulsubGaussian} we present some useful ancillary results on the second statistical moment of the sub-Gaussian random vectors given in~\cref{ass:noise}.

\subsection{On the construction of probability measures}\label{sec:mut}

We now characterize the probability measures $\mu_t$, defined for each $t=0,\dots,T-1$, associated with each term of the input trajectory $\inputseq$.\\
The classic set-up involves \textit{independent} samples. In this situation, the $\sigma$-algebra on $\X^T$ is given by the tensor product of the single $\sigma$-algebras $\Filt_t$. Moreover, by construction each component $X_t$ of the trajectory $\inputseq$ has a distribution $\mu_t$, and the resulting probability space is $(\X^T, \otimes_{t=0}^{T-1}\Filt_t, \prod_{t=0}^{T-1}\mu_t)$ --- see, e.g., ~\cite[Chapter~VII]{halmos_measure_1950} and~\cite[Section~18]{billingsley_probability_2012}.\\
We now detail the case with \textit{dependent data} building upon the results in~\cite[Chapter~I.6]{cinlar_probability_2011}. We are in the situation in which the transition between $X_{t-1}$ and $X_{t}$ for all $t=1,\dots, T$ is described by a map from $(\X,\Filt_{t-1})$ to $(\X, \Filt_{t})$. Such a map is called \textit{transition kernel} $\transker{\cdot}{\cdot}{t}\colon \X \times \Filt_{t} \to \R_{\geq 0}$ and is such that $x_{t-1} \mapsto \transker{x_t}{A}{t}$ is $\Filt_{t-1}$-measurable for every set $A \in \Filt_t$, and $A \mapsto \transker{x_{t-1}}{A}{t}$ is a measure on $(\X, \Filt_t)$ for every $x_{t-1} \in \X$. Before proving the main result in~\cref{thm:marginals}, we recall two key results:
\begin{lemma}\label{lemma:cinlar421}
    Let $(E,\mathcal{E})$ be a measurable space, and let $L$ be a functional mapping the space of non-negative measurable functions defined on $\mathcal{E}$ to $\R_{\geq 0}$. Then there exists a unique measure $\nu$ on $(E,\mathcal{E})$ such that $L(g) = \nu g$ for any function $g$ in the domain of $L$ if and only if
    \begin{enumerate}
        \item $g = 0$ implies $L(g) =0$;
        \item for any $\mathfrak{a},\mathfrak{b}\in \R_{\geq 0}$ and any $g,g^{\prime}$ in the domain of $L$, $L(\mathfrak{a}g + \mathfrak{b}g^{\prime}) = \mathfrak{a}L(g) + \mathfrak{b}L(g^{\prime})$;
        \item for any increasing sequence $\{g_n\}_n \nearrow g$ we have that $L(g_n) \nearrow L(g)$.
    \end{enumerate}
\end{lemma}
\begin{proof}
    We defer the interested reader to~\cite[Theorem~4.21]{cinlar_probability_2011}.
\end{proof}

\begin{lemma}\label{lemma:kernelComposition}
    Let $\transker{\cdot}{\cdot}{\tau}$ be a transition kernel from $(\X,\Filt_{\tau-1})$ to $(\X,\Filt_{\tau})$, and $\transker{\cdot}{\cdot}{\tau + 1}$ be a transition kernel from $(\X,\Filt_{\tau})$ to $(\X,\Filt_{\tau+1})$. Then, their product is the transition kernel $\mathcal{K}_{\tau}\mathcal{K}_{\tau+1}$ from $(\X,\Filt_{\tau-1})$ to  $(\X,\Filt_{\tau+1})$ such that
    \begin{align}
\mathcal{K}_{\tau}\mathcal{K}_{\tau+1}(x_{\tau-1},A) =\int_{\X} \transker{x_{\tau-1}}{dx_{\tau}}{\tau} \transker{x_{\tau}}{A}{\tau+1} \quad \text{for }x_{\tau -1} \in \X, A \in \Filt_{\tau+1}.\notag
    \end{align}
\end{lemma}
\begin{proof}
    It follows directly from~\cref{lemma:cinlar421}.
\end{proof}
We are now ready to state the existence and uniqueness of the probability measures associated with each component $X_t$ of the input trajectory $\inputseq$.
\begin{theorem}\label{thm:marginals}
    Let $g\colon \X\to \R_{\geq 0}$, and assume that there exists a probability measure $\mu_0$ associated with the first component of the input trajectory $\inputseq$. Then, for each $t=1,\dots,T-1$ there exists a unique probability measure such that \begin{align}
\int_{\X^T}g(X_t)d\jointProb =  \int_{\X} g(X_t)d\mu_t.\notag
    \end{align}
\end{theorem}
\begin{proof}
We are considering
\begin{equation}\label{eq:rewriteIntegral}
   \int_{\X^T} g(X_t)d\jointProb = \int_{\X}\mu_0(dX_0)\,\cdots \int_{\X}\transker{X_{t-1}}{dX_t}{t}g(X_t) \,\cdots \int_{\X} \transker{X_{T-1}}{dX_T}{T}.
\end{equation}
By an iterative application of~\cref{lemma:kernelComposition} to~\cref{eq:rewriteIntegral}, the contribution of the transition kernels $\transker{\cdot}{\cdot}{t+1},\cdots \transker{\cdot}{\cdot}{T}$ integrates to 1. Furthermore, we can apply again~\cref{lemma:kernelComposition} to the kernels $\transker{\cdot}{\cdot}{1},\cdots,\transker{\cdot}{\cdot}{t}$ and obtain the composed kernel $\bar{\mathcal{K}}(\cdot,\cdot)$ such that~\eqref{eq:rewriteIntegral} can be re-written as
\begin{align}
 \int_{\X^T} g(X_t)d\jointProb = \int_{\X} \mu_0(dX_0)\int_X \bar{\mathcal{K}}\left(X_0, dX_t\right)g(X_t).
 \notag 
\end{align}
It can be shown, along the lines of~\cite[Theorem~6.3]{cinlar_probability_2011}, that the right-hand side of the equation above satisfies~\eqref{lemma:cinlar421}, thus proving the claim.
\end{proof}
Note that this theorem holds also for the independent-measures case, where each $\mu_t$ is the $t$-th marginal of $\jointProb$ and can be computed relying on Fubini's Theorem.

\subsection{On $S$-persistence and data dependence}\label{sec:Spersistence} We now focus on the concept of $S$-persistence (see~\cref{ass:S-persistence}) and on how it relates to the dependence of the samples in the trajectory $\inputseq$.\\
We first start by recalling the definition of $S$-persistence. The tuple $(\F,\jointProb)$ is $S$-persistent if, for every $\xi \geq 0$ and $f\in\F$, we have that
\begin{equation}
\mathbb{E}\left[\exp\left(-\xi \sumt   \norm{f(X_t)}{2}^2 \right) \right] \leq \exp\left( -\xi \sumt   \mathbb{E}\left[\norm{f(X_t)}{2}^2 \right] + \frac{\xi^2 S}{2}\sumt   \mathbb{E}\left[ \norm{f(X_t)}{2}^4\right] \right). \notag
\end{equation}
The parameter $S$ is related to the ``degree of dependence" of the samples within the trajectory $\inputseq$: this is typically quantified in terms of the \emph{dependence matrix} (also defined \emph{mixing matrix} in~\cite{paulin_concentration_2018}), which we now define. 

Let $\Filt_{i:j}$ be the $\sigma$-field generated by the truncated input sequence $\{X_t\}_{t=i}^j$, which we represent as $\inputseq_{i:j}$. Additionally, denote with $\norm{\cdot}{TV}$ the total variation norm between two probability measures on $\Filt_{i:j}$: specifically, such a metric is defined as $\norm{\nu_1 - \nu_2}{TV} \doteq \sup_{A \in \Filt_{i:j}} |\nu_1(A) - \nu_2(A)|$ for any couple of measures $\nu_1$ and $\nu_2$ defined on the $\sigma$-algebra $\Filt_{i:j}$. The dependence matrix $\Gamma(\jointProb)$ is a lower-triangular matrix whose $(i,j)$ element, for $i,j =1,\cdots,T$, is given by
\begin{align}\label{eq:dependenceMatrix}
   [\Gamma(\jointProb)]_{i,j} \doteq \begin{cases}
   \sqrt{2\sup_{A \in \mathcal{X}_{0:T-i}} \norm{\Prob_{X_{j:T-1}}(\cdot|A) - \Prob_{X_{j:T-1}}}{TV}} & (i < j) \\
   1 & (i = j)\\
   0 & (i > j)
    \end{cases}
\end{align}
Such a matrix provides a measure of dependence through its induced 2-norm, which is $\norm{\Gamma(\jointProb)}{2} = \arg\inf_{a > 0} \{ \norm{\Gamma(\jointProb)v}{2} \leq a \norm{v}{2}, \, v \in \R^{T}\}$: thus, we have that $S$-persistence is ruled by $S = \norm{\Gamma(\jointProb)v}{2}$. If the stochastic process has independent samples, then it holds that $\norm{\Gamma(\jointProb)}{2} = 1$; on the other hand, if the process is fully dependent, then we have that $\norm{\Gamma(\jointProb)}{2}$ grows linearly in $T$. We will focus on scenarios in which $\norm{\Gamma(\jointProb)}{2}$ is a constant: as elucidated in~\cite[Section~II]{samson_concentration_2000}, these are the following.
\begin{enumerate}[label = (\alph*)]
\item \emph{uniformly ergodic Markov chains}. In these Markov chains, the transition kernels $\transker{\cdot}{\cdot}{t} = \transker{\cdot}{\cdot}{}$ are time-homogeneous, and there exists an invariant distribution $\widetilde{\pi}$ such that, for every initial condition $x$, there exists a rate $\mathfrak{r} < 1$ and a constant $\mathfrak{A}>0$ such that $\norm{\transker{x}{\cdot}{}^t - \widetilde{\pi}(\cdot)}{TV} \allowbreak \leq \mathfrak{A}\mathfrak{r}^t$~\citep[Section~16.2.1]{meyn_markov_2009}. Another characterization of uniformly ergodic Markov chains is given by the Doeblin condition, and in~\cite[Section~2.4, Theorem~1]{doukhan_mixing_1994} uniform ergodicity is proven also for non-homogeneous kernels. In general, Markov chains satisfying uniform ergodicity are given, for instance, by linear and stable auto-regressive models $X_{t+1} = FX_t + W_t$: indeed, these are T-chains~\citep[Proposition~6.3.5]{meyn_markov_2009}, and the latter that are uniformly ergodic~\citep[Theorem~16.2.5]{meyn_markov_2009}. Another notable example is given by nonlinear state-space models $X_{t+1} = F_t(X_t, U_0,\cdots, U_t, W_k)$ of the form presented in~\cite[Section~2.2.2]{meyn_markov_2009} with control model $F_t(U_0,\cdots,U_t)$ that is stable in the sense of Lagrange~\citep[Section~16.2.3]{meyn_markov_2009}.

\item \emph{contracting Markov chains}. A weaker condition imposed on the behavior of the transition kernels is given by~\cite{marton_measure_1996}, where a weaker form of Doeblin's condition states that \\$\sup_{x_1,x_2 \in \X}\norm{\transker{x_1}{}{t} - \transker{x_2}{}{t}}{TV} < 1$ for all $t$. Markov chains satisfying such a condition do not have to be time-homogeneous and are called \emph{contracting} in~\cite{marton_measure_1996}. As argued in~\cite[Equation (2.8)]{samson_concentration_2000}, this kind of Markov chain also leads to a $\norm{\Gamma(\jointProb)}{2}$ that does not depend on $T$, as argued in~\cite[Equation (2.8)]{samson_concentration_2000}.

\item \emph{$\phi$-mixing processes.} A more general way of characterizing ergodicity without assuming the Markov property of the process is given by means of \emph{mixing processes} reviewed, e.g., in~\cite{bradley_basic_1986,bradley_basic_2005} and~\cite[Chapter 1]{doukhan_mixing_1994}. Here the focus is on $\phi$-mixing processes, thoroughly studied in~\cite{ibragimov_limit_1962}, which can lead to $\norm{\Gamma(\jointProb)}{2} = \bigO(1)$. These kind of processes characterize mixing through the measure $\Phi(\Filt_{0:i},\Filt_{j:T-1})  \doteq \sup_{A \in \Filt_{0:i}}\sup_{B\in \Filt_{j:T-1}} \allowbreak|\Prob(B|A) - \Prob(B)|$; additionally, let $\Phi_k(i,j) \doteq \sup_{i,j=0,\dots T-1}\{ \Phi(\Filt_{0:i},\Filt_{j:T-1}),\, j-i \geq k\}$. A process is $\phi$-mixing if $\Phi_{k}(i,j) \to 0$ as $k \to \infty$ for all $i,j$: in other words, the measure $\Phi(\cdot,\cdot)$ quantifies how ``independent" two non-overlapping blocks of variables in the trajectory $\inputseq$ become as their distance increases. Examples of $\phi$-mixing processes are uniformly ergodic Markov processes~\citep{diananda_probability_1953} and chains of infinite order, which are non-Markovian processes where the transition probability is influenced only slightly by the remote past that have been deployed to model psychology experiments~\citep{harris_chains_1955,bush_stochastic_1953,lamperti_chains_1959}.\\ Noting how $\Phi$ enters the definition of the dependence matrix~\eqref{eq:dependenceMatrix}, $\norm{\Gamma(\jointProb)}{2}$ can be characterized by $\Phi_k \doteq \sup_{i,j} \Phi_k(i,j)$: if $\Phi_k$ exhibits an exponential decay, or it holds that $\sum_{k=1}^{\infty} \sqrt{\Phi_k} < \infty$, then we have that $\norm{\Gamma(\jointProb)}{2}$ is a constant~\cite[p.425]{samson_concentration_2000}. 

\end{enumerate}

\subsection{Useful properties of sub-Gaussian vectors}\label{sec:usefulsubGaussian}
We now present two lemmas that will be useful in the proofs of the martingale offset complexity bounds in~\cref{sec:MOC}. First, we recall that, according to~\cref{ass:noise}, the noise sequence is a martingale difference sequence with sub-Gaussian noise. Recalling the definition, we then have, for every $\xi \in \R$ and every $u$ in the unit sphere in $(\R^{\dy}, \norm{\cdot}{2})$,
\begin{equation}
    \E{\exp\left\{ \xi \innerprod{W_t}{u}{2} \vert \Filt_{t-1} \right\}}{} \leq \exp\left\{ \frac{\xi^2 \sigmaw}{2}\right\}. \tag{\ref{eq:sub-Gaussian}}\notag
\end{equation}

We now present the lemma on the second moment of a sub-Gaussian random vector.

\begin{lemma}\label{lemma:usefulsub-Gaussian}
    Let $W$ be a sub-Gaussian random vector according to~\cref{ass:noise} and let $\mathfrak{A}$ be a positive constant. Then
    \begin{equation}
        \E{\exp\left\lbrace \xi^2 \norm{W}{2}^2 \right\rbrace}{W} \leq \exp\left\lbrace   \mathfrak{A}^2 \xi^2  \right\rbrace \quad \text{for } |\xi| < \frac{1}{\mathfrak{A}}.\notag
    \end{equation}
\end{lemma}
\begin{proof}
    The proof for the scalar case can be found in~\cite[Proposition 2.5.2]{vershynin_high-dimensional_2024}: i.e., when $\dy =1$, we obtain that 
    \begin{equation}
       \E{\exp\left\lbrace \xi^2 W^2 \right\rbrace}{W} \leq \exp\left\lbrace \mathfrak{A}^2 \xi^2 \right\rbrace \quad \text{for } |\xi| < \frac{1}{\mathfrak{A}}. \notag
    \end{equation}
To extend the claim to the case $\dy \geq 1$, we follow the argument in~\cite[Lemma 6.3.4]{ziemann_statistical_2022} and consider an auxiliary random vector $V$ that is drawn uniformly over the canonical Euclidean basis of $\R^{\dy}$, which we denote by $\lbrace e_1,\cdots, e_{\dy}\rbrace$. With such a construction, we have then that $\E{VV^{\top}}{V} = \frac{1}{\dy}\mathbb{I}_{\dy}$, where $\mathbb{I}_{\dy}$ is the identity matrix in $\R^{\dy}$. Then, we can consider
\begin{align}
    \E{\exp\left\lbrace \xi^2 \norm{W}{2}^2 \right\rbrace}{W} &= \E{\exp\left\lbrace \xi^2 W^{\top}\mathbb{I}_{\dy}W \right\rbrace}{W} \notag \\
    & = \E{\exp\left\lbrace \xi^2 \dy W^{\top}\E{VV^{\top}}{V}W \right\rbrace }{W} \notag \\
    & \leq \E{\exp \left\lbrace \xi^2 \dy (V^{\top}W)^2\right\rbrace}{W,V} \quad \text{ by Jensen's inequality,}\notag \\
    & \leq \exp\left\lbrace \mathfrak{A}^2\dy^2\xi^2 \right\rbrace \quad \text{for } |\xi|< \frac{1}{\mathfrak{A}\dy}, \notag
\end{align}
by applying the result of~\cite[Proposition 2.5.2]{vershynin_high-dimensional_2024} on the scalar sub-Gaussian random variable $V^{\top}W$ conditioned on $V$.
\end{proof}
Next, we derive a similar result for the 2-norm of a sub-Gaussian random vector.
\begin{lemma}\label{lemma:usefulsub-Gaussian2}
    Let $W$ be a sub-Gaussian random vector according to~\cref{ass:noise}. Then
    \begin{equation}
        \E{\norm{W}{2}}{W} \leq 3\sqrt{\dy}\sigma_W.\notag
    \end{equation}
\end{lemma}
\begin{proof}
    Again, the claim for the scalar case $\dy=1$ can be found in~\cite[Proposition 2.5.2(ii)]{vershynin_high-dimensional_2024}, where we have that $\E{|W|}{W} \leq 3 \sigma_W$: the value for the exact constants can be retrieved from the proof using $K_5^2 = \sigmaw/2$, $K_1^2 = 2\sigmaw$ and $K_2 = (3p)^{1/p}\sqrt{p K_1^2/2}$. 

    To obtain the result for the general case $\dy \geq 1$ we proceed similarly to~\cref{lemma:usefulsub-Gaussian} and work with the auxiliary random vector $V$ that is uniformly drawn from the canonical basis of $\R^{\dy}$, resulting in
    \begin{align}
        \E{\sqrt{W^{\top}W}}{W} = \E{\sqrt{\dy W^{\top}\E{VV^{\top}}{V}W}}{W} = \sqrt{\dy}\E{|V^{\top}W|}{W,V} \leq 3\sqrt{\dy}\sigma_W. \notag
    \end{align}
\end{proof}

\section{Sobolev spaces}\label{sec:sobolevProperties}
We now provide the essential information on 
 Sobolev spaces needed in this paper. For a deeper treatment on the subject, we defer to the monograph~\cite{adams_sobolev_2003}, as well as~\cite[Chapter~5]{evans_partial_2010}, \cite{brezis_functional_2011}, \cite[Chapter~1]{grisvard_elliptic_2011}, \cite[Chapter~7]{renardy_introduction_2004}, \cite[Chapter~4]{taylor_partial_2023}. References focused on weighted Sobolev spaces are~\cite{kufner_weighted_1985,kilpelainen_weighted_1994,chua_weighted_1996,goldshtein_weighted_2009}. Finally, results of vector-valued Banach spaces given by direct sums of scalar spaces can be found in~\cite[Chapter~I.6]{conway_course_2007}, \cite[Proposition~5.81]{farenick_fundamentals_2016},~\cite[Section~3.7]{bowers_introductory_2014}.

 \subsection{Imbedding properties}\label{sec:SobolevImbedding}
 We start with the well-known Sobolev imbedding theorem. 
 
 \paragraph{Preliminary definitions.}
 A normed space $(\HH_a, \|\cdot \|_{\HH_a})$ is \textit{imbedded} in another normed space $(\HH_b, \|\cdot \|_{\HH_b})$ if $\HH_a \subseteq \HH_b$, and for any $u \in \HH_a$ there exists a constant $\mathfrak{C}_b$ such that $\|u\|_{\HH  _{b}} \leq \mathfrak{C}_b \|u\|_{\HH_{a}}$. We denote such an imbedding by $\HH_a \hookrightarrow \HH_b$. \\ 
 We now define the Banach space of differential functions that  have continuous partial derivatives up to some order $j \in \Z$ for our vector-valued set-up. We will consider the space $\Cont^j\left(\X^T;\R^{\dy}\right)$ describing functions $f \colon \X \to \R^{\dy}$ evaluated along each component of the input trajectory $\inputseq$. We endow such a space with the norm
\begin{align}
    \norm{f}{\Cont^j\left(\X^T;\R^{\dy}\right)} \doteq \sum_{i=1}^{\dy} \sum_{|\alpha|\leq j} \frac{1}{T}\sumt \sup_{X_t \in \X} |\D f_i(X_t)| = \sum_{i=1}^{\dy} \sum_{|\alpha|\leq j}\sup_{X_t \in \X} |\D f_i(X_t)|.  \notag
\end{align} 

\paragraph{Main result.}
We are now ready to prove the Sobolev imbedding theorem for vector-valued Sobolev spaces as constructed in~\cref{sec:function_spaces}.
 
\begin{theorem}[Sobolev imbedding]\label{thm:sobolev_imbedding}
    Let~\cref{ass:densities} and~\cref{ass:SobolevOrder} hold, and let $j$ be a non-negative integer such that~\cref{ass:SobolevOrder} can be written as $s = \lceil \frac{\dx}{2} \rceil + j$ .Then
    \begin{enumerate}[label = (\alph*)]
        \item $\Sobo(\X^T,\jointProb;\R^{\dy}) \hookrightarrow \Cont^j(\X^T;\R^{\dy})$;
        \item $\Sobo(\X^T,\jointProb;\R^{\dy}) \hookrightarrow \mathscr{L}^q(\X^T,\jointProb;\R^{\dy}) \quad$ for $q \geq 2$.
    \end{enumerate}
\end{theorem}
\begin{proof}
Let us first define the order $j$. We can write~\cref{ass:SobolevOrder} as $s = 2\dx + j'$ for some $j' \in \Z$. Then it has to hold that $j = 2\dx - \lceil \frac{\dx}{2} \rceil + j' \geq 2\dx - \lceil \frac{\dx}{2} \rceil$.\\  
After noting that (i) with our definition of $\Elltwo$, we can consider each component of the input trajectory separately; (ii)~\cref{ass:densities} ensures that the weighted Sobolev spaces are equivalent to the standard ones~\citep{kufner_weighted_1985}; and (iii) the structure of the multi-output function spaces is given in terms of direct sums of scalar ones, we can then consider the claim for $\Sobo(\X,\mu_t;\R^{\dy})$ and apply~\cite[Theorem~4.12]{adams_sobolev_2003}, to which we defer for a complete proof.\\

The imbedding theorem in~\cite{adams_sobolev_2003} is stated under the assumption that the input domain satisfies the ``cone condition"~\citep[Definition~4.6]{adams_sobolev_2003}. We now argue that our set-up satisfies it.
    Since our input domain $\X$ is bounded with locally Lipschitz boundary, we have that it satisfies the \textit{strong local Lipschitz condition}~\citep[Definition 4.9]{adams_sobolev_2003}. Then, from~\cite[Paragraph~4.11]{adams_sobolev_2003} we have that such a property implies the \textit{uniform cone condition}~\citep[Definition~4.8]{adams_sobolev_2003} (see also~\cite[Theorem~1.2.2.2]{grisvard_elliptic_2011}), which in turn implies that the cone condition holds. 
\end{proof}
As a corollary of the imbedding theorem, one can obtain the following result:
\begin{corollary}[\cite{berlinet_reproducing_2004}, Theorem 121]\label{cor:berlinet}
Under~\cref{ass:densities} and~\cref{ass:SobolevOrder}, $\Sobo(\X^T,\jointProb;\R^{\dy})$ is a Reproducing Kernel Hilbert space. 
\end{corollary}

\subsection{Periodic Sobolev space $\Soboper$}\label{sec:periodicSobolev}

We will now characterize another kind of Sobolev space that will be useful for further analysis. Let us denote $\smallcube \doteq [-L,L]^{\dx}$ the hyper-cube in $\R^{\dx}$ that contains the input domain $\X$; similarly, we will use $\mediumcube \doteq [-2L,2L]^{\dx}$ and $\bigcube \doteq [-4L,4L]^{\dx}$. In the following, we will define and present the main properties of the periodic Sobolev space $\Soboper\left(\mediumcube^T,\Lebmeas^T; \R^{\dy}\right)$, i.e., considering the Lebesgue measure on $\X$ instead of the marginals $\mu_t$ for each $t=0,\cdots,T-1$. This simplifies the presentation, but later we will leverage~\cref{ass:densities} to tackle the case of $\Soboper\left(\mediumcube^T,\jointProb; \R^{\dy}\right)$.\\
The material in this subsection presents a generalization of the material of~\cite[Appendix~A]{doumeche_physics-informed_2024}. Additional references on periodic Sobolev are~\cite{temam_navier-stokes_1995,iorio_fourier_2001,berlinet_reproducing_2004,penteker_sobolev_2015,gwinner_fourier_2018,bogachev_fourier_2020}.

\paragraph{Preliminary definitions.}
Given a point $x = (x_1,\cdots,x_{\dx})$ and a function \mbox{$f \colon \mediumcube \to \R^{\dy}$}, its \emph{periodic extension} $\PeriodExt{f}{x}$ is the operator mapping $\Elltwo(\mediumcube, \Lebmeas^T; \R^{\dy})$ to $\Elltwo(\bigcube, \Lebmeas^T; \R^{\dy})$; considering $f = (f_1, \cdots, f_{\dy})$, the periodic extension acts component-wise on $f$ as $[\PeriodExt{f}{x}]_i \doteq f_i\left(x_1 - 4L \lfloor \frac{x_1}{4L}\rfloor, \cdots, x_{\dx}) - 4L \lfloor \frac{x_{\dx}}{4L}\rfloor \right)$. Such an operator allows us to define the \emph{periodic Sobolev space} $\Soboper\left(\mediumcube^T,\Lebmeas^T; \R^{\dy}\right)$ as the space of functions such that $\PeriodExt{f}{\cdot}$ belongs to $\Sobo(\bigcube^T,\Lebmeas^T; \R^{\dy})$. Therefore, $\Soboper\left(\mediumcube^T,\Lebmeas^T; \R^{\dy}\right)$ is a subspace of $\Sobo\left(\mediumcube^T,\Lebmeas^T; \R^{\dy}\right)$ consisting of functions whose $4L$-periodic extension is still $s$-times differentiable.

\paragraph{Fourier characterization of $\Soboper\left(\mediumcube^T,\Lebmeas^T; \R^{\dy}\right)$.}
A more practical representation of\\  $\Soboper\left(\mediumcube^T,\Lebmeas^T; \R^{\dy}\right)$ is given by means of Fourier basis: indeed, for each component $f_i(\cdot)$ in $f(\cdot) = [f_1(\cdot),\cdots, f_{\dy}(\cdot)]^{\top}$ there exists a unique (infinite-dimensional) vector $z_i$ indexed by $\mathbb{Z}^{\dx}$ with components in  $\mathbb{C}$ (thus, $z \in \mathbb{C}^{\mathbb{Z}^{\dx}}$) such that we have
\begin{subequations}\label{eq:fourier_rep}
\begin{align}
    f_i(x) &= \sum_{k \in \mathbb{Z}^{\dx}} z_{i,k} \exp\left\{ \i \frac{\pi}{2L} \innerprod{k}{x}{2} \right\}, \text{ and} \label{eq:fourier_f} \\
    \D f_i(x) &= \left(\i\frac{\pi}{2L} \right)^{|\alpha|}\sum_{k \in \mathbb{Z}^{\dx}} z_{i,k}\exp\left\{ \i \frac{\pi}{2L} \innerprod{k}{x}{2} \right\} \prod_{j=1}^{\dx}k_j^{\alpha_j} \label{eq:fourier_df}
\end{align}
    \end{subequations}
for any multi-index $\alpha = (\alpha_1,\cdots,\alpha_{\dx}) \in \Z^{\dx}$ such that $|\alpha|\leq s$~\cite[Proposition~A.5]{doumeche_physics-informed_2024}). Therefore, we obtain the following result. 
\begin{proposition}\label{prop:fourierSoboper}
    The periodic Sobolev space $\Sobo\left(\mediumcube^T,\Lebmeas^T; \R^{\dy}\right)$ can be characterized as
    \begin{equation}
        \Soboper\left(\mediumcube^T,\Lebmeas^T; \R^{\dy}\right) = \left\{ z   \in \bigoplus_{i=1}^{\dy} \mathbb{C}^{\mathbb{Z}^{\dx}} \, \bigg\vert\, \sum_{i=1}^{\dy}\sum_{k \in \mathbb{Z}^{\dx}} |z_{i,k}|^2 \norm{k}{2}^{2s} < \infty, \; z_{i,-k} = z_{i,k}^*   \right\}.\notag
    \end{equation}
\end{proposition}
\begin{proof}
We specify the expression for the Sobolev norm: the rest of the claim follows directly by~\cite[Proposition~A.5]{doumeche_physics-informed_2024}. \\
By simply considering the norm in $\Sobo(\mediumcube, \Lebmeas^T;\R^{\dy})$, we obtain by using~\eqref{eq:fourier_rep} that 
\begin{equation}
    \norm{f}{\Sobo(\mediumcube, \Lebmeas^T;\R^{\dy})}^2 = \sum_{i=1}^{\dy}\sum_{k \in \mathbb{Z}^{\dx}}|z_{i,k}|^2\sum_{|\alpha|\leq s} \left(\frac{\pi}{2L}\right)^{2|\alpha|}\prod_{j=1}^{\dx}k_j^{2\alpha_j}. \notag
\end{equation}
We now show that such a norm is equivalent to $\sum_{i=1}^{\dy}\sum_{k \in \mathbb{Z}^{\dx}} |z_{i,k}|^2 \norm{k}{2}^{2s}$. First, neglecting the sums over $i$ and $k$, we can find a pair of constants $0<\underline{\mathfrak{a}}<\overline{\mathfrak{a}}$ such that $\underline{\mathfrak{a}}\sum_{|\alpha|\leq s}\prod_{j=1}^{\dx} k_j^{2\alpha_j} \leq \sum_{|\alpha|\leq s}\left(\frac{\pi}{2L}\right)^{2|\alpha|}\prod_{j=1}^{\dx}k_j^{2\alpha_j} \leq \overline{\mathfrak{a}}\sum_{|\alpha|\leq s}\prod_{j=1}^{\dx} k_j^{2\alpha_j}$, so we can focus on the term $\sum_{|\alpha|\leq s}\prod_{j=1}^{\dx}k_j^{2\alpha_j}$. By an application of the multinomial theorem (see also~\cite{royer_brief_2020}), we can find constants $0<\underline{\mathfrak{b}}<\overline{\mathfrak{b}}$ such that \mbox{$\underline{\mathfrak{b}} \sum_{|\alpha|\leq s}\prod_{j=1}^{\dx} k_j^{2\alpha_j} \leq (1 + \norm{k}{2}^2)^s \leq \overline{\mathfrak{b}}\sum_{|\alpha|\leq s}\prod_{j=1}^{\dx} k_j^{2\alpha_j}$}. Finally, we can find another pair of non-negative constants $\underline{\mathfrak{c}}<\overline{\mathfrak{c}}$, again bounded away from zero, such that \mbox{$\underline{\mathfrak{c}} \norm{k}{2}^{2s} \leq (1 + \norm{k}{2}^2)^s \leq \overline{\mathfrak{c}} \norm{k}{2}^{2s}$} (for instance, $\underline{\mathfrak{c}}=1$ will do, and $\overline{\mathfrak{c}}$ can be found by upper-bounding the formula of the binomial theorem) -- see also~\cite[Chapter~2.1]{temam_navier-stokes_1995}.
\end{proof}

Finally, we also point out that the characterization of~\cref{prop:fourierSoboper} can be also more conveniently rewritten using a re-indexing of $z$: indeed, by~\cite[Proposition~A.7]{doumeche_physics-informed_2024}, there exists a one-to-one mapping $ \ell \in \mathbb{N} \mapsto k(\ell) \mathbb{Z}^{\dx}$ such that we can write $\phi_{\ell}(x) \doteq \exp\{\i \pi \innerprod{k(\ell)}{x}{2}/(2L) \}$. With this, we can express each component $f_i(x)$, for $i=1,\cdots,\dy$, as $f_i(x) = \sum_{\ell \in \mathbb{N}} z_{i,\ell}\phi_{\ell}(x)$. Thus, the space is characterized as follows:
\begin{equation}
    \Soboper\left(\mediumcube^T,\Lebmeas^T; \R^{\dy}\right) = \left\{ z   \in \bigoplus_{i=1}^{\dy} \mathbb{C}^{\mathbb{Z}^{\dx}} \, \bigg\vert\, \sum_{i=1}^{\dy}\sum_{\ell \in \mathbb{N}} |z_{i,\ell}|^2\ell^{2s/\dx}  < \infty \right\}. \label{eq:FourierNatural}
\end{equation}

\paragraph{Extension results for $\Sobo(\X^T,\Lebmeas^T;\R^{\dy})$.}
Along the lines of the analysis carried out in~\cite[Proposition~A.6]{doumeche_physics-informed_2024}, it is possible to show that the characterization given in~\cref{prop:fourierSoboper} holds also for $\Sobo(\X^T,\Lebmeas^T;\R^{\dy})$. Indeed, one can use the Sobolev extension Theorem~\citep[Chapter~VI]{stein_singular_1970} to linearly extend every function $f$ in $\Sobo(\X^T,\Lebmeas^T;\R^{\dy})$ to the function $\ExtOp{f}{x}$, whose $i$-th component is given by $[\ExtOp{f}{x}]_i = \sum_{k\in\mathbb{Z}^{\dx}} z_{i,k}\exp\left\{ \iota \frac{\pi}{2L}\innerprod{k}{x}{2} \right\}$. Thus, $\ExtOp{f}{x}$ belongs to $\Soboper(\mediumcube^T,\Lebmeas^T;\R^{\dy})$ and is such that \mbox{$||\ExtOp{f}{\cdot}||_{\Soboper(\mediumcube^T,\Lebmeas^T;
R^{\dy})}^2 \leq \mathfrak{C}_{s,\X} \norm{f}{\Sobo(\X^T,\Lebmeas^T;
R^{\dy})}^2$} for some constant $\mathfrak{C}_{s,\X}$, yielding the imbedding $\Sobo(\X^T,\Lebmeas^T;
R^{\dy}) \hookrightarrow \Soboper(\mediumcube^T,\Lebmeas^T;
R^{\dy})$. This allows us to leverage the structure of the periodic Sobolev space to derive a key property of $\Sobo(\X^T,\Lebmeas^T;
R^{\dy})$, namely the trajectory $(C,2)$-hypercontractivity (see~\cref{sec:hypercontractivity}).
\paragraph{From Lebesgue measure to $\jointProb$.}
In the preceding paragraphs we focused on the Lebesgue measure defined on $\X$. In such a set-up, it holds that $\norm{f}{\Elltwo(\X^T, \Lebmeas^T; \R^{\dy})}^2 = \norm{f}{\Elltwo(\X, \Lebmeas^T; \R^{\dy})}^2$, and in such a space  the functions $\exp\{\iota \pi \innerprod{k}{x}{2}/(2L)\}$ are orthonormal: this enables the application of Parseval's Theorem, leading to the characterization in~\cref{prop:fourierSoboper}. Such a property gets lost as soon we consider the distribution $\jointProb$; nevertheless, thanks to~\cref{ass:densities}, we obtain that $\Sobo(\X^T,\jointProb;\R^{\dy}) \hookrightarrow \Sobo(\X^T, \Lebmeas^T; \R^{\dy})$; therefore, when needed, analysis can be carried out in $\Sobo(\X^T,\Lebmeas^T;\R^{\dy})$ and the results carry over to the space of interest $\Sobo(\X^T,\jointProb;\R^{\dy})$.

\section{On covering numbers}\label{sec:covering_numbers_results}
This section focuses on the complexity measure for the hypothesis space, which plays a crucial role in the excess risk bounds derived in this paper. After recalling the definition of covering number and metric entropy, we present how classical results, stated for scalar function spaces, extend to our vector-valued set-up. This section culminates with the derivation of the covering number of the effective hypothesis space $\Freg$.

\subsection{From scalar to vector-valued hypothesis spaces}\label{sec:scalar2vec_covering}
To quantify the complexity of a certain hypothesis space $\HH$, we will resort to its \emph{covering number}, which is defined as follows:
\begin{definition}
    Let $\mathcal{S}$ be a subset of a metric space $\HH_{\mathfrak{a}}$ with distance function induced by its norm $\norm{\cdot}{\HH_{\mathfrak{a}}}$. The \emph{$\varepsilon$-cover} of $\mathcal{S}$ is a set $\{f^1, \cdots, f^N\} \subset \mathcal{S}$ such that, for each $f \in \mathcal{S}$, there exists some $\ell=1,\dots,N$ such that $\norm{f-f^{\ell}}{\HH_{\mathfrak{a}}} \leq\varepsilon$. The \emph{$\varepsilon$-covering number}, denoted by $\N{\mathfrak{a}}{\mathcal{S}}{\varepsilon}$, is the cardinality of the smallest $\varepsilon$-cover.\\
    Additionally, $\log \N{\mathfrak{a}}{\mathcal{S}}{\varepsilon}$ is called \textit{metric entropy} of $\mathcal{S}$ at resolution $\varepsilon$. 
\end{definition}

There is a vast literature on bounds for covering numbers: see, e.g.,~\cite[Chapter~5]{cucker_learning_2007},~\cite[Chapter~5]{wainwright_high-dimensional_2019}, as well as~\cite{zhou_covering_2002,guo_covering_2002,zhou_capacity_2003,wang_estimation_2009}. However, results are typically presented for \emph{scalar} function spaces. In this section, we will adapt covering number estimates to our set-up involving multi-output function spaces. 
\begin{proposition}\label{prop:vectorvaluedcovering}
    Let us consider the metric $\norm{\cdot}{\Ellinf(\X^T;\R^{\dy})}$, and let $\overline{\HH} =  \bigoplus_{i=1}^{\dy} \HH$ be a subset of the Sobolev space $\Sobo(\X^T,\jointProb;\R^{\dy})$. Assume there exists a $\varepsilon/\dy$-cover of $\HH$ in the direct sum with the metric of $\Ellinf(\X^T;\R)$, and let $\N{\infty}{\HH}{\varepsilon/\dy}$ be its covering number: then it holds that
    \begin{equation}
        \N{\infty}{\overline{\HH}}{\varepsilon} \leq \left(\N{\infty}{\HH}{\varepsilon/\dy} \right)^{\dy}.\notag
    \end{equation}
\end{proposition}
\begin{proof}
For it to be a $\varepsilon$-cover of $\overline{\HH}$, it has to hold that, for every $f\in\overline{\HH}$ there exists an element $f'$ in the cover such that $\norm{f-f'}{\Ellinf(\X^T;\R^{\dy})}\leq \varepsilon$. Similarly, assume there is a $\breve{\varepsilon}$-cover of $\HH$ such that, for any arbitrary element $h\in\HH$, we have a function $h'$ in the cover such that $\norm{h-h'}{\Ellinf(\X^T;\R)} \leq \breve{\varepsilon}$. Now, by the definition of the vector-valued version of the infinity norm (see~\cref{sec:function_spaces}), we have
\begin{align}
    \norm{f-f'}{\Ellinf(\X^T;\R^{\dy})} = \sup_{x \in \X} \norm{f(x)-f'(x)}{2}  \leq  \sum_{i=1}^{\dy}\sup_{x\in\X} |f_i(x)-f'_i(x)| \leq \dy \breve{\varepsilon}.\notag
\end{align}
Thus, letting $\breve{\varepsilon} = \varepsilon/\dy$, we obtain an $\varepsilon$-approximation for the covering of $\overline{\HH}$. 
The claim follows by observing that, by construction of the direct sum of scalar hypothesis spaces, the cover of $\overline{\HH}$ is given by the Cartesian product of the covers of $\HH$.
\end{proof}

\subsection{Covering numbers for vector-valued hypothesis spaces}\label{sec:covering_vector}
We conclude this section by adapting standard covering number bounds for scalar function spaces to vector-valued ones. We start from~\cite[Theorem~5.3]{cucker_learning_2007}.
\begin{proposition}\label{thm:cuckerzhoucovering}
    Let $\overline{\HH} = \bigoplus_{i=1}^{\dy} \HH$ be a finite-dimensional Banach space of dimension $E$, and let $\mathcal{B}_R$ be the set such that $\mathcal{B}_R = \{f \in \overline{\HH} \,|\, \norm{f}{\overline{\HH}} \leq R\}$. Then it holds that
    \begin{equation}
        \N{\overline{\HH}}{\mathcal{B}_R}{\varepsilon} \leq \left( \frac{2R\dy}{\varepsilon} +1 \right)^{E\cdot \dy}.\notag
    \end{equation}
\end{proposition}
\begin{proof}
    This result follows by taking~\cite[Theorem~5.3]{cucker_learning_2007} and extending it according to the construction in~\cref{prop:vectorvaluedcovering}.
\end{proof}

We now proceed by characterizing the covering number for a ball in the Sobolev space.
\begin{lemma}\label{lemma:cucker_smale_vector}
Let $\mathcal{B}_R$ be a ball of radius $R$ in the Sobolev space $\Sobo(\X^T,\jointProb;\R^{\dy})$ satisfying~\cref{ass:SobolevOrder}, where $\X$ is a domain with locally-Lipschitz boundary. Then the metric entropy of $\mathcal{B}_R$ satisfies
\begin{equation}
    \log \N{\infty}{\mathcal{B}_R}{\varepsilon} \leq \Cc' \dy^{\frac{2s+\dx}{2s}}\left(\frac{R}{\varepsilon} \right)^{\frac{\dx}{s}}. \notag
\end{equation}
\end{lemma}
\begin{proof}

Let us start from the \emph{scalar-valued} Sobolev space $\Sobo(\X^T,\jointProb;\R)$ and let ${\scriptscriptstyle \mathcal{B}}_R$ be its ball of radius $R$. If the input domain $\X$ is smooth, then~\cite[Chapter~I,~Section 6,~Proposition~6]{cucker_mathematical_2002} claims that, for some positive constant $c$, we have
\begin{equation}
    \log \N{\infty}{{\scriptscriptstyle \mathcal{B}}_R}{\varepsilon} \leq \left(\frac{cR}{\varepsilon}\right)^{\frac{\dx}{s}} + 1 \leq \Cc' \left(\frac{R}{\varepsilon}\right)^{\frac{\dx}{s}} \label{eq:claim_covering_cuckersmale}
\end{equation}
 for some other constant $\Cc'$ that is big enough to absorb also the contribution of the ``+1" in~\eqref{eq:claim_covering_cuckersmale}. 

Such a result relies on a bound on the entropy number of the embedding $\Sobo(\X^T,\jointProb;\R) \hookrightarrow \Ellinf(\X^T;\R)$ given by~\cite[Section~3.3]{edmunds_function_1996} (using their notation, we are looking at $F_{2,q_1}^{s} \hookrightarrow F_{\infty,q_2}^0$). However, in~\cite[Section~3.5]{edmunds_function_1996}, such a result is extended to non-smooth domains --- specifically, to the \emph{minimally regular} ones~\citep[Section 2.5, Definition 2]{edmunds_function_1996}, which include domains with locally-Lipschitz boundary as special cases. By the way, the theorem in~\cite[Section~3.5]{edmunds_function_1996} is stated for the function spaces $B_{pq}^{s}$, but the results holds also for the spaces $F_{pq}^{s}$: see the argument presented in the proof of the Theorem in~\cite[Section 3.3.2]{edmunds_function_1996}.\\

Thus, overall,~\eqref{eq:claim_covering_cuckersmale} holds also for our choice of $\X$. The proof is concluded by invoking~\cref{prop:vectorvaluedcovering} to extend~\eqref{eq:claim_covering_cuckersmale} to the vector-valued case. Specifically, we have that $\N{\infty}{\mathcal{B}_R}{\varepsilon} \leq \left( \N{\infty}{{{\scriptscriptstyle \mathcal{B}}_R}}{\varepsilon/\dy}\right)^{\dy}$, and taking logarithms we obtain $\log \N{\infty}{\mathcal{B}_R}{\varepsilon} \leq \dy  \log \N{\infty}{{{\scriptscriptstyle \mathcal{B}}_R}}{\varepsilon/\dy}$, and substituting~\eqref{eq:claim_covering_cuckersmale} leads to the final claim.
\end{proof}

We conclude this section by deriving a bound for the metric entropy of the effective hypothesis space $\Freg$ presented in~\eqref{eq:effectiveHypothesis}. The idea is to leverage the ellipticity of the differential operator (\cref{ass:elliptic}) and find a ball in the Sobolev space that approximates $\Freg$, and then invoke~\cref{lemma:cucker_smale_vector} to bound its metric entropy. 
\begin{proposition}\label{prop:coveringFreg}
Let~\cref{ass:SobolevOrder,ass:elliptic} hold and consider the effective hypothesis space 
\begin{align}
    \Freg = \left\lbrace f \in \Fstar \,\vert\, \norm{\Der(f)}{\Elltwo(\X^T,\jointProb;\R^{\dy})}^2 \leq \rho \right\rbrace. \tag{\ref{eq:effectiveHypothesis}}
\end{align}
Then, for some positive constant $\Cc$ not depending on $\varepsilon$ and $\rho$, we have that 
\begin{align}
        \log \N{\infty}{\Freg}{\varepsilon} \leq \Cc\dy^{\frac{2s+\dx}{2s}}\left( \frac{\sqrt{\rho}}{\varepsilon}\right)^{\frac{\dx}{s}}. \notag
    \end{align}
\end{proposition}

\begin{proof}
Similarly to~\cref{lemma:cucker_smale_vector}, we prove the result for scalar-valued function spaces and then invoke~\cref{prop:vectorvaluedcovering} to obtain the claim for the vector-valued case.\\

We start by recalling an important property of elliptic operators derived from~\cite[Chapter 6.3,~Theorem 5]{evans_partial_2010} that will allow us to find the Sobolev ball centered at $\fstar$ containing $\Freg$. Specifically, for some positive constant $C_e$ and $f \in \Fstar$, we have that \begin{equation}\label{eq:propertyElliptic}
        \norm{f}{\Sobo(\X^T,\jointProb;\R^{\dy})} \leq C_e \left( \norm{\Der(f)}{\Elltwo(\X^T,\jointProb;\R^{\dy})} + \norm{f}{\Elltwo(\X^T,\jointProb;\R^{\dy})}\right).
    \end{equation}
    Let $\HH$ be the scalar-valued version of $\Freg$. We decompose $\HH$ into the null-space of $\Der$ and its orthogonal complement, obtaining $\HH = \ker(\Der) \oplus \ker(\Der)^{\bot}$. Accordingly, any $f \in \HH$ can be written as $f = g+h$, with $g \in \ker(\Der)$ and $h \in \ker(\Der)^{\bot}$, and the constraint \mbox{$\norm{\Der(f)}{\Elltwo(\X^T,\jointProb;\R)}^2 \leq \rho$} reduces to \mbox{$\norm{\Der(h)}{\Elltwo(\X^T,\jointProb;\R)}^2 \leq \rho$}.\\

We first focus on the subspace $\ker(\Der)^{\bot}$ and prove a preliminary result that allows us to rewrite~\eqref{eq:propertyElliptic}. Specifically, we want to show that, for some positive constant $C_l$ and any $h \in \ker(\Der(f))^{\bot}$,  \begin{equation}\label{eq:otherEllipticConsequence}
\norm{h}{\Elltwo(\X^T,\jointProb;\R)} \leq C_l \norm{\Der(h)}{\Elltwo(\X^T,\jointProb;\R)}.
\end{equation} 
We show this by contradiction. If the claim were false, then we could find a sequence $\lbrace h_k \rbrace_k$ in $\ker(\Der(f))^{\bot}$ such that $\norm{h_k}{\Elltwo(\X^T,\jointProb;\R)} = 1$ and $\norm{\Der(h_k)}{\Elltwo(\X^T,\jointProb;\R)} = 1/k$. By~\eqref{eq:propertyElliptic}, the sequence $\lbrace h_k \rbrace_k$ is uniformly bounded, which implies that there exists a subsequence $\lbrace h_{k_{\ell}} \rbrace_{\ell} \subset \lbrace h_k \rbrace_k$ that converges weakly to some $h \in \ker(\Der)^{\bot}$~\citep[Appendix D, Theorem 3]{evans_partial_2010} (see also~\cite[Theorem~3.6]{adams_sobolev_2003}). However, this implies that $\Der(h_{k_{\ell}}) \to \Der(h) =0$, and $h\neq 0$ because $\norm{h_k}{\Elltwo(\X^T,\jointProb;\R)}=1$ by construction. Hence, we would obtain that $h \in \ker(\Der)$, which is absurd.\\

Thanks to~\eqref{eq:otherEllipticConsequence}, we can re-write~\eqref{eq:propertyElliptic} for $h \in \ker(\Der)^{\bot}$ as $\norm{h}{\Sobo(\X^T,\jointProb;\R)} \leq C_e(C_l + 1)\sqrt{\rho}$,
showing that any $h \in \ker(\Der)^{\bot}$ is contained in a ball of the Sobolev space of radius proportional to~$\sqrt{\rho}$. Finally, we can determine the covering number for $\ker(\Der)^{\bot}$ by invoking~\eqref{eq:claim_covering_cuckersmale}.\\

We now proceed by focusing on $\ker(\Der)$. Ellipticity stated in~\cref{ass:elliptic} implies that such a subspace is finite-dimensional. Additionally, by~\cref{thm:sobolev_imbedding} (Sobolev imbedding), there exists a uniform positive constant $\tilde{B}$ such that $\norm{f}{\Ellinf(\X^T;\R)} \leq \tilde{B}$ for every function $f \in \HH$, which is a closed and convex subset of the Sobolev space $\Sobo(\X^T,\jointProb;\R)$ --- thus, also $\norm{f}{\Ellinf(\X^T;\R)} \leq \tilde{B}$. In light of these considerations, $\ker(\Der)$ belongs a ball of radius $\tilde{B}$ of a finite-dimensional Euclidean space with dimension $\dim\ker(\Der)$, and its covering number can be calculated according to~\cite[Theorem 5.3]{cucker_learning_2007}. \\

At this point, since the decomposition of $\HH$ into $\ker(\Der)$ and $\ker(\Der)^{\bot}$ is orthogonal, the covering number of $\HH$ is given by the product of the covering numbers of the two subspaces. Taking logarithms, we obtain, for  sufficiently large constants $\mathfrak{c}$ and $\Cc$,
\begin{equation}
    \log \N{\infty}{\HH}{\varepsilon} \leq \mathfrak{c}\left( \left( \frac{\sqrt{\rho}}{\varepsilon}\right)^{\frac{\dx}{s}} + \dim\ker(\Der)\log\left(\frac{\tilde{B}}{\varepsilon}\right) \right) \leq \Cc \left( \frac{\sqrt{\rho}}{\varepsilon}\right)^{\frac{\dx}{s}},
\end{equation}
where the contribution of $\ker(\Der)$ is incorporated in $\Cc$ as the logarithmic term is negligible for small~$\varepsilon$.\\

To conclude the proof, we proceed along the lines of the proof for~\cref{lemma:cucker_smale_vector} and obtain the final claim by invoking~\cref{prop:vectorvaluedcovering}.

\end{proof}

\section{Properties of the hypothesis spaces}\label{sec:properties_hypothesis_spaces}
We now demonstrate some useful properties of our hypothesis spaces $\F$ and $\Freg$. We will start by focusing on convexity of the spaces and on the boundedness of the functions that belong to them; next, we proceed by showing that our effective hypothesis space $\Freg$ satisfies the small-ball condition introduced in~\cite{mendelson_learning_2014} at least on a subset of interest for the following proofs.

\subsection{Convexity and $B$-boundedness}\label{sec:convexity_boundedness}

\begin{lemma}\label{lemma:propF}
    The hypothesis space $\F$ presented in~\eqref{eq:hypothesisSpace} is convex -- i.e., for any $0 \leq \xi \leq 1$ and any $f,h \in \F$, we have that $\xi f + (1-\xi)h$ still belongs to $\F$. Additionally, there exists a constant $B>0$ such that every $f\in \F$ is $B$-bounded, i.e., $\norm{f}{\infty} \leq B$ for all $f\in \F$. 
\end{lemma}
\begin{proof}
Given an arbitrary Hilbert space $\HH$ with norm induced by the inner product $\innerprod{\cdot}{\cdot}{\HH}$, the parallelogram law yields
\begin{equation}\label{eq:parallelogram}
  \norm{x-y}{\HH} +\norm{x+y}{\HH} =  2\norm{x}{\HH} + 2\norm{y}{\HH} \Rightarrow  \norm{x-y}{\HH} \leq 2\norm{x}{\HH} + 2\norm{y}{\HH}
\end{equation}
This leads to showing that Hilbert spaces (then, also the Sobolev space $\Sobo(\X^T,\jointProb;\R^{\dy})$) are uniformly convex~\citep[Definition~1.20 and Theorem 3.5]{adams_sobolev_2003}: therefore, the first claim follows by noting that $\F$ is a convex subset of the Sobolev space. The second claim is a consequence of the imbedding presented in~\cref{thm:sobolev_imbedding}(b): indeed, there exist a constant $\mathfrak{C}_{\infty}$ such that $\norm{f}{\Ellinf(\X^T,\jointProb;\R^{\dy})} \leq \mathfrak{C}_{\infty}\norm{f}{\Elltwo(\X^T,\jointProb;\R^{\dy})} \leq \mathfrak{C}_{\infty}\RSob \doteq B$ by definition of $\F$ in~\eqref{eq:hypothesisSpace}.
\end{proof}

\subsection{$(C,2)$-hypercontractivity}\label{sec:hypercontractivity}
We start by providing the general definition.
\begin{definition}\label{def:hypercontractivity}
    For a given hypothesis space $\HH$ of vector-valued functions and uniform constants $C>0$ and $\alpha \in [1,2]$, the tuple \mbox{$(\HH, \Prob_{\X})$} is \emph{$(C,\alpha)$-hypercontractive} if it holds that, for every $f \in \HH$, 
    \begin{align}
        \E{\frac{1}{T}\sumt  \norm{f(X_t)}{2}^4}{\jointProb}  \leq C \left( \E{\frac{1}{T}\sumt  \norm{f(X_t)}{2}^2}{\jointProb} \right)^\alpha.
    \end{align}
\end{definition}

In this paper, we will be focusing on the corner case $\alpha=2$, which implies that the \textit{small-ball condition}~\citep{mendelson_learning_2014} holds:
\begin{lemma}
    Let $(\HH,\jointProb)$ be $(C,2)$-hypercontractive for a suitable positive constant $C$. Then it holds that, for any $f,h$ in $\HH$, there exist $\varepsilon$ and $\xi$ such that
    \begin{equation}
        \jointProb\left( \frac{1}{T}\sumt \norm{f(X_t) - h(X_t)}{2} \geq \xi \norm{f-h}{\Elltwo(\X^T,\jointProb;\R^{\dy})} \right) \geq \varepsilon.\notag
    \end{equation}
\end{lemma}
\begin{proof}
First, as pointed out in the discussion in~\cref{sec:function_spaces}, note that~\cref{def:hypercontractivity} can be written as $\norm{f}{\mathscr{L}^4(\X^T,\jointProb;\R^{\dy})}^4 \leq C \norm{f}{\Elltwo(\X^T,\jointProb;\R^{\dy})}^{2\cdot 2}$. Next, let \begin{equation}
c_{2,4} \doteq \norm{f-h}{\Elltwo(\X^T,\jointProb;\R^{\dy})}/\norm{f-h}{\mathscr{L}^4(\X^T,\jointProb;\R^{\dy})}\notag
\end{equation}. 
By the Paley-Zygmund inequality~\citep[Corollary~3.3.2]{de_la_pena_decoupling_1999}, we have that
\begin{equation}
    \jointProb\left(\frac{1}{T}\sumt \norm{f(X_t) - h(X_t)}{2} \geq u \norm{f-h}{\Elltwo(\X^T,\jointProb;\R^{\dy})} \right) \geq [(1-u^2)c_{2,4}^2]^2 \geq (1-u^2)^2/C,\notag
\end{equation}
where the last step follows by~\cref{def:hypercontractivity}. Conclusion follows as soon as we let $u \leftrightarrow \xi$ and $(1-u^2)^2/C \leftrightarrow \varepsilon$.
\end{proof}

Drawing inspiration from~\cite[Proposition~3.4.4]{ziemann_statistical_2022}, we now prove that the hypothesis space $\F$ satisfies this particular kind of hypercontractivity on a subset of interest 
for~\cref{thm:new_thm312}, which will allow us to quantify the probability of the lower isometry event. 
\begin{theorem}\label{thm:Fhypercontractive}
    Let~\cref{ass:densities,ass:SobolevOrder,ass:containment,ass:S-persistence} hold. Given some $r > 0$, consider the set $\partial B(r) = \{f \in \F | \norm{f-\fstar}{\Elltwo(\X^T,\jointProb;\R^{\dy})}^2 = r^2\}$ and let $\Fe$ be its cover with balls of radius $\epsilon$. 
    Furthermore, let $\tiRSob \propto \RSob/\kappau$, where $\kappau$ is as per~\cref{ass:densities}. Then, we have that the covering number of $\partial B(r)$ (in other words, the cardinality of $\Fe$) satisfies
\begin{equation}\label{eq:covering_claim}
        \N{\infty}{\partial B(r)}{\epsilon} \leq \left( \frac{8\tiRSob \me^{s/\dx}\dy}{\epsilon} + 1 \right)^{\me\dy}
    \end{equation}
with $\me$ being the smallest integer solution of
\begin{equation}
    m \geq \left( \frac{16 \tiRSob\dx}{(2s-\dx)\epsilon^2} \right)^{\dx/(2s-\dx)}. \notag
\end{equation}
Additionally, as long as $\epsilon \leq \inf_{f\in\partial B(r/\sqrt(\kappau))}\norm{f}{\Elltwo(\X^T,\Lebmeas^T;\R^{\dy})}$, the tuple $(\Fe,\jointProb)$ is $(C(\epsilon),2)$-hypercontractive, with
\begin{equation}\label{eq:Ceps}
    C(\epsilon) \propto \left( \frac{\Lebmeas(\X)}{32} + 8\Lebmeas(\X)\me^2 \left( \frac{\Lebmeas(\X)}{8} + 2 \right)^2 \right).
\end{equation}
\end{theorem}
\begin{proof}
We will make use of the construction presented in~\cref{sec:periodicSobolev} and focus on the Fourier characterization of $\Sobo(\X^T,\Lebmeas^T;\R^{\dy})$: the result carries over to $\Sobo(\X^T,\jointProb;\R^{\dy})$ by the imbedding $\Sobo(\X^T,\jointProb;\R^{\dy}) \hookrightarrow \Sobo(\X^T,\Lebmeas^T;\R^{\dy})$. As such, we will consider the space \\$\Ftmp \allowbreak \doteq \left\{ f \in \Sobo(\X^T,\Lebmeas^T;\R^{\dy}) \,\mid\, \norm{f}{\Sobo(\X^T,\Lebmeas^T;\R^{\dy})}^2 \leq \RSob^2/\kappau  \right\}$ such that $\F \subseteq \Ftmp$: indeed, the condition on the norm reads as $\kappau \norm{f}{\Sobo(\X^T,\Lebmeas^T;\R^{\dy})}^2 \geq \norm{f}{\Sobo(\X^T,\jointProb;\R^{\dy})}^2$ by~\cref{ass:densities}. According to the construction in~\cref{sec:periodicSobolev} and the Fourier decomposition in terms of the basis functions \mbox{$\phi_{\ell}(x) = \exp\{\i \pi \innerprod{k(\ell)}{x}{2}/(2L)\}$}, we can leverage~\eqref{eq:FourierNatural} to write
\begin{equation}\label{eq:Fmu}
    \Ftmp = \left\{ f(x) = \sum_{\ell \in \mathbb{N}} z_{\ell}\phi_{\ell}(x),\, z_{\ell} \in \R^{\dy} \:\forall \ell \in \mathbb{N} \:\bigg\vert\,  \sum_{\ell \in \mathbb{N}} \frac{\norm{z_{\ell}}{2}^2}{\ell^{-2s/\dx}} \leq \tiRSob^2,   \right\}, 
\end{equation}
where $\tiRSob^2$ is equal to $\RSob^2/\kappau$ times a multiplicative constant that can be retrieved along the lines of the proof of~\cref{prop:fourierSoboper}, but we do not specify  it because it does not affect the main message of our results. Additionally, we will consider $\tir ^2 \doteq r^2/\kappau$ such that $B(\tir ) \supset B(r)$. Note that, due to such an inclusion, $\N{\infty}{\partial B(r)}{\epsilon} \leq \N{\infty}{\partial B(\tir)}{\epsilon}$.

\paragraph{Covering.} We start by proving the result on the covering of $\partial B(\tir )$. The idea is to find a suitable finite-dimensional approximation of the space $\Ftmp$, compute its covering number using~\cref{thm:cuckerzhoucovering}, and then express $\N{\infty}{\partial B(\tir )}{\eps}$ in terms of such a cover.\\
We first seek an approximation of $\Ftmp$ at resolution $\eps/4$ by considering, for some $m\in \Z$, the finite-dimensional space
\begin{equation}
    \Ftmp^m = \left\{ f(x) = \sum_{\ell=1}^m z_{\ell}\phi_{\ell}(x),\, z_{\ell} \in \R^{\dy}  \:\bigg\vert\,  \sum_{\ell \in \mathbb{N}} \frac{\norm{z_{\ell}}{2}^2}{\ell^{-2s/d}} \leq \tiRSob^2   \right\}. \notag
\end{equation}
Now, fix $f \in \Ftmp$ with coordinates $\{z_{\ell}\}_{\ell \in \mathbb{N}}$ and let $f'$ be its projection onto $\Ftmp^m$. Then the following holds:
\begin{align}
    \norm{f-f'}{\Ellinf(\X^T;\R^{\dy})} &= \norm{\sum_{\ell = m+1}^{\infty} z_{\ell}\phi_{\ell}}{\Ellinf(\X^T;\R^{\dy})} \notag \\ &\leq \norm{\sqrt{\sum_{\ell=m+1}^{\infty} \frac{\norm{z_{\ell}}{2}^2}{\ell^{-2s/\dx}}}\sqrt{\sum_{\ell=m+1}^{\infty} \ell^{-2s/\dx}|\phi_{\ell}|^2}}{\Ellinf(\X^T;\R^{\dy})} \notag\\
    &\stackrel{\eqref{eq:Fmu}}{\leq} \tiRSob \norm{\sqrt{\sum_{\ell=m+1}^{\infty}\ell^{-2s/\dx}|\phi_{\ell}|^2}}{\Ellinf(\X^T;\R^{\dy})} \notag \\ &\leq \tiRSob\sqrt{\sum_{\ell=m+1}^{\infty} \ell^{-2s/\dx} \norm{|\phi_{\ell}|^2}{\Ellinf(\X;\R)}} \leq \tiRSob\sqrt{\sum_{\ell=m+1}^{\infty} \ell^{-2s/\dx}},\label{eq:Fmtmpbound}
\end{align}
where the first inequality is given by the Cauchy-Schwarz one, and the last one follows by definition of the basis $\phi_{\ell}(\cdot)$.
We now use the integral test~\citep[Chapter~6,~Exercise 8]{rudin_principles_1976} to upper-bound the last expression. Let $\p \doteq 2s/\dx$ to simplify notation, and note that $\p >1$ by~\cref{ass:SobolevOrder}. Then we have that
\begin{align}
    \sum_{\ell=m+1}^{\infty} \left(\frac{1}{\ell}\right)^{\p} \leq \sum_{\ell=m+1}^{\infty}\int_{\ell-1}^{\ell}\left(\frac{1}{\xx}\right)^{\p}d\xx = \int_{m}^{\infty} \left(\frac{1}{\xx}\right)^{\p}d\xx = \frac{1}{\p-1}\frac{1}{m^{\p-1}}. \notag
\end{align}
Plugging such a bound in~\eqref{eq:Fmtmpbound}, to ensure that $\norm{f-f'}{\Ellinf(\X^T;\R^{\dy})} \leq \eps/4$ for any $f\in\Ftmp$ we then require that
\begin{equation}
    \tiRSob\sqrt{\frac{1}{\p-1}\frac{1}{m^{p-1}}} \leq \eps/4 \Longleftrightarrow m \geq \left( \frac{16 \tiRSob^2 \dx}{(2s-\dx)\eps^2} \right)^{\frac{\dx}{2s-\dx}}.\label{eq:meps}
\end{equation}
Thus, we take $\me$ the smallest integer $m$ satisfying~\eqref{eq:meps} to have the approximation of $\Ftmp$ at resolution $\eps/4$.\\
We can now proceed by constructing the covering for $\Ftmp^{\me}$ at resolution $\eps/4$. We start by characterizing such a space in terms of the coefficients $\{z_{\ell}\}_{\ell=1}^{\me}$ by considering
\begin{equation}
    \mathcal{Z}^{\me} \doteq \left\{ z_{\ell} \in \R^{\dy},\, \ell=1,\cdots,\me \:\bigg\vert \:\sum_{\ell=1}^{\me} \frac{\norm{z_{\ell}}{2}^2}{\ell^{-2s/\dx}} \leq \tiRSob^2. \notag \right\}.\notag
\end{equation}
This is a ball of radius $\tiRSob$ in a finite-dimensional Banach space with a weighted 2-norm, which we denote by $\norm{\cdot}{w}$. With such a norm, by~\cref{thm:cuckerzhoucovering}, its covering number with balls of radius $\bar{\epsilon}$ satisfies
\begin{equation}\label{eq:finite_dim_w_metric_covering}
\N{w}{\mathcal{Z}^{\me}}{\bar{\epsilon}} \leq (2\tiRSob\dy/\bar{\epsilon}+1)^{\me \dy} \doteq N.
\end{equation} Now, denote the elements of the optimal covering of $\mathcal{Z}^{\me}$ as $\{z_{\ell}^{1}, \cdots, z_{\ell}^N\}_{\ell=1,\cdots,\me}$. These identify a further approximation $\Ftmp^{\me,N} \subset \Ftmp^{\me}$ defined as 
\begin{equation}
    \Ftmp^{\me,N} \doteq \left\{ \sum_{\ell=1}^{\me} z_{\ell}^{1}\phi_{\ell}, \cdots, \sum_{\ell=1}^{\me} z_{\ell}^{N}\phi_{\ell}  \right\}.
    \notag
\end{equation}
We can use this construction to characterize the radius $\bar{\epsilon}$. Let $f'(\cdot) = \sum_{\ell=1}^{\me} z_{\ell}'\phi_{\ell}(\cdot)$ be an arbitrary function in $\Ftmp^{\me}$. Then we have that
\begin{align}
    \min_{n=1,\cdots,N} &\norm{f' - \sum_{\ell=1}^{\me} z_{\ell}^{n}\phi_{\ell}}{\Ellinf(\X^T;\R^{\dy})} = \min_{n=1,\cdots,N} \norm{\sum_{\ell=1}^{\me} (z_{\ell}' - z_{\ell}^{n})\phi_{\ell}}{\Ellinf(\X^T;\R^{\dy})}\notag\\
    &\leq \min_{n=1,\cdots,N}\norm{\sqrt{\sum_{\ell=1}^{\me} \frac{\norm{z_{\ell}'-z_{\ell}^{n}}{2}^2}{\ell^{-2s/\dx}}} \sqrt{\sum_{\ell=1}^{\me} \ell^{-2s/\dx} |\phi_{\ell}|^2}  }{\Ellinf(\X^T;\R^{\dy})} \notag \\
    &\stackrel{\eqref{eq:finite_dim_w_metric_covering}}{\leq} \bar{\epsilon} \sqrt{\sum_{\ell=1}^{\me} \left(\frac{1}{\ell} \right)^{\p}} \leq  \bar{\epsilon} \me^{\p/2}.\notag 
\end{align}
Therefore, if we take $\bar{\epsilon} \leq \eps/(4\me^{\p/2})$ we obtain the $\eps/4$-cover of $\Ftmp^{\me}$ we seek. This leads to the claim that the covering number of $\Ftmp^{\me}$ satisfies
\begin{equation}
    \N{\infty}{\Ftmp^{\me}}{\frac{\eps}{4}} \leq \left( \frac{8\tiRSob\me^{\p/2}\dy}{\eps} + 1\right)^{\dy\me}.
\end{equation}
This part of the proof is concluded by converting the covering $\Ftmp^{\me,N}$ into an \emph{exterior cover} of the set $\partial B(\tir ) \subset \F$ --- that is, its elements cover $\partial B(\tir )$ but they are required to belong to $\F$ and not necessarily to $\partial B(\tir )$. Indeed, with the construction carried out so far, for each $f \in \partial B(\tir )$ we can identify $f^{\prime} \in \Ftmp^{\me}$ such that $\norm{f-f^{\prime}}{\Ellinf(\X^T;\R^{\dy})}\leq \eps/4$, and a $f^{\prime\prime} \in \Ftmp^{\me,N}$ such that $\norm{f^{\prime}-f^{\prime\prime}}{\Ellinf(\X^T;\R^{\dy})}\leq \eps/4$: thus, by the triangle inequality, $\norm{f-f^{\prime\prime}}{\Ellinf(\X^T;\R^{\dy})}\leq \eps/2$, implying that $\Ftmp^{\me,N}$ is an exterior cover of $\partial B(\tir )$ of resolution $\eps/2$. The final claim follows by applying~\cite[Exercise~4.2.9]{vershynin_high-dimensional_2024}.

\paragraph{Hypercontractivity.} We now prove $(C({\eps}),2)$-hypercontractivity of the tuple $(\Fe,\Lebmeas^T)$ --- the same claim will hold also for $(\Fe,\jointProb)$ by multiplying $C_{\eps}$ by a constant not depending on $\eps$ and is thus not relevant to our analysis. \\
We start by showing that $\Ftmp^{\me}$
satisfies the hypercontractivity condition (\cref{def:hypercontractivity}) with $\alpha=2$. Letting $f = \sum_{\ell=1}^{\me} z_{\ell}\phi_{\ell} \in \Ftmp^{\me}$, we first observe that
\begin{equation}\label{eq:IImomentFm}
    \norm{\sum_{\ell=1}^{\me} z_{\ell}\phi_{\ell}}{\Elltwo(\X^T,\Lebmeas^T;\R^{\dy})}^2 = \norm{\sum_{\ell=1}^{\me} z_{\ell}\phi_{\ell}}{\Elltwo(\X,\Lebmeas;\R^{\dy})}^2 = \sum_{\ell=1}^{\me} \norm{z_{\ell}}{2}^2.
\end{equation}
Next, looking at the fourth moment,
\begin{align}\label{eq:Fmhyper}
    \norm{\sum_{\ell=1}^{\me} z_{\ell}\phi_{\ell}}{\mathscr{L}^4(\X^T,\Lebmeas^T;\R^{\dy})}^4 &= \norm{\sum_{\ell=1}^{\me} z_{\ell}\phi_{\ell}}{\mathscr{L}^4(\X,\Lebmeas;\R^{\dy})}^4 \notag\\
    &= \int_{\X} \norm{\sum_{\ell=1}^{\me} z_{\ell}\phi_{\ell}(\xx)}{2}^4d\xx \notag \\
    & \leq \int_{\X} \left( \sum_{\ell=1}^{\me} \norm{z_{\ell}}{2} |\phi_{\ell}(x)|\right)^4 d\xx \notag\\
    & \leq \Lebmeas(\X) \left(\sum_{\ell=1}^{\me} \norm{z_{\ell}}{2} \right)^4 \notag \\
    & \leq \Lebmeas(\X) \left( \sqrt{\sum_{\ell=1}^{\me} \norm{z_{\ell}}{2}^2}\right)^4 \left( \sqrt{\me}\right)^4 \notag\\
    & \stackrel{\eqref{eq:IImomentFm}}{=} \Lebmeas(\X)\me^2 \left(\norm{\sum_{\ell=1}^{\me} z_{\ell}\phi_{\ell}}{\Elltwo(\X^T,\Lebmeas^T;\R^{\dy})}^2\right)^2,
\end{align}
which shows the $(\Lebmeas(\X)\me^2,2)$-hypercontractivity of $\Ftmp^{\me}$.\\
Before showing hypercontractivity of $(\Fe,\Lebmeas^T)$, we first state some additional useful relations. Recalling that $\eps  \leq \inf_{f\in \partial B(\tir )} \norm{f}{\Elltwo(\X^T,\Lebmeas^T;\R^{\dy})}$, letting $f$ be an arbitrary element in the cover $\Fe$ and $f'$ be a function in $\Ftmp^{\me}$ such that $\norm{f-f'}{\Ellinf(\X,\R^{\dy})} \leq \eps /4$, we have: 
\begin{subequations}
\begin{align}
     &\norm{f(x)}{2}^4 \leq 8\left( \norm{f(x) - f'(x)}{2}^4 + \norm{f'(x)}{2}^4 \right) \leq \frac{\eps ^4}{32} + 8\norm{f'(x)}{2}^4 \label{eq:fx4}\\ 
     &\norm{f'(x)}{2}^2 \leq 2\left( \norm{f(x) - f'(x)}{2}^2 + \norm{f(x)}{2}^2 \right) \leq \frac{\eps ^2}{8} + 2\norm{f(x)}{2}^2 \label{eq:fx2}\\
     &\eps ^2 \leq \norm{f}{\Elltwo(\X^T,\Lebmeas^T;\R^{\dy})}^2,\qquad \eps ^4 \leq \left(\norm{f}{\Elltwo(\X^T,\Lebmeas^T;\R^{\dy})}^2\right)^2 \label{eq:epsleqnorm}.
\end{align}
\end{subequations}
We can now get to the final claim and show hypercontractivity for $\Fe$. Letting again $f$ be an arbitrary element in the cover $\Fe$, it holds that
\begin{align}
    \norm{f}{\mathscr{L}^4(\X^T,\Lebmeas^T;\R^{\dy})} &= \int_{\X} \norm{f(\xx)}{2}^4d\xx \notag \\  
     &\stackrel{\eqref{eq:fx4}}{\leq}\Lebmeas(\X)\frac{\eps ^4}{32} + 8\int_{\X} \norm{f'(\xx)}{2}^4d\xx \notag \\
     & = \Lebmeas(\X)\frac{\eps ^4}{32} + 8\int_{\X} \norm{\sum_{\ell=1}^{\me} z_{\ell}'\phi_{\ell}(\xx)}{2}^4d\xx \notag\\
&\stackrel{\eqref{eq:Fmhyper}}{\leq} \Lebmeas(\X)\frac{\eps ^4}{32} + 8\Lebmeas(\X)\me^2 \left( \int_{\X} \norm{f'(\xx)}{2}^2 d\xx\right)^2 \notag \\
& \stackrel{\eqref{eq:fx2}}{\leq} \Lebmeas(\X)\frac{\eps ^4}{32} + 8\Lebmeas(\X)\me^2 \left(\int_{\X} \frac{\eps ^2}{8} + 2\norm{f(\xx)}{2}^2 d\xx\right)^2 \notag\\
&\stackrel{\eqref{eq:epsleqnorm}}{\leq} \left( \frac{\Lebmeas(\X)}{32} + 8\Lebmeas(\X)\me^2 \left( \frac{\Lebmeas(\X)}{8} + 2 \right)^2 \right) \left( \norm{f}{\Elltwo(\X^T,\Lebmeas;\R^{\dy})}\right)^2,\notag
\end{align}
which concludes the proof.
\end{proof}

\section{On the regularizer $\reg{f}$}\label{sec:regularizer_properties}
We now present the main properties of the regularizer $\reg{f} \colon \F \to \R_{\geq 0}$ introduced in the learning problem~\eqref{eq:RERM}, where we recall that $\reg{f} = || \Der(f)||_{\Elltwo(\X^T,\jointProb;
R^{\dy})}^2$.

\subsection{Physics-informed regularization is 2-proper}\label{sec:reg_2proper}
We start from the definition of $\eta$-regularizer~\citep{lecue_regularization_2017}:
\begin{definition}\label{def:eta_regularizer}
    An $\eta$-proper regularizer defined for a hypothesis space $\HH$ is a function $\reg{\cdot} \colon \HH \to \R$ satisfying the following properties:
    \begin{enumerate}[label =(\alph*)]
    \item it is non-negative, even, convex, and such that $\reg{0} = 0$;
    \item for $\eta \geq 1$, it holds for every $f,h$ in $\F$ that $\reg{f+h} \leq \eta \left( \reg{f} + \reg{h} \right)$;
    \item for every $0\leq a \leq 1$, $\reg{a f} \leq a \reg{f}$. Additionally, if $\eta=2$, it holds that \mbox{$\reg{a f} \leq a^{2} \reg{f}$}.
    \end{enumerate}
\end{definition}
In particular, any square-norm-based regularizer is 2-proper. Therefore, by construction, the physics-informed regularizer considered in this paper (see~\eqref{eq:pderegu}) is 2-proper. 

\subsection{Useful inequality from \cite{lecue_regularization_2017}} \label{sec:inequality_LM}
We now report an inequality that will be used in the proof of~\cref{thm:main_probab} proved in~\cref{sec:proofboundprob}.
\begin{lemma}[\cite{lecue_regularization_2017}, Inequality 2.3]\label{lemma:lecue_inequality}
Denote with $\fhat$ the solution of the regularized empirical risk minimization over $\Freg$. Write $\fhat = \fstar + R(h-\fstar)$, where $R \geq 1$ and $\reg{h-\fstar} = \rho$, with $\rho \geq 5\eta \reg{\fstar}$. Then it holds that    
    \begin{align}\label{eq:lecue_inequality}
    \reg{\fhat} - \reg{\fstar} \geq \frac{R}{2\eta^2} (\reg{h} - \reg{\fstar}).
    \end{align}
\end{lemma}
\begin{proof}
    By the triangle inequality and the fact that the regularizer is an even function (both reported in~\cref{def:eta_regularizer}), it holds that 
    \begin{align}
        \reg{\fhat} = \reg{\fstar + R(h-\fstar)}\geq \frac{1}{\eta}\reg{R(h-\fstar)} - \reg{\fstar} \geq \frac{R}{\eta}\reg{h-\fstar} - \reg{\fstar}\notag
    \end{align}
recalling that $R\geq 1$. Adding $\reg{\fstar}$ on both sides,
\begin{align}\label{eq:lecue1}
    \reg{\fhat} - \reg{\fstar} \geq \frac{R}{\eta}\reg{h-\fstar} - 2\reg{\fstar}.
\end{align}   
As an intermediate step, we find a lower bound on the term $\frac{R}{\eta}\reg{h-\fstar}$ in~\eqref{eq:lecue1}: specifically, it holds that
\begin{align}
    \frac{R}{\eta} \reg{h-\fstar} &= \frac{R}{2\eta}\reg{h-\fstar} + \frac{R}{2\eta}\reg{h-\fstar} \notag \\
    &\geq \frac{R}{2\eta}\reg{h-\fstar} + \frac{5R}{2}\reg{\fstar} \quad \text{ because $\rho = \reg{h-\fstar} \geq 5\eta\reg{\fstar)}$} \notag\\
    & \geq \frac{R}{2\eta}\reg{h-\fstar} +  \frac{R}{2}\reg{\fstar} + 2\reg{\fstar} \quad \text{because } R\geq 1, \notag \\
    &\geq \frac{R}{2\eta}\left(\reg{h-\fstar} + \reg{\fstar}\right) + 2\reg{\fstar} \quad \text{ as $\eta\geq 1$ and $R\geq 1$,} \notag \\
    &\geq \frac{R}{2\eta^2}\reg{h} + 2\reg{\fstar}\notag
\end{align}
again as a consequence of the triangle inequality in~\cref{def:eta_regularizer}(b). Plugging such an inequality back in~\eqref{eq:lecue1}, we obtain
\begin{align}
    \reg{\fhat} - \reg{\fstar} \geq \frac{R}{2\eta^2}\reg{h} \geq \frac{R}{2\eta^2}\reg{h} - \frac{R}{2\eta^2}\reg{\fstar},\notag
\end{align}
which yields the claim.
\end{proof}

\section{Lower isometry bound}\label{sec:lower_isometry_bound}
We start by presenting in~\cref{sec:combo_CalphaS} an ancillary result combining $(C,\alpha)$-hypercontractivity (\cref{def:hypercontractivity}) and $S$-persistence (\cref{sec:Spersistence}). This will then play a key role in~\cref{sec:bound_Alor}, where we prove an upper bound for the probability of the lower isometry event, which will be crucial in our main results stated in~\cref{sec:bounds,sec:convergence_rates}. 

\subsection{Combining $(C,\alpha)$-hypercontractivity and $S$-persistence}\label{sec:combo_CalphaS}
We now present an ancillary result obtained by generalizing~\cite[Lemma~3.1.1.]{ziemann_statistical_2022}.

\begin{lemma}\label{lemma:ziemann311}
Consider $g\colon \X \to \R_{\geq 0}$ satisfying $(C,\alpha)$-hypercontractivity (see~\cref{def:hypercontractivity}) and $S$-persistence (see~\cref{ass:S-persistence}). Then, for $\theta \geq 8$, it holds that
    \begin{equation}
        \jointProb\left(\sumt  g(X_t) \leq \frac{4}{\theta}\sumt \E{g(X_t)}{\jointProb}\right) 
        \leq  \exp\left(-\frac{8T}{CS\theta^2}\left(\sumt   \E{g(X_t)}{\jointProb} \right)^{2-\alpha} \right)\label{eq:ziemann_lemma311}
    \end{equation}
\end{lemma}
\begin{proof}
We start by generalizing the $S$-persistence bound in~\cref{ass:S-persistence}. Specifically, introducing $\varepsilon > 0$, it holds that 
\begin{align}
    \mathbb{E}\left[\exp\left(- \xi \sumt g(X_t) \right) \right] \leq \exp\left( -\frac{8\xi}{\theta} \sumt \mathbb{E}[g(X_t)]  + \frac{\xi^2 S \varepsilon}{\theta} \sumt \mathbb{E}[g^2(X_t)] \right), \label{eq:new_S_persistence}
\end{align}
where it is required that $\theta \geq 8$ and $\varepsilon/\theta \geq 1/2$.\\ 
Now we consider the left-hand side of~\eqref{eq:ziemann_lemma311} and apply a Chernoff bound to obtain 
    \begin{align}
        \Prob\left(\sumt  g(X_t) \leq \frac{4}{\theta}\sumt \mathbb{E}[g(X_t)]\right) \leq \inf_{\xi > 0} \mathbb{E}\left[\exp \left( \frac{4\xi}{\theta} \sumt   \mathbb{E}[g(X_t)] - \xi \sumt g(X_t)\right)   \right] \notag \\
\stackrel{\eqref{eq:new_S_persistence}}{\leq} \inf_{\xi > 0} \mathbb{E}\left[\exp \left( -\frac{4\xi}{\theta} \sumt   \mathbb{E}[g(X_t)] + \frac{\xi^2 S\varepsilon}{\theta} \sumt \mathbb{E}[g^2(X_t)]\right)   \right].\notag 
    \end{align}
We find the optimal $\xi$, which reads as
\begin{align}
    \xi = \frac{2\sumt \mathbb{E}[g(X_t)]}{S\varepsilon \sumt \mathbb{E}[g^2(X_t)]}, \notag
\end{align}
and plugging it in we obtain 
\begin{align}
    \Prob\left(\sumt  g(X_t) \leq \frac{4}{\theta}\sumt \mathbb{E}[g(X_t)]\right) &\leq \exp\left( -\frac{4}{\theta S\varepsilon} \frac{\left( \sumt \mathbb{E}[g(X_t)]\right)^2}{\sumt \mathbb{E}[g^2(X_t)]} \right) \notag \\
    &\leq \exp\left( -\frac{4T}{\theta CS\varepsilon} \left( \sumt \mathbb{E}[g(X_t)]\right)^{2-\alpha} \right) \notag
\end{align}
by $(C,\alpha)$-hypercontractivity given in~\cref{def:hypercontractivity}. Conclusion follows by minimizing the bound over $\varepsilon \geq \theta/2$.
\end{proof}

\subsection{The main bound on lower isometry}\label{sec:bound_Alor}
We are now ready to prove the key bound for the lower isometry event by generalizing~\cite[Theorem~3.1.2.]{ziemann_statistical_2022}. 

\begin{theorem}\label{thm:new_thm312}
Assume that the tuple $(\Freg,\jointProb)$ is $S$-persistent (\cref{ass:S-persistence}). For a given $r>0$, define $B(r) \doteq \left\{ f \in \F \mid \norm{f}{\Elltwo(\X^T,\jointProb;\R^{\dy})} \leq r^2 \right\}$ and let $\partial B(r)$ be its boundary. Additionally, assume that the hypothesis space satisfies the $(C,\alpha)$-hypercontractivity condition (\cref{def:hypercontractivity}) on $\partial B(r)$. For a fixed $\theta >8$, define $\Fr$ the $r/\sqrt{\theta}$-cover in the $\norm{\cdot}{\Ellinf(\X^T;\R^{\dy})}$ of $\partial B(r)$, and denote by $\N{\infty}{\partial B(r)}{\frac{r}{\sqrt{\theta}}}$ the corresponding covering number. 
    Define the lower-isometry event 
    \begin{align}\label{eq:lower_isometry_event}
        \Alo_r \doteq \sup_{f \in \Fregstar \setminus B(r)} \left\{\frac{1}{T}\sumt   \norm{f(X_t)}{2}^2 - \frac{1}{\theta}\norm{f}{\Elltwo(\X^T,\jointProb; \R^{\dy})}^2  \leq 0\right\}.
    \end{align}
    Then the following lower-isometry estimate holds:
    \begin{align}
 \jointProb\left(\Alo_r \right)
        \leq \N{\infty}{\partial B(r)}{\frac{r}{\sqrt{\theta}}}\exp\left\{-\frac{8Tr^{4-2\alpha}}{\theta^2CS} \right\}.\notag
    \end{align}
\end{theorem}
\begin{proof}
We first show a preliminary result that allows us to focus just on the boundary $\partial B(r)$ instead of the full $\Fstar \setminus B(r)$. Specifically, we show that, if $f \in \Fregstar$ and $0 \leq \xi \leq 1$, then $\xi f$ still belongs to $\Fregstar$: that is, we show that $\Fregstar$ is \emph{star-shaped} around 0 (see~\cite[Definition~5.1]{mendelson_learning_2014}). To prove that $\xi(f - \fstar)$ belongs to $\Fregstar$ for $f,\fstar \in \Freg$ we deploy~\cref{lemma:propF}: specifically, we have 
\begin{align}
        \xi f - \xi \fstar \pm \fstar = \underbrace{\xi f + (1-\xi)\fstar}_{w} - \fstar,\notag
    \end{align}
    and $w\in\Freg$ by convexity, thus proving the claim.\\
Thanks to the result above obtained, we can focus on $\partial B(r)$ and then obtain the final claim by rescaling: if $f^{\prime} \in \Fregstar \setminus B(r)$, then $\|f^{\prime}\|_{\Elltwo(\X^T,\jointProb;\R^{\dy})} > r$ by construction, which implies $\frac{r}{\|f^{\prime}\|_{\Elltwo}(\X^T,\jointProb;\R^{\dy})} < 1$; thus, if we consider $f = f^{\prime}\frac{r}{\|f^{\prime}\|_{\Elltwo}(\X^T,\jointProb;\R^{\dy})}$, we are on $\partial B(r)$, and $f \in \Fregstar$ by it being star-shaped around 0.\\

    Define the event 
    \begin{align}
        \Eset \doteq \bigcup_{f^i \in \Fr} \left\{ \frac{1}{T} \sumt   \norm{f^i(X_t)}{2}^2 \leq \frac{4}{\theta} \|f^i\|_{\Elltwo(\X^T,\jointProb;\R^{\dy})}^2 \right\}.\notag 
    \end{align}
    By \cref{lemma:ziemann311} with $g(x) = \norm{f^i(x)}{2}^2$, the union bound yields
    \begin{align}
        \jointProb\left(\Eset \right) \leq  \N{\infty}{\partial B(r)}{\frac{r}{\sqrt{\theta}}}\exp\left\{-\frac{8Tr^{4-2\alpha}}{\theta^2 CS} \right\}.\notag
    \end{align}
    Now, fixing an arbitrary $f\in \partial B(r)$:
    \begin{align}
        \frac{1}{T}\sumt   \norm{f(X_t)}{2}^2 &\geq  \frac{1}{2T}\sumt   \norm{f^i(X_t)}{2}^2 - \frac{r^2}{\theta} &\quad {\text{ \cref{eq:parallelogram},}}\notag \\
        & \geq \frac{2}{\theta}\norm{f^i}{\Elltwo(\X^T,\jointProb;\R^{\dy})}^2 - \frac{r^2}{\theta} &\quad \text{ on }\Eset^{\complement},  \notag \\
        &= \frac{2r^2}{\theta} - \frac{r^2}{\theta} 
        = \frac{r^2}{\theta} &\quad \text{by definition of $\Fr$.} \notag 
    \end{align}
Since $f$ was arbitrary, we obtain
\begin{align}
\Prob\left(\sup_{f \in \partial B(r)} \left\{ \frac{1}{T}\sumt   \norm{f(X_t)}{2}^2 - \frac{r^2}{\theta} \leq 0 \right\} \right) \notag \leq \N{\infty}{\partial B(r)}{\frac{r}{\sqrt{\theta}}}\exp\left\{-\frac{8Tr^{4-2\alpha}}{\theta^2 CS} \right\}. \notag
\end{align}
The claim is finally obtained by rescaling.
\end{proof}

We can now provide a special case of~\cref{thm:new_thm312} that will be useful in the derivations of the paper.
\begin{corollary}\label{cor:thm312}
    Under the assumptions of~\cref{thm:new_thm312}, assume that the hypothesis space satisfies the $(C(r),2)-$hypercontractivity condition according to~\cref{thm:Fhypercontractive}. Then, the following lower-isometry estimate holds:
    \begin{equation}
        \jointProb(\Alo_r) \leq \left(C_L \left(\frac{1}{r}\right)^{\frac{4s-\dx}{2s-\dx}} +1 \right)^{\dy C_m \left(\frac{1}{r}\right)^{\frac{2\dx}{2s-\dx}}}\exp\left\lbrace -\frac{8T r^{\frac{4\dx}{2s-\dx}}}{\theta^2 C_h S} \right\rbrace,\notag
    \end{equation}
    where $C_L,C_m$ and $C_h$ are constants that depend on $\rho_f, \kappau,\dy,\theta,s,\dx$ and $\X$. 
\end{corollary}
\begin{proof}
    The corollary is obtained by leveraging~\cref{thm:Fhypercontractive}, which provides expressions for both the covering number of $\Fr$ and for the hypercontractivity parameter $C(r)$,  setting the covering of $\partial B(r)$ to have resolution equal to $r/\sqrt{\theta}$. 

    \paragraph{Covering number.} Using the notation of~\cref{thm:Fhypercontractive}, we have that
    \begin{align}
       &\N{\infty}{\partial B(r)}{\frac{r}{\sqrt{\theta}}} \stackrel{\eqref{eq:covering_claim}}{\leq} \left( \frac{8\widetilde{\rho_f}\dy \sqrt{\theta} m_{\frac{r}{\sqrt{\theta}}}^{\frac{s}{\dx}}}{r} + 1 \right)^{\dy m_{\frac{r}{\sqrt{\theta}}}} \notag,\\
       &\text{where } m_{\frac{r}{\sqrt{\theta}}} \geq \left(\frac{16\widetilde{\rho_f}\dx\theta }{2s-\dx}\cdot \left(\frac{1}{r^2}\right)\right)^{\frac{\dx}{2s-\dx}} \longrightarrow m_{\frac{r}{\sqrt{\theta}}} = C_m\left(\frac{1}{r}\right)^{\frac{2\dx}{2s-dx}}.\label{eq:m_r_sqrttheta}
    \end{align}
With such a value for $m_{\frac{r}{\sqrt{\theta}}}$, we obtain that the covering number admits the following upper bound:
\begin{align}
    \N{\infty}{\partial B(r)}{\frac{r}{\sqrt{\theta}}} \leq \left( \underbrace{8\widetilde{\rho_f}\dy\sqrt{\theta}C_m^{\frac{s}{\dx}}}_{\doteq C_L}\left(\frac{1}{r}
    \right)^{\frac{4s-\dx}{2s-\dx}} + 1 \right)^{\dy C_m\left(\frac{1}{r}\right)^{\frac{2\dx}{2s-\dx}}}. \label{eq:covering_number_cor}
\end{align}
\paragraph{Hypercontractivity parameter.} Again, using the result in~\cref{thm:Fhypercontractive} and using the expression for $m_{\frac{r}{\sqrt{\theta}}}$ in~\eqref{eq:m_r_sqrttheta}, we obtain that
\begin{align}
    &C(r) \stackrel{\eqref{eq:Ceps}}{\propto} \left( \frac{\Lebmeas(\X)}{32} + 8\Lebmeas(\X) \left( \frac{\Lebmeas(\X)}{8} + 2 \right)^2 C_m^2 \left(\frac{1}{r}\right)^{\frac{4\dx}{2s-\dx}}\right) \notag \\ \longrightarrow \: &C(r) = C_h\left(\frac{1}{r}\right)^{\frac{4\dx}{2s-\dx}} \text{ for some sufficiently large constant $C_h$.}\label{eq:hypercontract_cor}
\end{align}
The lower-isometry probability bound is then obtained by plugging~\eqref{eq:covering_number_cor} and~\eqref{eq:hypercontract_cor} in the claim of~\cref{thm:new_thm312}.
\end{proof}

\section{Martingale offset complexity bounds}\label{sec:MOC}

In this section we focus on some useful results concerning the martingale offset complexity presented in~\eqref{eq:MOC} and that will play a prominent role in the main results of~\cref{sec:bounds,sec:convergence_rates}, being an upper-bound on the empirical excess risk. We start by proving the inequality leading to the definition of the martingale offset complexity, building upon~\cite{liang_learning_2015}. Next, we report its bounds in probability and in expectation obtained by the chaining arguments of~\cite[Theorem~4.2.2, Theorem~3.2.1]{ziemann_statistical_2022}. The proofs of the latter, given in~\cref{sec:MOC_prob} and~\cref{sec:MOC_exp} respectively, are given in full detail to keep track of all of the constants involved.

\subsection{Behind the scenes of the definition}\label{sec:MOC_def}

The first result is an ancillary inequality derived by extending~\cite{liang_learning_2015} to the regularized case (see also~\cite[Lemma~1]{ziemann_single_2022}). The following lemma is the basis yielding the definition of martingale offset complexity.
\begin{lemma}\label{lemma:liang}
    Let $\fhat$ be the solution of the regularized empirical risk minimization problem~\eqref{eq:RERM}. Then, it holds that 
\begin{align}\label{eq:liang}
    \frac{1}{T}\sumt   \norm{\fhat(X_t) - \fstar(X_t)}{2}^2 \leq \frac{1}{T}  \sumt   4\innerprod{W_t}{\fhat(X_t)-\fstar(X_t)}{2} - \norm{\fhat(X_t)-\fstar(X_t)}{2}^2. 
    \notag
\end{align}
\end{lemma}

\begin{proof}
    We start by showing the following facts:
    \begin{itemize}
    \item \textit{Fact 1:} for any $f$ and the measurement model $Y_t = \fstar(X_t) + W_t$ in~\eqref{eq:measurement_model}, it holds that \begin{align} \norm{f(X_t) - \fstar(X_t)}{2}^2 = \norm{Y_t - f(X_t)}{2}^2 - \norm{Y_t - \fstar(X_t)}{2}^2 + 2\innerprod{W_t}{f(X_t) - \fstar(X_t)}{2}; \notag \end{align}
    \item \textit{Fact 2:} from the construction of~\eqref{eq:RERM}, we have that 
    \begin{equation}
    \frac{1}{T}\sumt  \norm{Y_t - \fhat(X_t)}{2}^2 + \lambda_T \reg{\fhat} \leq \frac{1}{T}\sumt   \norm{Y_t - \fstar(X_t)}{2}^2 + \lambda_T \reg{\fstar}. \label{eq:fact2a}
    \end{equation}
    Additionally, since $\reg{\fstar}\leq \reg{\fhat}$, we have
    \begin{align}
        \frac{1}{T}\sumt  \norm{Y_t - \fhat(X_t)}{2}^2  - \norm{Y_t - \fstar(X_t)}{2}^2 \leq 0. \label{eq:fact2b}
    \end{align}
\end{itemize}

\textit{Fact 2} follows immediately from optimality of $\fhat$. To see why \textit{Fact 1} holds:
\begin{align}
    \norm{f(X_t) - \fstar(X_t) \pm Y_t}{2}^2 = \norm{Y_t - f(X_t)}{2}^2 + \norm{Y_t - \fstar(X_t)}{2}^2 -2\innerprod{Y_t-f(X_t)}{Y_t - \fstar(X_t)}{2}.\notag 
\end{align}
Considering the last addendum on the right-hand side, adding and subtracting $\fstar(X_t)$ in $(Y_t - f(X_t))$ and using the definition of $W_t$ for the other term, we obtain
\begin{align}
    \norm{f(X_t) - \fstar(X_t)}{2}^2 = &\norm{Y_t - f(X_t)}{2}^2 + \norm{Y_t - \fstar(X_t)}{2}^2 \notag \\&-2\norm{Y_t-\fstar(X_t)}{2}^2 - 2\innerprod{W_t}{ \fstar(X_t)-f(X_t)}{2},\notag 
\end{align}
thus proving the claim in \textit{Fact 1}.

We are now ready to prove~\cref{lemma:liang}. We start by applying \textit{Fact 1} to the estimate $\fhat$ of~\eqref{eq:RERM} and multiplying everything by 2. Rearranging the terms, we then obtain 
\begin{align}
   \frac{1}{T}\sumt   \norm{\fhat(X_t) - \fstar(X_t)}{2}^2 &= \frac{1}{T}  \sumt  2\left[ \overbrace{\norm{Y_t - \fhat(X_t)}{2}^2 - \norm{Y_t - \fstar(X_t)}{2}^2}^{(\natural)} \right] \notag \\ & \qquad + 4\innerprod{W_t}{\fhat(X_t) - \fstar(X_t)}{2}  -\norm{\fhat(X_t) - \fstar(X_t)}{2}^2 . \notag
\end{align}
The conclusion follows by applying \textit{Fact 2} to $(\natural)$. In particular, the claim is obtained by deploying~\eqref{eq:fact2b}. If on the other hand one would use~\eqref{eq:fact2a}, we would obtain
\begin{align}
    &\frac{1}{T}\sumt   \norm{\fhat(X_t) - \fstar(X_t)}{2}^2 \notag \\ \leq& \frac{1}{T}  \sumt   4\innerprod{W_t}{\fhat(X_t)-\fstar(X_t)}{2}  - \norm{\fhat(X_t)-\fstar(X_t)}{2}^2 + 2\lambda_T\left(\reg{\fstar} - \reg{\fhat}\right) \label{eq:otherMOCineq_tot}   \\
    \leq &\frac{1}{T}  \sumt   4\innerprod{W_t}{\fhat(X_t)-\fstar(X_t)}{2}  - \norm{\fhat(X_t)-\fstar(X_t)}{2}^2 + 2\lambda_T\reg{\fstar}. \label{eq:otherMOCineq}
\end{align}

\end{proof}

\subsection{Bound in probability}\label{sec:MOC_prob}
We now provide the high-probability bound for the martingale offset complexity of a given hypothesis space $\HH$. This will then be deployed in the high-probability bound for the excess risk in~\cref{thm:main_probab}, and its analysis will be key to determine the desired rate.
\allowdisplaybreaks[0]
\begin{theorem}[\cite{ziemann_statistical_2022}, Theorem 4.2.2]\label{thm:MOC_prob}
    Let~\cref{ass:densities,ass:SobolevOrder,ass:noise,ass:S-persistence} hold, and let
    $\HH$ be a convex hypothesis space belonging to $\Sobo(\X^T;\jointProb;\R^{\dy})$ and satisfying~\cref{ass:containment}. Let $u,v,w \geq 0$. Then, with probability $1 
    -4\exp\left\{-u^2/2\right\} - \exp\left\{-v/2\right\}$ the martingale offset complexity satisfies 
    \begin{align}
        \MOC{\HH} 
        \leq \inf_{\gamma > 0}  \Bigg\{ &8\gamma (u+1) \sqrt{\frac{\sigmaw}{T}} + 8\int_{0}^{\gamma} \sqrt{\frac{\sigmaw\log\N{\infty}{\HH}{\varepsilon}}{T}}d\varepsilon  \notag \\ + &32\sigmaw\frac{(v + \log \N{\infty}{\HH}{\gamma})}{T} + 4\gamma^2 \Bigg\}.
        \notag
    \end{align}
\end{theorem}
\allowdisplaybreaks
\begin{proof}

The core of the proof consists in a \emph{chaining} argument~\citep{talagrand_generic_2005}, i.e., in finding a suitable finite cover of $\HH$ and deploying it to derive the desired bounds. We start by defining the terms of the chaining.
\paragraph{Chaining set-up.}
      Let $F_k$ denote the cover of $\HH$ with radius $\epsilon_k = \frac{1}{2^k}$; given two arbitrary positive scalars $\delta < \gamma$, the values of $k$ belong to an interval of integers $[\Kl, \Ku]$ such that \begin{equation}\label{eq:defs_chaining}
        \frac{1}{2^{\Ku+1}} \leq \delta \leq \frac{1}{2^{\Ku}} \leq \frac{1}{2^{\Kl+1}} \leq \gamma.
    \end{equation}
    For an arbitrary $f \in \HH$, denote with $\pi_k(f)$ the center of the ball in the cover $F_k$ that contains $f$ -- i.e., the function such that $\| f - \pi_k(f)\|_{\Ellinf(\X^T;\R^{\dy})} \leq \epsilon_k$.\\

    Let us write the martingale offset complexity $\MOC{\HH}$ using the following notation: 
    \begin{align}\label{eq:defsMOC}
        \MOC{\HH} = \sup_{f\in\HH}  \frac{1}{T}\Bigg[\underbrace{\underbrace{ \sumt   4\innerprod{W_t}{f(X_t)}{2}}_{\doteq M_T(f)} - \underbrace{\sumt   \norm{f(X_t)}{2}^2}_{\doteq S_T(f)}}_{\doteq N_T(f)}\Bigg]. 
    \end{align}
Now, let us exploit the sequence of coverings to write $M_T(f)$ as a telescopic sum:
\begin{align}
   M_T(f) &= M_T(f) \pm M_T(\pi_{\Ku}(f)) \pm \dots \pm M_T(\pi_{\Kl}(f))  \notag \\
   &= \left[ M_T(f) - M_T(\pi_{\Ku}(f)) \right] + D_T(f) + M_T(\pi_{\Kl}(f)),\notag 
\end{align}
    where we set \mbox{$D_T(f) \doteq \sum_{k=\Kl+1}^{\Ku} M_T(\pi_k(f)) - M_T(\pi_{k-1}(f))$}. We can now focus on using such a rewriting of $M_T(f)$ in the expression of $N_T(f)$ of~\cref{eq:defsMOC}. With adding and subtracting $S_T(\pi_{\Kl}(f))/2$, it reads as
    \begin{align}
        N_T(f) &= M_T(f) - S_T(f) \notag \\
        & = \left[ M_T(f) - M_T(\pi_{\Ku}(f)) \right] + D_T(f)  
        +  \left[M_T(\pi_{\Kl}(f)) - \frac{S_T(\pi_{\Kl}(f))}{2} \right] \notag \\ &\quad + \left[\frac{S_T(\pi_{\Kl}(f))}{2} - S_T(f) \right]\notag \\
& \leq  \left[ M_T(f) - M_T(\pi_{\Ku}(f)) \right] + D_T(f) 
+  \left[M_T(\pi_{\Kl}(f)) - \frac{S_T(\pi_{\Kl}(f))}{2} \right] + T\epsilon_{\Kl}^2,\notag
    \end{align}
where the inequality acting on the last term is obtained through~\eqref{eq:parallelogram}. Thus, overall, to find a bound for $\MOC{\HH} = \sup_{f\in\HH} N_T(f)/T$, we are interested in
\allowdisplaybreaks[0]
\begin{align}\label{eq:termsNANBNC}
    \sup_{f\in \HH} N_T(f) &\leq  \underbrace{\sup_{f,g \in \HH \atop \norm{f-g}{\Ellinf(\X^T;\R^{\dy})} \leq 2^{-\Ku}} \left[M_T(f) - M_T(g) \right]}_{\doteq \NA} + \underbrace{\sup_{f\in \HH} D_T(f)}_{\doteq \NB} \notag \\ &+ \underbrace{\sup_{f \in F_{\Kl}} \left[  M_T(f) - \frac{S_T(f)}{2} \right]}_{\doteq \NC} + T \underbrace{\Big(   \frac{1}{2^{\Kl}}   \Big)^2}_{\stackrel{\eqref{eq:defs_chaining}}{\leq} 4\gamma^2 }.
\end{align}
\allowdisplaybreaks
The proof proceeds with the following steps. For each $\textbf{(\texttt{N.x})}$ with $\textbf{\texttt{x}} = \textbf{\texttt{A,B,C}}$, we derive high-probability bounds of the form
\begin{equation}
    \boxed{\Prob\left(  \frac{1}{T}\textbf{(\texttt{N.x})} > \frac{1}{T}\boxed{\text{value({\bfseries\texttt{x}})}}\,  \right) \leq \boxed{\text{probability bound({\bfseries\texttt{x}})}}\,;}\notag
\end{equation}
then, by deploying the union bound, we combine those results and obtain an upper-bound for $\MOC{\HH}$ that depends on the parameters $\delta$ and $\gamma$ introduced in the chaining set-up~\eqref{eq:defs_chaining}. Finally, by leveraging the condition on the Sobolev order in~\cref{ass:SobolevOrder}, we show that we can let $\delta \to 0$ and $w \to +\infty$ to obtain the final claim. 

\paragraph{Bound for $\NA$.}
Defining $\Fku \doteq \lbrace f  = f^{\mathfrak{a}} - f^{\mathfrak{b}};\, f^{\mathfrak{a}}, f^{\mathfrak{b}}  \in \HH \,\vert\, \norm{f}{\Ellinf(\X^T;\R^{\dy})} \leq 2^{-\Ku} \rbrace$ and by linearity of $M_T(\cdot)$, we are interested in
\begin{align}
    \Prob\left( \sup_{f\in \Fku} M_T(f) > w \right) \leq 
    e^{-\xi w} \E{ \exp\left\lbrace \xi \sup_{f \in \Fku} \sumt 4\innerprod{W_t}{f(X_t)}{2}
    \right\rbrace}{} \label{eq:NA_p_1},
\end{align}
where the inequality follows by applying a Chernoff bound with $\xi > 0$. Now, by deploying monotonicity of the exponential function, we can work on finding upper-bounds for the term in curly brackets in~\eqref{eq:NA_p_1}. Specifically, we consider
\begin{align}
    \xi \sup_{f \in \Fku} \sumt 4\innerprod{W_t}{f(X_t)}{2} &\leq \xi \sup_{f \in \Fku} 4 \sqrt{\sumt \norm{W_t}{2}^2} \sqrt{\sumt \norm{f(X_t)}{2}^2} \notag \\
    & \leq \xi \frac{4\sqrt{T}}{2^{\Ku}}\sqrt{\sumt \norm{W_t}{2}^2} \notag\\ 
    &\leq \frac{1}{2}\left(\frac{4\sqrt{T}\xi}{2^{\Ku}}\right)^2\sumt \norm{W_t}{2}^2 + \frac{1}{2} \quad \text{ by Young's inequality.} \label{eq:NA_p_2}
\end{align}
Now, by plugging~\eqref{eq:NA_p_2} into~\eqref{eq:NA_p_1}, we obtain
\begin{align}
    \Prob\left( \sup_{f\in \Fku} M_T(f) > w \right) &\leq 
    e^{-\xi w + 1/2}\E{\exp \left\lbrace \left(\frac{4\sqrt{T}\xi}{\sqrt{2}\cdot 2^{\Ku}}\right)^2\sumt \norm{W_t}{2}^2 \right\rbrace}{} \notag \\
    & \stackrel{(Lemma~\ref{lemma:usefulsub-Gaussian})}{\leq} 
    \exp \left\lbrace -\xi w + \frac{1}{2} +  \xi^2\left(\frac{4 T \dy \sigmaw}{\sqrt{2}\cdot2^{\Ku}}\right)^2 \right\rbrace, \label{eq:NA_p_3}
\end{align}
provided that $\xi < \frac{\sqrt{2}\cdot 2^{\Ku}}{4T\dy\sigmaw}$ and using the law of total expectation on the sum of $\norm{W_t}{2}^2$. In view of obtaining a bound in terms of $w$ and not $w^2$, we can choose at our convenience $\xi$ such that the quadratic term in~\eqref{eq:NA_p_3} becomes equal to 1/2. To this aim, setting $\xi = \frac{2^{\Ku}}{4T\dy\sigmaw}$ (note that it satisfies the constraint of~\cref{lemma:usefulsub-Gaussian}) and deploying the definition of $\delta$ in~\eqref{eq:defs_chaining}, we obtain 
\begin{align}
    \Prob\left( \sup_{f\in \Fku} M_T(f) > w \right) &\leq \exp\left\lbrace 1 - \frac{w}{4\delta T\dy\sigmaw}\right\rbrace. \notag
\end{align}
By substituting $w \leftrightarrow w\cdot 4\delta T \dy \sigmaw$ and dividing by $T$, we finally obtain

\begin{equation}
    \boxed{\Prob\left( \sup_{f\in \Fku} \frac{M_T(f)}{T} > 4w\delta  \dy \sigmaw \right) \leq \exp\lbrace -w + 1\rbrace.} \label{eq:bound_NA_p}
\end{equation}

\paragraph{Bound for $\NB$.}
We start by introducing the short-hand notation for the function space $\Fk \doteq \left\lbrace f = f^{\mathfrak{a}} - f^{\mathfrak{b}};\, f^{\mathfrak{a}} \in F_k, f^{\mathfrak{b}}  \in F_{k-1} \,\vert\, \norm{f}{\Ellinf(\X^T;\R^{\dy})} \leq 2^{-k}   \right\rbrace$ for all $k = \Kl+1,\dots, \Ku$. Additionally, by linearity of $M_T(\cdot)$, we also have that
\begin{equation}
    \sup_{f\in\HH} D_T(f) \leq \sum_{k=\Kl+1}^{\Ku} \sup_{f\in \HH} M_T\left(\pi_k(f) - \pi_{k-1}(f)   \right) =  \sum_{k=\Kl+1}^{\Ku} \max_{f\in \Fk} M_T(f).\notag
\end{equation}
We proceed by first studying the single addendum $\max_{f\in \Fk} M_T(f)$, and then apply a union bound to reach the desired claim for $\NB$. 

Letting $u_k > 0$, by deploying a Chernoff bound, we get
\begin{align}
    \Prob\left( \max_{f\in \Fk} M_T(f) > u_k  \right) &\leq \min_{\xi} e^{-\xi u_k}\E{\exp\left\lbrace \xi \max_{f\in \Fk} M_T(f) \right\rbrace}{} \notag\\
    & \leq \min_{\xi} e^{-\xi u_k}\E{ \sum_{f \in \Fk} \exp\left\lbrace \xi M_T(f) \right\rbrace}{}. \label{eq:NB_p_1}
\end{align}
We now upper-bound~\eqref{eq:NB_p_1} by iteratively applying the law of total expectation: specifically, we have that
\begin{align}
    &\E{ \sum_{f \in \Fk} \exp\left\lbrace \xi M_T(f) \right\rbrace}{} = \sum_{f \in \Fk} \E{\E{\exp\left\lbrace \xi \sumt 4\innerprod{W_t}{f(X_t)}{2}  \right\rbrace  \,\bigg\vert \, \Filt_{T-2} }{}}{}  \notag \\
    & =\sum_{f \in \Fk} \E{\exp\left\lbrace \xi \sum_{t=0}^{T-2} 4\innerprod{W_t}{f(X_t)}{2}  \right\rbrace}{} \E{\exp\left\lbrace \xi 4\innerprod{W_{T-1}}{f(X_{T-1})}{2} \right\rbrace \,\bigg\vert\, \Filt_{T-2}}{} \notag \\
    &\stackrel{\eqref{eq:sub-Gaussian}}{\leq} \sum_{f \in \Fk} \E{\exp\left\lbrace \xi \sum_{t=0}^{T-2} 4\innerprod{W_t}{f(X_t)}{2}  \right\rbrace}{} \exp\left\lbrace \frac{8\xi^2 \sigmaw}{2^{2k}} \right\rbrace \notag \\
    &\leq \vdots \quad \text{(i.e., repeating the argument with the next filtrations $\Filt_{T-3},..., \Filt_0$)}\notag \\
    & \leq |\Fk|\exp\left\lbrace \frac{8T\xi^2 \sigmaw}{2^{2k}} \right\rbrace, \label{eq:NB_p_2}
\end{align}
where $|\Fk|$ is the cardinality of $\Fk = F_{k}\times F_{k-1}$. Now, after noting that $|\Fk| \leq \left(\N{\infty}{\HH}{2^{-k}}\right)^2$, we can plug the bound of~\eqref{eq:NB_p_2} into~\eqref{eq:NB_p_1} and obtain, by minimizing over $\xi$ (yielding $\xi = \nicefrac{2^{2k}u_k}{(16 T\sigmaw}$),
\begin{align}
    \Prob\left( \max_{f\in \Fk} M_T(f) > u_k  \right) 
    & \leq \left(\N{\infty}{\HH}{\frac{1}{2^k}}\right)^2 \exp\left\lbrace -\frac{2^{2k}u_k^2}{32T\sigmaw} \right\rbrace  \notag 
\end{align}
Additionally, by substituting $u_k \leftrightarrow u_k + 2^{-k}\sqrt{64 T\sigmaw \log\N{\infty}{\HH}{\frac{1}{2^k}}} > 0$, we can remove the dependence on the covering number from the probability bound. Applying the union bound over all $k = \Kl+1,\cdots, \Ku$, we obtain that the bound for $\NB$ can be written as 
\begin{equation}
    \Prob\left( \sup_{f\in \HH} D_T(f) > \sum_{k=\Kl+1}^{\Ku} u_k + 2^{-k}\sqrt{64 T\sigmaw \log\N{\infty}{\HH}{\frac{1}{2^k}}}  \right) \leq \sum_{k=\Kl+1}^{\Ku} \exp\left\lbrace -\frac{2^{2k}u_k^2}{32\sigmaw T} \right\rbrace.\label{eq:NB_p_3}
\end{equation}
We now want the right-hand side of~\eqref{eq:NB_p_3} to depend on a single $u \in \R$, in order to obtain an upper-bound that reads, informally, as $\boxed{constant} \times \exp\left\lbrace -u^2/2\right\rbrace$. To do so, we operate on the terms $u_k$, $k=\Kl+1,\cdots,\Ku$ and set them to $u_k = 2^{2-k} \sqrt{\sigmaw T}\sqrt{u^2 - \log 2^{-k+\Kl+1}}$: thanks to this choice, the right-hand side of~\eqref{eq:NB_p_3} becomes
\begin{align}
&\sum_{k=\Kl+1}^{\Ku} \exp\left\lbrace - \frac{2^{2k}}{32T\sigmaw} \cdot \left( 2^{2-k}\sqrt{\sigmaw T} \sqrt{u^2 - \log 2^{-k+\Kl+1}}  \right)^2 \right\rbrace \notag \\
=& \sum_{k=\Kl+1}^{\Ku} \exp\left\lbrace \frac{-u^2}{2}\right\rbrace (\sqrt{2})^{-k+\Kl+1} \notag \\
\leq & \exp\left\lbrace \frac{-u^2}{2}\right\rbrace \sum_{k=0}^{\infty} \left( \frac{1}{\sqrt{2}}\right)^k = (2+\sqrt{2}) \exp\left\lbrace \frac{-u^2}{2}\right\rbrace \label{eq:NB_p_4}
\end{align}
as desired. Now we can analyze such a choice for $u_k$ in the left-hand side of~\eqref{eq:NB_p_3}, which becomes
\begin{align}
    \Prob\Bigg( \sup_{f\in\HH} D_T(f) > &\overbrace{\sum_{k=\Kl+1}^{\Ku} 2^{2-k}\sqrt{\sigmaw T(u^2 - \log 2^{-k+\Kl+1})}}^{\doteq s_1} \notag \\ &+ \underbrace{\sum_{k=\Kl+1}^{\Ku}2^{-k}\sqrt{64 T\sigmaw \log \N{\infty}{\HH}{2^{-k}}}}_{\doteq s_2}\Bigg).\label{eq:NB_p_5}
\end{align}
We now want to find upper bounds for $s_1$ and $s_2$ and remove the sum over $k$. Regarding $s_1$, we have 
\begin{align}
    s_1 &\leq \sum_{k=\Kl+1}^{\Ku} 2^{2-k}\sqrt{\sigmaw Tu^2} + \sum_{k=\Kl+1}^{\Ku} 2^{2-k}\sqrt{\sigmaw T}\sqrt{\log\left(2^{k-\Kl-1}\right)} \notag \\
    & \leq 4\cdot 2^{-\Kl-1} \sqrt{\sigmaw Tu^2}\sum_{k=0}^{\infty}2^{-k} + 4\cdot 2^{-\Kl-1} \sqrt{\sigmaw T \log 2}\underbrace{\sum_{k=0}^{\infty} 2^{-k}\sqrt{k}}_{=\text{Li}_{-1/2}(1/2)} \notag \\
    & \stackrel{\eqref{eq:defs_chaining}}{\leq} 8\gamma \sqrt{\sigmaw T}(u+1),\label{eq:NB_p_s1}
\end{align}
because the polylogarithmic function satisfies $\text{Li}_{1/2}(1/2) \approx 1.35$. Now, going to $s_2$, noting that $2^{-k} = (2^{-k+1} - 2^{-k})$ and by deploying a truncated Dudley's entropy integral~\cite[Theorem 5.22]{wainwright_high-dimensional_2019}, we have
\begin{align}
    s_2 &\leq \sum_{k=\Kl+1}^{\Ku}\left(\frac{1}{2^{k-1}} - \frac{1}{2^k} \right) \sqrt{64 T\sigmaw \log \N{\infty}{\HH}{2^{-k}}} \notag \\
    & \leq \int_{2^{-\Ku}}^{2^{-\Kl-1}} \sqrt{64 T\sigmaw \log \N{\infty}{\HH}{\varepsilon}}d\varepsilon \notag \\
    & \stackrel{\eqref{eq:defs_chaining}}{\leq} \int_{\delta}^{\gamma} \sqrt{64 T\sigmaw \log \N{\infty}{\HH}{\varepsilon}}d\varepsilon. \label{eq:NB_p_s2}
\end{align}
Thus, plugging~(\ref{eq:NB_p_s1},\ref{eq:NB_p_s2}) in~\eqref{eq:NB_p_5}, dividing by $T$ and using the bound in~\eqref{eq:NB_p_4}, we obtain the desired bound for $\NB$, which reads as
\begin{align}
    \boxed{\Prob\left( \sup_{f\in\HH} \frac{D_T(f)}{T} > \frac{8(u+1)\gamma \sqrt{\sigmaw}}{\sqrt{T}} + \frac{8}{\sqrt{T}}\int_{\delta}^{\gamma} \sqrt{\sigmaw \log \N{\infty}{\HH}{\varepsilon}}d\varepsilon\right) \leq 4\exp\left\lbrace -\frac{u^2}{2}\right\rbrace.}\label{eq:bound_NB_p}
\end{align}

\paragraph{Bound for $\NC$.} 
Applying again a Chernoff bound along the lines of the manipulations for $\NB$ in~\eqref{eq:NB_p_1}, we consider
\begin{align}
&\Prob\left(\max_{f\in F_{\Kl}} \sumt 4\innerprod{W_t}{f(X_t)}{2} -\frac{1}{2}\norm{f(X_t)}{2}^2 > v \right) \notag 
\\
& \leq \min_{\xi} e^{-\xi v}\sum_{f\in F_{\Kl}} \E{\exp\left\lbrace \xi \Bigg(  \sumt 4\innerprod{W_t}{f(X_t)}{2} -\frac{1}{2}\norm{f(X_t)}{2}^2\Bigg)\right\rbrace}{}\notag \\
&\leq \min_{\xi} e^{-\xi v}\sum_{f\in F_{\Kl}} \E{\E{\exp\left\lbrace \xi \left( \sumt 4\innerprod{W_t}{f(X_t)}{2} -\frac{1}{2}\norm{f(X_t)}{2}^2\right)\right\rbrace \bigg\vert \Filt_{T-2}}{}}{} \notag\\
&\leq \min_{\xi} e^{-\xi v} \Bigg(\xi \sum_{f\in F_{\Kl}} \E{\exp \left( \left\lbrace  \sum_{t=0}^{T-2} 4\innerprod{W_t}{f(X_t)}{2} -\frac{1}{2}\norm{f(X_t)}{2}^2\right)\right\rbrace}{}\Bigg) \notag\\
&\qquad \qquad \qquad \times \E{\exp\left\lbrace 4\xi \innerprod{W_{T-1}}{f(X_{T-1})}{2} - \xi\frac{1}{2}\norm{f(X_{T-1})}{2}^2\right\rbrace \vert \Filt_{T-2}}{}.
\label{eq:NC_p_1}
\end{align}
We now focus on last expected value in~\eqref{eq:NC_p_1} and discuss its upper bound. Specifically, we have
\begin{align}
   &\exp\left\lbrace - \xi \frac{1}{2}\norm{f(X_{T-1})}{2}^2 \right\rbrace\E{\exp\left\lbrace  4\xi \innerprod{W_{T-1}}{f(X_{T-1})}{2} \right\rbrace}{}\notag  \\
   &\stackrel{\eqref{eq:sub-Gaussian}}{\leq} \exp\left\lbrace \norm{f(X_{T-1})}{2}^2 \left(-\frac{\xi}{2} + 8\xi^2 \sigmaw \right) \right\rbrace \leq 1 \notag
\end{align}
by setting $\xi = (32\sigmaw)^{-1}$. By this choice of $\xi$, applying the law of total expectation iteratively over $t$ in~\eqref{eq:NC_p_1} and using the definition of $\gamma$ in~\eqref{eq:defs_chaining}, we obtain
\begin{align}
    \Prob\left(\max_{f\in F_{\Kl}} \sumt 4\innerprod{W_t}{f(X_t)}{2} -\frac{1}{2}\norm{f(X_t)}{2}^2 > v \right) \leq \N{\infty}{\HH}{\gamma}\exp\left\lbrace -\frac{v}{32\sigmaw}\right\rbrace;\notag
\end{align}
finally, substituting $v \leftrightarrow 32 \sigmaw(v/2 + \log \N{\infty}{\HH}{\gamma})$ and dividing by $T$, we obtain the bound for~$\NC$:
\begin{align}\label{eq:bound_NC_p}
    \boxed{\Prob\left( \sup_{f\in  F_{\Kl}} \frac{M_T(f)}{T} - \frac{1}{2T}S_T(f) > \frac{32 \sigmaw}{T}\left(\frac{v}{2} + \log \N{\infty}{\HH}{\gamma}\right)  \right) \leq \exp\left\lbrace -\frac{v}{2} \right\rbrace.}
\end{align}

\paragraph{Obtaining the final bound.}
We can now combine these results and derive the high-probability bound for the martingale offset complexity $\MOC{\HH} = \sup_{f\in\HH} N_T(f)/T$. Leveraging the decomposition of $N_T(f)$ according to~\eqref{eq:termsNANBNC}, we combine the bounds on $\NA$, $\NB$ and $\NC$ in~\eqref{eq:bound_NA_p},~\eqref{eq:bound_NB_p} and~\eqref{eq:bound_NC_p} using the union bound and obtain that
\begin{align}
    &\Prob\Bigg( \MOC{\HH} > 4w\delta\dy\sigmaw + 8\sqrt{\frac{\sigmaw}{T}}\left((u+1)\gamma + \int_{\delta}^{\gamma} \sqrt{\log \N{\infty}{\HH}{\varepsilon}}d\varepsilon \right) \notag \\
    &\qquad \qquad \qquad + \frac{32\sigmaw}{T}\left( \frac{v}{2} + \log \N{\infty}{\HH}{\gamma}\right) + 4\gamma^2\Bigg) \notag \\  &\leq \exp\left \lbrace -w+1\right\rbrace + 4\exp\left\lbrace -\frac{u^2}{2}\right\rbrace + \exp\left\lbrace -\frac{v}{2}\right\rbrace, \notag
\end{align}
 and the expression for the lower bound of the martingale offset complexity is to be maximized with respect to  $\gamma$ and $\delta < \gamma$.
 
 We now claim that we can set $\delta = 0$ and simplify the bound. Setting $\delta = 0$ is possible only if the integral in the term $\NB$ converges. Under our assumptions, we have that (see~\cref{sec:covering_vector})
 \begin{equation} 
     \int_{\delta}^{\gamma}\sqrt{\log \N{\infty}{\HH}{\varepsilon}}d\varepsilon \propto \int_{\delta}^{\gamma}\left( \frac{1}{\varepsilon} \right)^{\frac{\dx}{2s}}d\varepsilon = \frac{\varepsilon^{1-\nicefrac{\dx}{2s}}}{1-\nicefrac{\dx}{2s}}, \label{eq:dudley_entropy_integral}
 \end{equation}
and the value is finite for $\delta \to 0$ if and only if $\nicefrac{\dx}{2s} < 1$, which is guaranteed by~\cref{ass:SobolevOrder}. Therefore, we can set $\delta =0$, and since the term associated to $\NA$ becomes 0, we can also let $w \to \infty$ and increase the final probability level in the claim of the theorem.

\end{proof}

\subsection{Bound in expectation}\label{sec:MOC_exp}

Along the lines of the result in probability of the previous subsection, we now present the bound in expectation for the martingale offset complexity. 

\begin{theorem}[\cite{ziemann_statistical_2022}, Theorem~3.2.1]\label{thm:MOC_exp}
    Let~\cref{ass:densities,ass:SobolevOrder,ass:noise,ass:S-persistence} hold, and let
    $\HH$ be a convex hypothesis space belonging to $\Sobo(\X^T;\jointProb;\R^{\dy})$ and satisfying~\cref{ass:containment}.  Then, the martingale offset complexity satisfies
    \begin{align}
        \E{\MOC{\HH}}{} \leq \inf_{\gamma > 0}  \left\{ 8\int_{0}^{\gamma} \sqrt{\frac{\sigmaw\log\N{\infty}{\HH}{\varepsilon}}{T}}d\varepsilon + \frac{32\sigmaw \log \N{\infty}{\HH}{\gamma}}{T} + 4\gamma^2 \right\}. \notag
    \end{align}
\end{theorem}

\begin{proof}
    The result is again obtained by chaining, using the construction leading to~\eqref{eq:termsNANBNC}. Using the definitions of $N_T(f)$, $M_T(f)$ and $S_T(f)$ in~\eqref{eq:defsMOC}, as well as the ones for the chaining resolutions $\delta,\gamma$ in~\eqref{eq:defs_chaining}, we are looking at
    \begin{align}
\E{\sup_{f\in\HH} \frac{1}{T} N_T(f)}{}  \leq &\frac{1}{T} \E{\underbrace{\sup_{f,g \in \HH \atop \norm{f-g}{\Ellinf(\X^T;\R^{\dy})} \leq 2^{-\Ku}} M_T(f) - M_T(g)}_{\NA}}{} + \frac{1}{T} \E{\underbrace{\sup_{f\in \HH} D_T(f)}_{\NB}}{} \notag \\ & + \frac{1}{T} \E{ \underbrace{\sup_{f\in F_{\Kl} } M_T(f) - \frac{S_T(f)}{2}}_{\NC}}{} + \underbrace{\Big(   \frac{1}{2^{\Kl}}   \Big)^2}_{\stackrel{\eqref{eq:defs_chaining}}{\leq} 4\gamma^2 }. \label{eq:boundExpI}
\end{align}
We now proceed with deriving the bounds for the expected values of the terms $\NA$, $\NB$ and $\NC$; the final claim is obtained by summing all of the contributions together. Finally, we will discuss the fact that we are allowed to let $\delta =0$ and simplify the bound.

\paragraph{Bound for $\NA$.} 
Define $\Fku \doteq \lbrace f = f^{\mathfrak{a}} - f^{\mathfrak{b}};\, f^{\mathfrak{a}}, f^{\mathfrak{b}}  \in \HH \,\vert\, \norm{f}{\Ellinf(\X^T;\R^{\dy})} \leq 2^{-\Ku} \rbrace$. By linearity of $M_T(\cdot)$, we are looking at
\begin{align}
    &\E{\sup_{f \in \Fku} 
    \sumt   4\innerprod{W_t}{f(X_t)}{2}}{} 
    \leq \E{\sup_{f \in \Fku} \sumt 4\norm{W_t}{2}\norm{f(X_t)}{2}}{} \notag \\
    \leq &\frac{4}{2^{\Ku}}\sumt \E{\norm{W_t}{2}}{} \stackrel{Lemma~\ref{lemma:usefulsub-Gaussian2}}{\leq} \frac{12 T}{2^{\Ku}}\sqrt{\dy\sigmaw} \stackrel{\eqref{eq:defs_chaining}}{\leq} 24 T\delta \sqrt{\dy \sigmaw}.\notag
    \end{align}
Dividing by $T$, we obtain the first term in the bound~\eqref{eq:boundExpI}.

\paragraph{Bound for $\NB$.} We start by defining the auxiliary search space \begin{equation} \Fk \doteq \left\lbrace f = f^{\mathfrak{a}} - f^{\mathfrak{b}};\, f^{\mathfrak{a}} \in F_k, f^{\mathfrak{b}}  \in F_{k-1} \,\vert\, \norm{f}{\Ellinf(\X^T;\R^{\dy})} \leq 2^{-k}   \right\rbrace \notag
\end{equation} for all $k = \Kl+1,\dots, \Ku$. Next, using the definition of $D_T(f)$ and exploiting its linearity, we consider
\begin{align}
    \E{\sup_{f\in \HH} \sum_{k={\Kl+1}}^{\Ku}   M_T\left(\pi_k(f) - \pi_{k-1}(f)\right)}{} \leq \sum_{k={\Kl+1}}^{\Ku} \E{\sup_{f\in \Fk} M_T(f)}{}. \label{eq:NB_e_1}
\end{align}
To upper-bound the right-hand side of~\eqref{eq:NB_e_1}, we focus on its addenda and proceed with the following argument. Noting that $\Fk$ is a finite-dimensional class, let us consider, for some $\xi >0 $,
\begin{align}
    \exp\left\lbrace \xi \E{\max_{f\in\Fk} M_T(f)}{} \right\rbrace &\leq \E{\exp \left\lbrace \xi \max_{f\in\Fk} M_T(f) \right\rbrace}{} \quad \text{by Jensen's inequality,} \notag\\
    & = \E{\max_{f\in\Fk} \exp\left\lbrace \xi M_T(f) \right\rbrace}{} \quad \text{by monotonicity,}\notag\\
    &\leq \sum_{f\in\Fk} \E{\exp\left\lbrace \xi M_T(f) \right\rbrace}{} \notag \\
    & = \sum_{f\in\Fk} \E{\E{\exp\left\lbrace \xi \sumt 4 \innerprod{W_t}{f(X_t)}{2}\right\rbrace  \bigg\vert \Filt_{T-2}}{}}{} \notag \\
    &\stackrel{\eqref{eq:sub-Gaussian}} {\leq} \E{\exp\left\lbrace \xi \sum_{t=0}^{T-2} 4 \innerprod{W_t}{f(X_t)}{2}\right\rbrace  }{}\exp\left\{ \frac{8\xi^2\sigmaw}{2^{2k}} \right\} \notag \\
    &\leq \vdots \quad \text{(i.e., repeating the argument with the subsequent filtrations)} \notag \\
    &\leq \left(\N{\infty}{\HH}{2^{-k}}\right)^2\exp\left\lbrace \frac{8T\xi^2\sigmaw}{2^{2k}} \right\rbrace, \label{eq:NB_e_2} 
\end{align}
noting that the cardinality of $\Fk = F_k \times F_{k-1}$ is upper-bounded by $\left(\N{\infty}{\HH}{2^{-k}}\right)^2$. Now, taking logarithms of both sides of the whole inequality~\eqref{eq:NB_e_2}, we obtain
\begin{align}
& \E{\max_{f\in\Fk} M_T(f)}{} \leq \frac{2}{\xi}\log \N{\infty}{\HH}{2^{-k}} + \frac{8T\xi\sigmaw}{2^{2k}} \notag \\
 \rightarrow &\E{\max_{f\in\Fk} M_T(f)}{} \leq 2^{-k}\sqrt{64 T\sigmaw \log \N{\infty}{\HH}{2^{-k}}}\label{eq:NB_e_3}
\end{align}
after minimizing with respect to $\xi$. 

We can now go back to~\eqref{eq:NB_e_1}. Plugging~\eqref{eq:NB_e_3}, we obtain
\begin{align}
  \E{\sup_{f\in \HH} \sum_{k={\Kl+1}}^{\Ku}   M_T\left(\pi_k(f) - \pi_{k-1}(f)\right)}{} &\leq \sum_{k=\Kl+1}^{\Ku} \frac{1}{2^k}\sqrt{64 T\sigmaw \log \N{\infty}{\HH}{2^{-k}}}  \notag\\
&=\sum_{k=\Kl+1}^{\Ku} \left( \frac{1}{2^{k-1}} - \frac{1}{2^{k}}\right)\sqrt{64 T\sigmaw \log \N{\infty}{\HH}{2^{-k}}}  \notag\\
  & \leq \sum_{k=\Kl+1}^{\Ku} \int_{2^{-\Ku}}^{2^{-\Kl-1}}\sqrt{64 T\sigmaw \log \N{\infty}{\HH}{\varepsilon}}d\varepsilon \notag \\
  & \stackrel{\eqref{eq:defs_chaining}}{\leq} 8\int_{\delta}^{\gamma} \sqrt{ T\sigmaw \log \N{\infty}{\HH}{\varepsilon}}d\varepsilon \notag
\end{align}
having used in the second inequality a truncated Dudley entropy integral~\citep[Theorem 5.22]{wainwright_high-dimensional_2019}. Finally, the second term in~\eqref{eq:bound_NA_p} is obtained by divigind the last inequality by $T$.

\paragraph{Bound for $\NC$.} We are now working to find the upper bound for the expected value of the martingale offset complexity of a finite class of functions, $\E{\MOC{F_{\Kl}}}{}$, where  $F_{\Kl}$ is the $2^{-\Kl}$-cover of the hypothesis space $\HH$. Similarly to what has been done for $\NB$, we start by noticing
that, for any $\xi >0$,
\begin{align}
&\exp\left\lbrace \xi \E{\max_{f \in F_{\Kl}} M_T(f) - \frac{1}{2}S_T(f)}{} \right\rbrace \notag \\
\leq& \E{\max_{f\in F_{\Kl}} \exp \left\lbrace \xi \left(M_T(f) - \frac{1}{2}S_T(f) \right) \right\rbrace }{} \quad \text{ (Jensen's inequality)}\notag \\
\leq &\sum_{f\in F_{\Kl}} \E{\exp\left \lbrace \xi \left(M_T(f) - \frac{1}{2}S_T(f) \right)\right\rbrace}{}\notag \\
= & \sum_{f\in F_{\Kl}} \E{\E{ \exp\left\lbrace \xi \left(\sumt 4\xi\innerprod{W_t}{f(X_t)}{2} - \frac{\xi}{2}\norm{f(X_t)}{2}^2 \right) \right\rbrace   \bigg \vert \Filt_{T-2} }{}}{} \quad \text{(total expectation)} \notag \\
= & \sum_{f\in F_{\Kl}} \E{ \exp\left\lbrace\sum_{t=0}^{T-1} 4\innerprod{W_t}{f(X_t)}{2} - \frac{1}{2}\norm{f(X_t)}{2}^2 \right\rbrace }{} 
\notag \\ & \qquad \cdot \E{4\innerprod{W_{T-1}}{f(X_{T-1})}{2} - \frac{1}{2}\norm{f(X_{T-1})}{2}^2 \,\big\vert \Filt_{T-2}}{} \notag \\ 
\stackrel{\eqref{eq:sub-Gaussian}}{\leq} &  \sum_{f\in F_{\Kl}} \E{ \exp\left\lbrace\sum_{t=0}^{T-1} 4\innerprod{W_t}{f(X_t)}{2} - \frac{1}{2}\norm{f(X_t)}{2}^2 \right\rbrace }{}\underbrace{\exp\left\lbrace \norm{f(X_{T-1)}}{2}^2 \left( -\frac{\xi}{2} + 8\xi^2 \sigmaw \right) \right\rbrace}_{\leq 1 \text{ by letting } \xi = (32\sigmaw)^{-1}} \notag\\
\leq & \quad \vdots \quad \text{ (iterating over the subsequent filtrations)} \notag\\
\leq & \N{\infty}{\HH}{\frac{1}{2^{\Kl}}}. \label{eq:NC_e_1}
\end{align}
Now, by taking the logarithm on both sides of~\eqref{eq:NC_e_1}and using the value $\xi = (32\sigmaw)^{-1}$ found above, we obtain
\begin{align}
    \E{\max_{f \in F_{\Kl}} M_T(f) - \frac{1}{2}S_T(f)}{} \leq 32\sigmaw \log \N{\infty}{\HH}{\frac{1}{2^{\Kl}}} \stackrel{\eqref{eq:defs_chaining}}{\leq} 32\sigmaw \log \N{\infty}{\HH}{\gamma},\notag
\end{align}
and the bound for the third term in~\eqref{eq:boundExpI} is obtained by dividing the terms above by $T$.

\paragraph{Wrapping up.}
Putting the bounds for all the terms $\NA$, $\NB$ and $\NC$ together, we obtain
\begin{align}
    \boxed{\E{\MOC{\HH}}{} \leq 24\delta \sqrt{\dy \sigmaw} + \int_{\delta}^{\gamma} \sqrt{\frac{64\sigmaw \log \N{\infty}{\HH}{\varepsilon}}{T}}d\varepsilon + \frac{32\sigmaw \log \N{\infty}{\HH}{\varepsilon}}{T} + 4\gamma^2.} \notag
\end{align}
Following the reasoning at the end of the proof for~\cref{thm:MOC_prob}, we have that in the scenarios of our interest the integral does not diverge at $\delta =0$. For this reason, in the final claim we will make use of $\delta=0$ and simplify the bound.
\end{proof}

\section{Proofs of the excess risk bounds in~\cref{sec:bounds}}\label{sec:proofs_sec4}

This section provides the proof of~\cref{thm:main_probab,thm:main_exp}. As discussed in~\cref{sec:ideaBounds}, the results are derived leveraging the lower isometry event~\eqref{eq:lower_isometry_event} and the bound on its probability presented in~\cref{sec:lower_isometry_bound}. Moreover, we make use of the regularizer's properties elucidated in~\cref{sec:regularizer_properties}, and of $(C(r),2)$-hypercontractivity proved in~\cref{sec:hypercontractivity}. Ultimately, we obtain that our \emph{complexity-dependent} bounds on the excess risk feature three main ingredients: the complexity of the hypothesis class, captured by the martingale offset complexity; the critical radius $r$ identifying the set $B(r)$ and determining its size (thus, the covering number of its boundary, see~\cref{thm:new_thm312}); and the ground-truth regularizer $\reg{\fstar}$. 
These results bring together the small-ball method with learning with dependent data, and are the starting point for the derivation of our convergence rate results presented in~\cref{sec:convergence_rates} and proved in~\cref{sec:proofs_rates}.

\subsection{Proof of~\cref{thm:main_probab} (result in probability)}\label{sec:proofboundprob}

\begin{repthm}
    {\bfseries \cref{thm:main_probab}.} Let~\cref{ass:densities,ass:SobolevOrder,ass:noise,ass:containment,ass:S-persistence} hold. Consider a parameter $\theta > 8$, and let $\fhat$ be the solution of the estimation problem~\eqref{eq:RERM} with $\lambda_T > 0$, and let the radius $\rho$ defining the effective hypothesis class $\Freg$ be such that $\rho \geq 10 \reg{\fstar}$. Then, on the event
    \begin{equation}
       \Alo_r^{\complement} \cap \left\{ \lambda_T \geq \frac{40}{3\rho} \MOC{\Freg}\right\} \notag
    \end{equation}
we have that
\begin{equation}
    \norm{\fhat - \fstar}{\Elltwo(\X^T,\jointProb;\R^{\dy})}^2 \leq \theta \MOC{\Freg} + 2\lambda_T \reg{\fstar} + r^2. \notag
\end{equation}  
\end{repthm}
\begin{proof}
We start by noting that $\Alo_r^{\complement}$ can happen in the following situations: 
\begin{enumerate}[label=(\roman*)]
    \item  $\norm{\fhat-\fstar}{\Elltwo(\X^T,\jointProb;\R^{\dy})}^2 \leq r^2$;
    \item $\fhat-\fstar$ is in $\Freg \setminus B(r)$, but it happens that $\|\fhat - \fstar \|^2_{\Elltwo} \leq  \frac{\theta}{T}\sumt   \norm{\fhat(X_t) - \fstar(X_t)}{2}^2$ (see~\eqref{eq:highProbEvent});
    \item $\fhat$ is outside $\Freg$.
\end{enumerate}
The key idea of this Theorem is to prove that scenario (iii) cannot occur with our choice of the regularization parameter $\lambda_T$. We will now analyze each situation separately.

\paragraph{Case (i).} This is the simple situation in which we are already in the $\Elltwo$-ball with radius $r$, $B(r)$, leading to $\norm{f-\fstar}{\Elltwo(\X^T,\jointProb;\R^{\dy})}^2 \leq r^2$.

\paragraph{Case (ii).} On this event, we have
\begin{align}
    \norm{\fhat-\fstar}{\Elltwo(\X^T,\jointProb;\R^{\dy})}^2 &\leq \frac{\theta}{T}\sumt \norm{\fhat(X_T) - \fstar(X_t)}{2}^2 \notag\\
    &\stackrel{\eqref{eq:otherMOCineq}}{\leq} \frac{\theta}{T} \sup_{g\in\Fregstar} \sumt \left[4\innerprod{W_t}{g(X_t)}{2} - \norm{g(X_t)}{2}^2\right] + 2\lambda_T\reg{\fstar} \notag\\
    &\stackrel{\eqref{eq:MOC}}{\leq} \theta \MOC{\Freg} + 2\lambda_T\reg{\fstar}.\notag
\end{align}

\paragraph{Case (iii).}
By~\cref{lemma:propF}, the hypothesis space is convex, and the regularizer $\reg{\cdot}$ is continuous: therefore, there exists $R>1$ and $h \in \partial \Freg$ such that $\fhat = \fstar + R(h-\fstar)$. Additionally, by~\cref{def:eta_regularizer}(b), we have that 
\begin{align}
    \reg{h} \geq \frac{1}{2}\reg{h-\fstar} - \reg{\fstar} \Rightarrow \reg{h} - \reg{\fstar} \geq \frac{3\rho}{10} \label{eq:eta_reg_b_conseq}
\end{align}
by virtue of our choice $\reg{\fstar - h} = \rho$ and by the assumption $\reg{\fstar} \leq \frac{\rho}{10}$. 
We can use this in our construction and consider
\begin{align}
&\frac{1}{T}\sumt   \norm{\fhat(X_t) - \fstar(X_t)}{2}^2 \notag \\ &\stackrel{\eqref{eq:otherMOCineq_tot}}{\leq} \frac{1}{T}  \sumt   4\innerprod{W_t}{\fhat(X_t)-\fstar(X_t)}{2}  - \norm{\fhat(X_t)-\fstar(X_t)}{2}^2 + 2\lambda_T\left(\reg{\fstar} - \reg{\fhat}\right)     \notag \\
&\stackrel{\cref{lemma:lecue_inequality}}{\leq} \frac{1}{T}  \sumt   4R\innerprod{W_t}{h(X_t)-\fstar(X_t)}{2}  - R^2\norm{h(X_t)-\fstar(X_t)}{2}^2 - \frac{R\lambda_T}{4} \left(\reg{h} - \reg{\fstar}\right) \notag\\
&\leq R\left[ \frac{1}{T}  \sumt   4\innerprod{W_t}{h(X_t)-\fstar(X_t)}{2}  - \norm{h(X_t)-\fstar(X_t)}{2}^2 - \frac{\lambda_T}{4} \left(\reg{h} - \reg{\fstar}\right) \right] \notag\\
&\stackrel{\eqref{eq:eta_reg_b_conseq}}{\leq} R\left[ \frac{1}{T}  \sumt   4\innerprod{W_t}{h(X_t)-\fstar(X_t)}{2}  - \norm{h(X_t)-\fstar(X_t)}{2}^2 - \frac{3 \rho \lambda_T}{40} \right]. \notag
\end{align}
However, by taking $\lambda > 40\MOC{\Freg}/(3\rho)$, the term in the square brackets becomes negative, leading to an absurd statement.

In light of the analysis for cases (i)-(iii), it results that only cases (i) and (ii) are of interest under the assumptions of~\cref{thm:main_probab}. Therefore, it holds that
\begin{align}
    \norm{\fhat - \fstar}{\Elltwo(\X^T,\jointProb;\R^{\dy})}^2 &\leq \min\left\{\theta \MOC{\Freg} + 2\lambda_T \reg{\fstar},\, r^2\right\} \notag \\
    &\leq \theta \MOC{\Freg} + 2\lambda_T \reg{\fstar} + r^2, \notag
\end{align}
as we wanted to prove.
    \end{proof}

\subsection{Proof of~\cref{thm:main_exp} (result in expectation)}\label{sec:proofboundexp}
\begin{repthm}
    {\bfseries \cref{thm:main_exp}.} Let~\cref{ass:densities,ass:SobolevOrder,ass:noise,ass:containment,ass:S-persistence} hold. Consider a parameter $\theta > 8$, a radius $r>0$, and let $\Fr$ be a $r/\sqrt{\theta}$-cover in the infinity norm of $\partial B(r)$ that is $(C(r),2)$-hypercontractive.
    Consider the regularized empirical risk minimization problem in~\eqref{eq:RERM} with regularization parameter satisfying 
    $\lambda_T \geq \frac{40}{3\rho}\E{\MOC{\Freg}}{W}$, where 
    $\rho \geq 10 \reg{\fstar}$. Then, letting $B$ be the positive constant such that $\norm{f}{\Ellinf(\X^T;\R^{\dy})} \leq B$ for all $f \in \F$, the estimate $\fhat$ satisfies
    \begin{align}
        \E{\norm{\fhat - \fstar}{\Elltwo(\X^T, \jointProb; \R^{\dy})}^2}{} &\leq 4B^2\N{\infty}{\partial B(r)}{\frac{r}{\sqrt{\theta}}}\exp\left\{ - \frac{8T}{\theta^2 C_r S}\right\} \notag \\ &+ \theta \E{\MOC{\Freg}}{} + \lambda_T \reg{\fstar} + r^2. \notag 
    \end{align}   \vspace{-2em} \end{repthm}
\begin{proof}
    First, we observe that $\Fr$ is $(C(r),2)$-hypercontractive as shown in~\cref{thm:Fhypercontractive}, and $B$-boundedness of $\F$ (thus, also of $\Freg \subset \F$) follows from~\cref{lemma:propF}. 
    
    We now use the lower isometry event $\Alo_r$ in~\eqref{eq:lower_isometry_event} to decompose the expected value as
    \begin{align}
       \E{\norm{\fhat - \fstar}{\Elltwo(\X^T, \jointProb; \R^{\dy})}^2}{} &= \E{\norm{\left(\fhat - \fstar\right) \chi_{\Alo_r}}{\Elltwo(\X^T, \jointProb; \R^{\dy})}^2}{} \notag \\ &+ \E{\norm{\left(\fhat - \fstar\right)\chi_{\Alo_r^{\complement}}}{\Elltwo(\X^T, \jointProb; \R^{\dy})}^2}{}, \label{eq:E_decomposition}
    \end{align}
 where $\chi_{\mathfrak{A}}$ is the indicator function of the event $\mathfrak{A}$, such that it is equal to 1 if $\mathfrak{A}$ is true, and 0 otherwise. To obtain the desired bound, we proceed by analyzing the two addenda separately. 

 \paragraph{First scenario ($\Alo_r$ is true).} In the lower isometry event, we can write 
 \begin{align}
     \E{\norm{\left(\fhat - \fstar\right) \chi_{\Alo_r}}{\Elltwo(\X^T, \jointProb; \R^{\dy})}^2}{} \leq \norm{\fhat - \fstar}{\Ellinf(\X^T;\R^{\dy})}^2\jointProb(\Alo_r). \notag
 \end{align}
Then, we can bound the norm on the right-hand side of such an expression by $(2B)^2$, being \mbox{$\sup_x \norm{\fhat(x) - \fstar(x)}{\Ellinf(\X^T;\R^{\dy})} \leq \sup_x \left(\norm{\fhat(x)}{\Ellinf(\X^T;\R^{\dy})} + \norm{\fstar(x)}{\Ellinf(\X^T;\R^{\dy})}\right) \leq 2B$}, and the bound for $\jointProb(\Alo_r)$ follows by~\cref{thm:new_thm312}. Ultimately, we obtain the bound for the first addendum in~\eqref{eq:E_decomposition} as
\begin{equation}
    \E{\norm{\left(\fhat - \fstar\right) \chi_{\Alo_r}}{\Elltwo(\X^T, \jointProb; \R^{\dy})}^2}{} \leq 4B^2\N{\infty}{\partial B(r)}{\frac{r}{\sqrt{\theta}}}\exp\left\{ - \frac{8T}{\theta^2 C_r S}\right\}. \label{eq:E_AlorTrue}
\end{equation}

\paragraph{Second scenario ($\Alo_r$ is false).} This case is treated as in the high-probability bound of~\cref{thm:main_probab}. Again, we express the cases leading to the realization of $\Alo_r^{\complement}$ as (i) $\fhat \in B(r)$; (ii) $\fhat \in \Freg \setminus B(r)$, but it happens that $\|\fhat - \fstar \|^2_{\Elltwo(\X^T,\jointProb;\R^{\dy})} \leq \theta \frac{1}{T}\sumt   ||\fhat(X_t) - \fstar(X_t)||_2^2$; (iii) $\fhat \in \F\setminus \Freg$. Along the lines of~\cref{thm:main_probab}, we find an upper bound for the second addendum in~\eqref{eq:E_decomposition} by showing that, with our choice of $\lambda_T$, case (iii) does not happen.

\subparagraph{Case (i)} When $\fhat \in B(r)$, by definition we have that $\E{\norm{\fhat - \fstar}{\Elltwo(\X^T,\jointProb;\R^{\dy})}^2}{} \leq r^2$.

\subparagraph{Case (ii)} Following the steps in the proof of~\cref{thm:main_probab}, in this high-probability scenario we have that
\begin{align}
    \E{\norm{\fhat - \fstar}{\Elltwo(\X^T,\jointProb;\R^{\dy})}^2}{} \leq \theta \E{\MOC{\Freg}}{} + 2\lambda_T\reg{\fstar}. \notag
\end{align}

\subparagraph{Case (iii)} The argument in the corresponding part of the proof of~\cref{thm:main_probab} carries out also when considering the expected value, leading to an absurd conclusion as soon as $\lambda_T \geq \frac{40 \E{\MOC{\Freg}}{}}{3\rho}$.

Therefore, overall, the term for the case in which $\Alo_r$ is false (i.e., the second addendum in~\eqref{eq:E_decomposition}) is upper-bounded by
\begin{align}
    \E{\norm{\left(\fhat - \fstar\right)\chi_{\Alo_r^{\complement}}}{\Elltwo(\X^T, \jointProb; \R^{\dy})}^2}{} \leq r^2 + \theta \E{\MOC{\Freg}}{} + 2\lambda_T\reg{\fstar}; \label{eq:E_AlorFalse}
\end{align}
thus, the claim follows by upper-bounding~\eqref{eq:E_decomposition} by the sum of~\eqref{eq:E_AlorTrue} and~\eqref{eq:E_AlorFalse}.
\end{proof}

\section{Proofs of convergence rate results in~\cref{sec:convergence_rates}}\label{sec:proofs_rates}

We now present the proofs of~\cref{thm:main_rate_probability,thm:main_rate_exp}. These results build upon~\cref{thm:main_probab,thm:main_exp} and rely on specifying the martingale offset complexity bounds (\cref{sec:MOC}) and the covering number of the boundary of $B(r)$ (\cref{sec:hypercontractivity}). By setting the squared critical radius $r^2$ be dominated by the martingale offset complexity term, we obtain the desired complexity-dependent bounds for the excess risk.

\newpage
\subsection{Proof of~\cref{thm:main_rate_probability} (result in probability)}\label{sec:proof_rate_prob}
\begin{repthm}
    {\bfseries \cref{thm:main_rate_probability}.}
Let~\cref{ass:densities,ass:SobolevOrder,ass:elliptic,ass:containment,ass:noise,ass:S-persistence} hold, and let $\fhat$ be the solution of~\eqref{eq:RERM}. Fix a probability of failure $\delta \in (0,1)$, and assume the regularization parameter $\lambda_T$ satisfies
    \begin{equation}
        \lambda_T \geq \frac{4}{3T^d} \left[ \frac{C_I \sigma_W^{1+d}}{\reg{\fstar}^{1-\frac{d'}{4}}} + \frac{(C_{II}+C_{IV})\sigma_W^{2d}}{\reg{\fstar}^{1-\frac{d'}{2}}} + \frac{C_{III}\sigmaw \log(1/\delta)}{\reg{\fstar}} \right],\notag
    \end{equation}
    where $d = \nicefrac{2s}{2s+\dx}$, $d' = \nicefrac{2\dx}{2s+\dx}$, and $C_I$, $C_{II}$, $C_{III}$ and $C_{IV}$ are constants depending only on $s,\dx,\dy$ and $\sqrt{\log(1/\delta)}$.
    If the number of samples $T$ satisfies
    \begin{equation}
T \geq \frac{\theta^2 C_h S}{8}\left[C_M\left(\frac{1}{r}\right)^{\frac{6\dx}{2s-\dx}}\log\left(1 + C_L\left(\frac{1}{r}\right)^{\frac{4s-\dx}{2s-\dx}} \right) + \left(\frac{1}{r}\right)^{\frac{4\dx}{2s-\dx}}\log(1/\delta) \right] \notag        
    \end{equation}
    for $r^2 = \lambda_T \reg{\fstar} + \sigmaw/T$ and $C_h,C_M,C_L$ being uniform constants depending on $\rho_f,\kappau,\theta,s,\dx$ and $\X$, then, with probability at least $1-6\delta$, the excess risk enjoys the following convergence rate:
    \begin{equation}
        \norm{\fhat - \fstar}{\Elltwo(\X^T,\jointProb;\R^{\dy})}^2 \leq \Cslow \frac{\max\left\lbrace \reg{\fstar}^{d'/4}, \reg{\fstar}^{d'/2}\right\rbrace}{T^d} + \Cfast \frac{\sigmaw \log(1/\delta)}{T}, \notag
    \end{equation}
    where $\Cslow$ is a constant that depends on $s, \dx, \dy, \sigmaw, \sqrt{\log(1/\delta)}$, and $\Cfast$ is a constant that depends on $s,\dx,\dy$.
    \end{repthm} 

\begin{proof}

The starting point is the bound in probability on the excess risk of~\cref{thm:main_probab} given in~\cref{eq:ER_bound_prob_general}. As one of its main ingredients is the bound on the martingale offset complexity, we start the proof by characterizing such a bound reported in~\cref{thm:MOC_prob} for the effective hypothesis space $\Freg$. A key role is also played by the covering number of $\Freg$, which is derived in~\cref{prop:coveringFreg}. Next, we choose the parameters $\rho$, $\lambda_T$ and $r^2$ according to the requirements of~\cref{thm:main_probab}, and this leads to the desired excess risk bound. The proof is concluded by characterizing the lower isometry event probability, which leads to the specification of the burn-in time stated in the claim.

\paragraph{Martingale offset complexity bound.}
We start by determining the bound for $\MOC{\Freg}$ entering~\eqref{eq:ER_bound_prob_general} using the general result of~\cref{thm:MOC_prob}. By setting $u = \sqrt{2\log(1/\delta)}$ and $v = 2\log(1/\delta)$, we have that, with probability at least $1-5\delta$,
\begin{align}
    \MOC{\Freg} \leq \inf_{\gamma > 0}\, &\Bigg\lbrace 8\gamma \sqrt{\frac{\sigmaw}{T}}(1 + \sqrt{2\log(1/\delta)}) + 8\sqrt{\frac{\sigmaw}{T}}\int_{0}^{\gamma} \sqrt{\N{\infty}{\Freg}{\varepsilon}}d\varepsilon \notag \\
    &\quad + \frac{64 \sigmaw \log(1/\delta)}{T} + \frac{32\sigmaw}{T}\N{\infty}{\Freg}{\gamma} + 4\gamma^2\Bigg\rbrace. \notag
\end{align}
By using the covering number result in~\cref{prop:coveringFreg} and noting that, according to~\eqref{eq:dudley_entropy_integral}, 
\begin{equation}
    \int_{0}^{\gamma}\left(\frac{1}{\varepsilon}\right)^{\frac{\dx}{2s}}d\varepsilon = \frac{\gamma^{1-\nicefrac{\dx}{2s}}}{1-\nicefrac{\dx}{2s}}, \label{eq:dudley_2}
\end{equation} the bound on the martingale offset complexity can be re-written as
\begin{align}
    \MOC{\Freg} \leq \inf_{\gamma > 0}\, &\Bigg\lbrace 8\gamma \sqrt{\frac{\sigmaw}{T}}(1 + \sqrt{2\log(1/\delta)}) + 8\sqrt{\frac{\sigmaw}{T}}\frac{\sqrt{\Cc} \dy^{\frac{2s+\dx}{4s}}}{1-\frac{\dx}{2s}}(\sqrt{\rho})^{\frac{\dx}{2s}}\gamma^{1 - \nicefrac{\dx}{2s}} \notag \\
    &\quad + \frac{64 \sigmaw \log(1/\delta)}{T} + \frac{32\sigmaw}{T}\Cc \dy^{\frac{2s+\dx}{2s}}\left( \frac{\sqrt{\rho}}{\gamma}\right)^{\frac{\dx}{s}} + 4\gamma^2\Bigg\rbrace. \label{eq:MOC_bound_prob_rate_1}
\end{align}
By following the reasoning presented in~\cite{liang_learning_2015} (see also~\cite{yang_information-theoretic_1999}), minimization over $\gamma$ is obtained by equating the last two terms in~\eqref{eq:MOC_bound_prob_rate_1}, which yields
\begin{equation}
    \gamma = \left(8 \Cc \dy^{\frac{2s+\dx}{2s}}\right)^{\frac{s}{2s+\dx}}\left( \frac{\sigmaw}{T}\right)^{\frac{s}{2s+\dx}}\left(\sqrt{\rho}\right)^{\frac{\dx}{2s+\dx}}. \notag 
\end{equation}
Plugging in such a value for $\gamma$ in~\eqref{eq:MOC_bound_prob_rate_1}, and recalling the definitions $d \doteq \nicefrac{2s}{2s+\dx}$ and $d' = \nicefrac{2\dx}{2s+\dx}$, we obtain
\begin{align}
    \MOC{\Freg} &\leq C_I \frac{\sigma_W^{1+d}}{T^{\frac{1+d}{2}}}(\sqrt{\rho})^{\frac{d'}{2}} + C_{II}\frac{(\sigmaw)^d}{T^d}(\sqrt{\rho})^{d'} + C_{III} \frac{\sigmaw \log(1/\delta)}{T} + C_{IV} \frac{(\sigmaw)^d}{T^d}(\sqrt{\rho})^{d'},\notag \\
    &\text{where} \quad 
    \begin{cases}
        C_I &\doteq 8(1 + \sqrt{2\log(1/\delta)})\left(8\Cc \dy^{1/d}\right)^{d/2}\\
        C_{II} &\doteq \frac{2s}{2s-\dx}8\sqrt{\Cc \dy^{1/d}}\left( 8 \Cc \dy^{1/d} \right)^{\frac{2s-\dx}{2(2s+\dx)}}\\
        C_{III} &\doteq 64\\
        C_{IV} &\doteq 8\left(8\Cc \dy^{1/d}\right)^d
    \end{cases} \label{eq:MOC_bound_prob_rate_2}
\end{align}

\paragraph{Choice of the parameters $\rho$, $\lambda_T$ and $r$.}
According to~\cref{thm:main_probab}, we set the radius of the effective hypothesis class $\Freg$ to satisfy $\rho = 10 \reg{\fstar}$; similarly, the regularization parameter is chosen as $\lambda_T = \frac{40}{3\rho}\MOC{\Freg}$. Regarding the radius of the $\Elltwo$-ball $\balltwo{r}$, we conveniently set it as $r^2 = \lambda_T\reg{\fstar} + \frac{\sigmaw \log(1/\delta)}{T}$. Thanks to these choices, the excess risk bound in~\eqref{eq:ER_bound_prob_general} reads as
\begin{align}
    \norm{\fhat - \fstar}{\Elltwo(\X^T,\jointProb;\R^{\dy})}^2 \leq (\theta+4) &\Bigg(C_I \frac{\sigma_W^{1+d}}{T^{\frac{1+d}{2}}}(\sqrt{\rho})^{\frac{d'}{2}} + C_{II}\frac{(\sigmaw)^d}{T^d}(\sqrt{\rho})^{d'} \notag\\&+ C_{III} \frac{\sigmaw \log(1/\delta)}{T} + C_{IV} \frac{(\sigmaw)^d}{T^d}(\sqrt{\rho})^{d'} \Bigg) + \frac{\sigmaw \log(1/\delta)}{T}. \notag 
\end{align}
Now, noting that $(1+d)/2 > d$, we obtain the desired claim, namely that 
\begin{align}
    &\norm{\fhat - \fstar}{\Elltwo(\X^T,\jointProb;\R^{\dy})}^2 \leq \Cslow \frac{\max\left\lbrace \reg{\fstar}^{d'/4}, \reg{\fstar}^{d'/2}\right\rbrace}{T^d} + \Cfast \frac{\sigmaw \log(1/\delta)}{T},\notag\\
    &\text{where }\quad \begin{cases}
        \Cslow & \doteq (\theta+4)\left(C_I 10^{d'/4}\sigma_W^{1+d} + C_{II}10^{d'/2}\sigma_W^{2d} + C_{IV}\sigma_W^{2d}10^{d'/2}\right)\\
        \Cfast &\doteq (\theta+4)C_{III} + 1.
    \end{cases}\notag 
\end{align}

\paragraph{Characterization of the burn-in time.} We conclude the proof by setting the probability of the lower isometry event $\Alo_r$ equal to $\delta$, so that the overall claim can hold with the desired probability $1 - 5\delta - \delta$.

By~\cref{cor:thm312}, we have the following bound for the probability of the lower isometry event: \begin{align}
\jointProb(\Alo_r) \leq \left(C_L \left(\frac{1}{r}\right)^{\frac{4s-\dx}{2s-\dx}} +1 \right)^{\dy C_m \left(\frac{1}{r}\right)^{\frac{2\dx}{2s-\dx}}}\exp\left\lbrace -\frac{8T r^{\frac{4\dx}{2s-\dx}}}{\theta^2 C_h S} \right\rbrace \stackrel{!}{\leq} \delta. \notag
\end{align}
Taking logarithms on both sides of the last inequality, letting $C_M  \doteq C_m \dy$, we obtain that $T$ has to satisfy the condition
\begin{align}
    T \geq \frac{\theta^2 C_h S}{8}\left[C_M\left(\frac{1}{r}\right)^{\frac{6\dx}{2s-\dx}}\log\left(1 + C_L\left(\frac{1}{r}\right)^{\frac{4s-\dx}{2s-\dx}} \right) + \left(\frac{1}{r}\right)^{\frac{4\dx}{2s-\dx}}\log(1/\delta) \right]. \notag
\end{align}
The effective condition is obtained by substituting $r^2 = \lambda_T\reg{\fstar} + \sigmaw/T$. 
\end{proof}

\subsection{Proof of~\cref{thm:main_rate_exp} (result in expectation)}\label{sec:proof_rate_exp}

\begin{repthm}
 {\bfseries \cref{thm:main_rate_exp}. }    Let~\cref{ass:densities,ass:SobolevOrder,ass:elliptic,ass:containment,ass:noise,ass:S-persistence} hold, and let $\fhat$ be the solution of~\eqref{eq:RERM} with regularization parameter $\lambda_T$ satisfying
    \begin{equation}
        \lambda_T \geq \frac{4(C_I + C_{II})(\sigmaw)^d}{3 T \reg{\fstar}^{1 - \frac{d'}{2}}}, \notag
    \end{equation}
 where $d = \nicefrac{2s}{2s+\dx}$ is the Sobolev minimax rate, $d' = \nicefrac{2\dx}{2s+\dx}$, and $C_I$ and $C_{II}$ are constants depending only on $s,\dx$ and ~$\dy$.
    If the amount of samples $T$ satisfies
    \begin{equation}
T \geq \frac{\theta^2 C_h S}{8}\left(\frac{1}{r}\right)^{\frac{4\dx}{2s-\dx}} \left[ C_M \left(\frac{1}{r} \right)^{\frac{2\dx}{2s-\dx}}\log\left(4B^2\left(1 + C_L\left(\frac{1}{r}\right)^{\frac{4s-\dx}{2s-\dx}}\right) \right) + \log\left(\frac{\sigmaw}{T}\right)\right],\notag    
    \end{equation}
    where $B$ is such that $\norm{f}{\Ellinf(\X^T;\R^{\dy})}\leq B$ for all $f\in\F$ and $C_M,C_h,C_L$ are constants depending on $\rho_f,\kappau,\theta,s,\dx$ and $\X$, then the excess risk enjoys the following convergence rate:
    \begin{equation}
        \E{\norm{\fhat - \fstar}{\Elltwo(\X^T,\jointProb;\R^{\dy})}^2}{} \leq \Cslow \frac{ \reg{\fstar}^{d'/2}}{T^d} + \Cfast \frac{\sigmaw \log(1/\delta)}{T}, \notag
    \end{equation}
    where $\Cslow$ and $\Cfast$ are constants that depend on $s, \dx, \dy$ and~$\sigmaw$.
\end{repthm}
\begin{proof}
    Similarly to the proof of~\cref{thm:main_rate_probability}, we start by characterizing the bound on the expected value of the martingale offset complexity of $\Freg$. Next, by choosing the parameters $\rho$, $\lambda_T$ and $r$ according to the requirements of~\cref{thm:main_exp}, we arrive to the desired claim on the bound. Finally, we discuss the burn-in time by characterizing the lower-isometry event probability.

    \paragraph{Martingale offset complexity bound.} As stated in~\cref{thm:MOC_exp}, we have that
    \begin{align}
        \E{\MOC{\Freg}}{} \leq \inf_{\gamma > 0} 8\sqrt{\frac{\sigmaw}{T}}\int_{0}^{\gamma} \sqrt{\log \N{\infty}{\Freg}{\varepsilon}}d\varepsilon + \frac{32 \sigmaw \log\N{\infty}{\Freg}{\gamma}}{T} + 4\gamma^2. \notag
    \end{align}
Leveraging~\ref{prop:coveringFreg} to characterize the metric entropy of $\Freg$ and leveraging~\eqref{eq:dudley_2}, such a bound can be re-written as
\begin{align}
    \E{\MOC{\Freg}}{} \leq \inf_{\gamma > 0} &8\sqrt{\frac{\sigmaw}{T}}\frac{\sqrt{\Cc \dy^{\frac{2s+\dx}{2s}}}}{1-\nicefrac{\dx}{2s}}(\sqrt{\rho})^{\frac{\dx}{2s}}\gamma^{1-\frac{\dx}{2s}} \notag + \frac{32\sigmaw}{T}(\Cc \dy^{\frac{2s+\dx}{2s}})\left(\frac{\sqrt{\rho}}{\gamma}\right)^{\frac{\dx}{s}} + 4\gamma^2. \notag
\end{align}
As done in~\cite{liang_learning_2015,yang_information-theoretic_1999}, we minimize the right-hand side by balancing the last two addenda, which leads to
\begin{equation}
    \gamma = \left(8\Cc \dy^{\frac{2s+\dx}{2s}}\right)^{\frac{s}{2s + \dx}}\left( \frac{\sigmaw}{T}\right)^{\frac{s}{2s+\dx}}(\sqrt{\rho})^{\frac{\dx}{2s+\dx}}. \notag
\end{equation}
    By substituting such a value of $\gamma$ in the martingale offset complexity bound, we obtain that
    \begin{align}
        \E{\MOC{\Freg}}{} \leq (C_I + C_{II}) \left(\frac{\sigmaw}{T}\right)^d (\sqrt{\rho})^{d'},\notag
    \end{align}
    where we recall that $d = \nicefrac{2s}{2s+\dx}$ and $d' = \nicefrac{2\dx}{2s+\dx}$, and the constants $C_I$ and $C_II$ are equal to
    \begin{align}
        \begin{cases}
           C_I &\doteq \frac{8\sqrt{\Cc \dy^d}}{1 - \nicefrac{\dx}{2s}}\left(8\Cc \dy^{\frac{1}{d}}\right)^{\frac{2s-\dx}{2(2s+\dx)}} \\
           C_{II} &\doteq 8\left(8\Cc \dy^{\frac{1}{d}}\right)^d
        \end{cases}\notag 
    \end{align}

\paragraph{Choosing parameters $\rho$, $\lambda_T$ and $r$.} We proceed by following the requirements of~\cref{thm:main_exp}, setting $\rho = 10\reg{\fstar}$ and $\lambda_T = \frac{40}{3\rho}\E{\MOC{\Freg}}{}$. Furthermore, setting $r^2 = 2\lambda_T\reg{\fstar} + \frac{\sigmaw}{T}$, we obtain that the desired bound reads as
\begin{align}
    \E{\norm{\fhat - \fstar}{\Elltwo(\X^T,\jointProb;\R^{\dy})}^2}{} \leq &4B^2 \N{\infty}{\partial \balltwo{r}}{\frac{r}{\sqrt{\theta}}}\exp\left\lbrace -\frac{8T}{\theta^2 C(r)S}\right\rbrace  \notag \\ 
    &+ (\theta+4)10^{\frac{d'}{2}}(C_I + C_{II})\left( \frac{\sigmaw}{T}\right)^d \reg{\fstar}^{\frac{d'}{2}} + \frac{ \sigmaw}{T}.\label{eq:tmp_ER_bound_exp} 
\end{align}

\paragraph{Characterizing the burn-in.} We conclude the proof by imposing that the first term on the right-hand side of~\eqref{eq:tmp_ER_bound_exp} is upper-bounded by $\frac{\sigmaw}{T}$, i.e., \begin{align}
    4B^2\N{\infty}{\partial B(r)}{\frac{r}{\sqrt{\theta}}}\exp\left\{ - \frac{8T}{\theta^2 C(r) S}\right\} \leq \frac{\sigmaw}{T}.\notag
\end{align}
Leveraging~\cref{cor:thm312}, we deploy the values for the covering number and the hypercontractivity parameter and obtain that the number of samples $T$ has to satisfy
\begin{align}
    T \geq \frac{\theta^2 C_h S}{8}\left(\frac{1}{r}\right)^{\frac{4\dx}{2s-\dx}} \left[ C_M \left(\frac{1}{r} \right)^{\frac{2\dx}{2s-\dx}}\log\left(4B^2\left(1 + C_L\left(\frac{1}{r}\right)^{\frac{4s-\dx}{2s-\dx}}\right) \right) + \log\left(\frac{\sigmaw}{T}\right)\right].\notag
\end{align}
The effective condition for the burn-in is obtained by letting $r^2 = 2\lambda_T \reg{\fstar} + \sigmaw/T$. Ultimately, with such a choice of $T$, we obtain that
\begin{align}
    \E{\norm{\fhat - \fstar}{\Elltwo(\X^T,\jointProb;\R^{\dy})}^2}{} \leq \Cslow \frac{\reg{\fstar}^{\frac{d'}{2}}}{T^d} + \Cfast \frac{\sigmaw}{T},\notag 
\end{align}
where the constants read as
\begin{align}
    \begin{cases}
      \Cslow &\doteq (\theta + 4)10^{\frac{d'}{2}}(C_I + C_{II})(\sigmaw)^d\\\Cfast &\doteq 2.
    \end{cases} \notag
\end{align}
\end{proof}

\section{Results for the case without physics-informed regularization}\label{sec:results_noreg}

We now derive the bounds in the situation in which there is no physics-informed regularization (i.e., $\lambda_T = 0$ in~\eqref{eq:RERM}). The obtained bounds will be the benchmark against which we compare~\cref{thm:main_rate_probability,thm:main_rate_exp}, showing the impact of knowledge alignment in the excess risk bounds to obtain faster convergence.  

We start by recalling that, when physics-informed regularization is absent, the empirical risk minimization problem~\eqref{eq:RERM} reads as
\begin{equation}
  \fhat' = \argmin_{f\in \F} \frac{1}{T}\sumt \norm{Y_t - f(X_t)}{2}^2,  \label{eq:ERM}
\end{equation}
and the lower isometry event takes this expression:
\begin{equation}
    \Alo_r' \doteq \sup_{f \in \Fstar \setminus B(r)} \left\{\frac{1}{T}\sumt   \norm{f(X_t)}{2}^2 - \frac{1}{\theta}\norm{f}{\Elltwo(\X^T,\jointProb; \R^{\dy})}^2  \leq 0\right\}. \label{eq:lower_isometry_noreg}
\end{equation}

\subsection{Corollary of~\cref{thm:main_rate_probability} (result in probability)}\label{sec:noreg_prob}

\begin{corollary}
Let~\cref{ass:densities,ass:SobolevOrder,ass:noise,ass:containment,ass:S-persistence} hold, and let $\fhat'$ be the solution of the estimation problem~\eqref{eq:ERM}. Setting $\theta > 8$, if
   \begin{equation}
       T \geq \frac{\theta^2 C_h S}{8}\left[C_M\left(\frac{1}{r}\right)^{\frac{6\dx}{2s-\dx}}\log\left(1 + C_L\left(\frac{1}{r}\right)^{\frac{4s-\dx}{2s-\dx}} \right) + \left(\frac{1}{r}\right)^{\frac{4\dx}{2s-\dx}}\log(1/\delta) \right], \notag
   \end{equation}
   where $C_h, C_M, C_L$ are uniform constants, then the excess risk satisfies
\begin{equation}
    \norm{\fhat' - \fstar}{\Elltwo(\X^T,\jointProb;\R^{\dy})}^2 \leq \Cslow' \frac{1}{T^d} + \Cfast'\frac{\sigmaw}{T}, \notag
\end{equation}  
with probability at least $1-6\delta$, where $\Cslow'$ is a constant that depends only on $s,\dx,\dy,\rho_f,\sigmaw,\theta$ and $\sqrt{\log(1/\delta)}$, and $\Cfast'$ is a constant that depends only on $\theta$.
\end{corollary}
\begin{proof}
We first adapt~\cref{thm:main_probab} to the non-regularized case, and then derive the bounds along the lines of~\cref{thm:main_rate_probability}.

\paragraph{Expression for the bound.}
    By adapting~\cref{thm:main_probab} to the lower-isometry event~\eqref{eq:lower_isometry_noreg}, we have that the event $(\Alo_r')^{\complement}$ happens in two scenarios: 
    \begin{enumerate}[label = (\roman*)]
        \item $\fhat' \in B(r) \Longrightarrow \norm{\fhat'-\fstar}{\Elltwo(\X^T,\jointProb;\R^{\dy})}^2 \leq r^2$;
        \item  $\fhat - \fstar$ belong to $\F \setminus B(r)$, but it holds that $$\norm{\fhat'-\fstar}{\Elltwo(\X^T,\jointProb;\R^{\dy})}^2 \leq \frac{\theta}{T}\sumt \norm{\fhat'(X_t) - \fstar(X_t)}{2}^2 \stackrel{}{\leq} \theta \MOC{\F}.$$ 
    \end{enumerate}
 Therefore, on $(\Alo_r')^{\complement}$, we have that
 \begin{equation}
     \norm{\fhat'-\fstar}{\Elltwo(\X^T,\jointProb;\R^{\dy})}^2 \leq \theta \MOC{\F} + r^2.\label{eq:bound_prob_noreg_1}
 \end{equation}

 \paragraph{Bounding the martingale offset complexity.}
We can now rely on~\cref{thm:MOC_prob} to bound the martingale offset complexity: setting $u = \sqrt{2\log(1/\delta)}$ and $v = 2\log(1/\delta)$, and using~\cref{lemma:cucker_smale_vector} to characterize the metric entropy of the hypothesis space $\F$, we have that
\begin{align}
    \MOC{\F} \leq \inf_{\gamma > 0} &8\sqrt{\frac{\sigmaw}{T}}(\sqrt{2\log(1/\delta)}+1) + 8\sqrt{\frac{\sigmaw}{T}} \frac{\sqrt{\Cc' \dy^{\frac{2s+\dx}{2s}}}}{1-\nicefrac{\dx}{2s}}\rho_f^{\frac{\dx}{2s}}\gamma^{\frac{2s-\dx}{2s}} \notag \\
    & +\frac{64\sigmaw \log(1/\delta)}{T} + \frac{32\sigmaw}{T}\Cc' \dy^{\frac{2s+\dx}{2s}}\left(\frac{\rho_f}{\gamma}\right)^{\frac{\dx}{s}} + 4\gamma^2. \label{eq:MOC_prob_noreg_1}
\end{align}
As done in~\cref{thm:main_rate_probability}, we balance the last two addenda to get
\begin{equation}
    \gamma = \left(8 \Cc' \dy^{\frac{2s+\dx}{2s}}\right)^{\frac{s}{2s+\dx}} \rho_f^{\frac{\dx}{2s+\dx}}\left(\frac{\sigmaw}{T}\right)^{\frac{s}{2s+\dx}}.\notag
\end{equation}
 Substituting into~\eqref{eq:MOC_prob_noreg_1} and recalling the definition of the Sobolev minimax rate $d = 2s/(2s+\dx)$, we obtain
 \begin{align}
     \MOC{\F} &\leq C_I'\left(\frac{\sigmaw}{T}\right)^{\frac{d+1}{2}} + C_{II}'\left(\frac{\sigmaw}{T}\right)^d + C_{III}'\log(1/\delta)\frac{\sigmaw}{T} + C_{IV}'\left(\frac{\sigmaw}{T}\right)^d,\notag\\
     &\text{where }\quad \begin{cases}
         C_I' &\doteq 8\left( 8\Cc' \dy^{\frac{2s+\dx}{2s}}\right)^{\frac{s}{2s+\dx}}(1 + \sqrt{2\log(1/\delta)})\rho_f^{\frac{\dx}{2s+\dx}}\\
         C_{II}' &\doteq 8\frac{\sqrt{\Cc'\dy^{\frac{2s+\dx}{2s}}}}{1-\nicefrac{\dx}{2s}}\left(8\Cc' \dy^{\frac{2s+\dx}{2s}}\right)^{\frac{2s-\dx}{2(2s+\dx)}}\rho_f^{\frac{2\dx}{2s+\dx}}\\
         C_{III}' &\doteq 64\\
         C_{IV}' &\doteq 8\left(8\Cc' \dy^{\frac{2s+\dx}{2s}}\right)^{\frac{2s}{2s+\dx}}\rho_f^{\frac{2\dx}{2s+\dx}}.
     \end{cases}\notag
 \end{align}
\paragraph{Final expression for the bound.} Going back to~\eqref{eq:bound_prob_noreg_1}, noting that $T^{-(d+1)/2} < T^{-d}$ and setting $r^2 \leq \sigmaw/T$, we obtain the bound for the excess risk
\begin{align}
    \norm{\fhat'-\fstar}{\Elltwo(\X^T,\jointProb;\R^{\dy})}^2 &\leq  \frac{\Cslow'}{T^d} + \frac{\Cfast' \sigmaw \log(1/\delta)}{T}\notag
\end{align}
where $\Cslow' = \theta(C_I' + C_{II}' + C_{IV}')\max\lbrace \sigma_W^{1+d}, \sigma_W^{2d} \rbrace$ and $\Cfast' = 1+(\theta)C_{III}'$.
\paragraph{Characterization of the burn-in.}
Similarly to the derivation in~\cref{thm:main_rate_probability}, the burn-in condition consists in having the amount $T$ satisfying
\begin{align}
    T \geq \frac{\theta^2 C_h S}{8}\left[C_M\left(\frac{1}{r}\right)^{\frac{6\dx}{2s-\dx}}\log\left(1 + C_L\left(\frac{1}{r}\right)^{\frac{4s-\dx}{2s-\dx}} \right) + \left(\frac{1}{r}\right)^{\frac{4\dx}{2s-\dx}}\log(1/\delta) \right] \notag
\end{align}
where $r^2 \leq \sigmaw/T$.
\end{proof}

\subsection{Corollary of~\cref{thm:main_rate_exp} (result in expectation)}\label{sec:noreg_exp}

\begin{corollary}
Let~\cref{ass:densities,ass:SobolevOrder,ass:noise,ass:containment,ass:S-persistence} hold, and let $\fhat'$ be the solution of the estimation problem~\eqref{eq:ERM}. Let $B$ the infinity-norm bound of functions in $\F$ and $\theta > 8$. If the number of samples $T$ satisfies
    \begin{equation}
         T \geq \frac{\theta^2 C_h S}{8}\left(\frac{1}{r}\right)^{\frac{4\dx}{2s-\dx}} \left[ C_M \left(\frac{1}{r} \right)^{\frac{2\dx}{2s-\dx}}\log\left(4B^2\left(1 + C_L\left(\frac{1}{r}\right)^{\frac{4s-\dx}{2s-\dx}}\right) \right) + \log\left(\frac{\sigmaw}{T}\right)\right]\notag
    \end{equation}
    with $r \leq \sqrt{\sigmaw/T}$ and uniform constants $C_h, C_L, C_M$, then we have that
    \begin{equation}
        \E{\norm{\fhat'-\fstar}{\Elltwo(\X^T,\jointProb;\R^{\dy})}^2}{} \leq  \frac{\Cslow'}{T^d} + \Cfast'\frac{\sigmaw}{T} \notag 
    \end{equation}
\end{corollary}

\begin{proof}
Similarly to the previous Corollary, we first adapt~\cref{thm:main_exp} to the non-regularized case and then derive the bounds following~\cref{thm:main_rate_exp}.

\paragraph{Expression for the bound.}
Considering the lower-isometry event~\eqref{eq:lower_isometry_noreg}, we decompose the expected value as follows:
\begin{align}
    \E{\norm{\fhat' - \fstar}{\Elltwo(\X^T, \jointProb; \R^{\dy})}^2}{} &= \overbrace{\E{\norm{\left(\fhat' - \fstar\right) \chi_{\Alo_r'}}{\Elltwo(\X^T, \jointProb; \R^{\dy})}^2}{}}^{\text{(I)}} \notag \\ &+ \underbrace{\E{\norm{\left(\fhat' - \fstar\right)\chi_{(\Alo_r')^{\complement}}}{\Elltwo(\X^T, \jointProb; \R^{\dy})}^2}{}}_{\text{(II)}}\label{eq:bound_exp_noreg_1}
\end{align}
along the lines of the proof of~\cref{thm:main_exp}. Looking at the two addenda separately:
\begin{enumerate}[label=(\Roman*)]
    \item when $\Alo_r'$ is true, we have that
    \begin{align}
        \E{\norm{\left(\fhat - \fstar\right) \chi_{\Alo_r'}}{\Elltwo(\X^T, \jointProb; \R^{\dy})}^2}{} &\leq \norm{\fhat' - \fstar}{\Ellinf(\X^T;\R^{\dy})}^2\jointProb(\Alo_r')\notag \\
         \stackrel{\cref{thm:new_thm312}}{\leq} &4B^2 \N{\infty}{\partial B(r)}{\frac{r}{\sqrt{\theta}}}\exp\left\lbrace -\frac{8T}{\theta^2 C(r) S} \right\rbrace \notag
    \end{align}
    by bounding the worst-case distance between $\fhat'$ and $\fstar$.
    \item when $\Alo_r'$ is false, there are two scenarios possible: (i) $\fhat' \in B(r)$; or (ii) $\fhat \in \F \setminus B(r)$, but it happens that $\|\fhat' - \fstar \|^2_{\Elltwo} \leq  \frac{\theta}{T}\sumt   ||\fhat'(X_t) - \fstar(X_t)||_2^2$. Looking at these two cases:
    \begin{enumerate}[label=(\roman*)]
        \item $\|\fhat' - \fstar \|^2_{\Elltwo(\X^T,\jointProb;\R^{\dy})} \leq r^2$ by definition;
        \item $\|\fhat' - \fstar \|^2_{\Elltwo(\X^T,\jointProb;\R^{\dy})} \leq \theta \MOC{\F}$.
    \end{enumerate}  
\end{enumerate}
Bringing all these terms together, the bound in~\eqref{eq:bound_exp_noreg_1} reads as
\begin{align}
   \E{\norm{\fhat' - \fstar}{\Elltwo(\X^T, \jointProb; \R^{\dy})}^2}{} &\leq 4B^2\N{\infty}{\partial B(r)}{\frac{r}{\sqrt{\theta}}}\exp\left\{ - \frac{8T}{\theta^2 C(r) S}\right\} \notag \\ &+ \theta \E{\MOC{\F}}{} + r^2. \label{eq:bound_exp_noreg_2}
\end{align}

\paragraph{Bounding the martingale offset complexity.}
We can proceed by upper-bounding the martingale offset complexity term by deploying~\cref{thm:MOC_exp}. Specifically, we have that
\begin{align}
    \E{\MOC{\F}}{} \leq \inf_{\gamma > 0} 8\sqrt{\frac{\sigmaw}{T}}\int_{0}^{\gamma} \sqrt{\log \N{\infty}{\F}{\varepsilon}}d\varepsilon + \frac{32\sigmaw \log \N{\infty}{\F}{\gamma}}{T} + 4\gamma^2.\notag
\end{align}
We deploy~\cref{lemma:cucker_smale_vector} to provide the expression for the metric entropy of $\F$ (see also the proof of~\cref{thm:main_rate_exp}) and obtain
\begin{align}
    \E{\MOC{\F}}{} \leq \inf_{\gamma > 0} &8\sqrt{\frac{\sigmaw}{T}}
    \frac{\sqrt{\Cc' \dy^{\frac{2s+\dx}{2s}}}}{1-\nicefrac{\dx}{2s}}\rho_f^{\frac{\dx}{2s}}\gamma^{\frac{2s-\dx}{2s}} + \frac{32\sigmaw}{T}\Cc'\dy^{\frac{2s+\dx}{2s}}\left(\frac{\rho_f}{\gamma}\right)^{\frac{\dx}{s}} + 4\gamma^2 \notag
\end{align}
Balancing the last two terms of the right-hand side, we obtain
\begin{equation}
    \gamma = \left(8\Cc' \dy^{\frac{2s+\dx}{2s}}\rho_f^{\frac{\dx}{s}}\right)^{\frac{s}{2s+\dx}}\left(\frac{\sigmaw}{T}\right)^{\frac{s}{2s+\dx}},\notag
\end{equation}
which we substitute back in the expression of the martingale offset complexity to get (recalling that $d = 2s/(2s+\dx)$ is the Sobolev minimax exponent)
\begin{align}
    &\E{\MOC{\F}}{} \leq \frac{C_I'}{T^d} + \frac{C_{II}'}{T^d},\notag \\
    &\text{where }\quad \begin{cases}
        C_I' &\doteq 8\frac{\sqrt{\Cc' \dy^{\frac{2s+\dx}{2s}}}}{1-\nicefrac{\dx}{2s}}\left(8\Cc' \dy^{\frac{2s+\dx}{2s}}\rho_f^{\frac{\dx}{s}}\right)^{\frac{s}{2s+\dx}} \rho_f^{\frac{\dx}{s} + \frac{\dx}{2s+\dx}}\\
        C_{II}' &\doteq 8\left(8\Cc' \dy^{\frac{2s+\dx}{2s}}\rho_f^{\frac{\dx}{s}}\right)^{\frac{2s}{2s+\dx}}.
    \end{cases}\label{eq:noreg_MOC_exp}
\end{align}
\paragraph{Final expression for the bound.} We can now go back to the excess risk bound~\eqref{eq:bound_exp_noreg_2}. We let $r^2 \leq \sigmaw/T$, substitute~\eqref{eq:noreg_MOC_exp} in the expected value for the martingale offset complexity, and let the first addendum in~\eqref{eq:bound_exp_noreg_2} be upper-bounded by $\frac{\sigmaw}{T}$. Ultimately, this yields
\begin{align}
    &\E{\norm{\fhat'-\fstar}{\Elltwo(\X^T,\jointProb;\R^{\dy})}^2}{} \leq  \frac{\Cslow'}{T^d} + \Cfast'\frac{\sigmaw}{T}, \notag\\
    &\text{where }\quad \begin{cases}
        \Cslow' &\doteq \theta(C_I'+C_{II}')(\sigmaw)^d\\
        \Cfast' &\doteq 2.\notag
    \end{cases}
    \end{align}

\paragraph{Characterization of the burn-in.} 
The bound has been obtained by imposing that the first addendum in~\eqref{eq:bound_exp_noreg_2} satisfies 
\begin{align}
    4B^2\N{\infty}{\partial B(r)}{\frac{r}{\sqrt{\theta}}}\exp\left\{ - \frac{8T}{\theta^2 C(r) S}\right\} \leq \frac{\sigmaw}{T}.\notag
\end{align}
Leveraging~\cref{cor:thm312}, we deploy the values for the covering number and the hypercontractivity parameter and obtain that the number of samples $T$ has to satisfy
\begin{align}
    T \geq \frac{\theta^2 C_h S}{8}\left(\frac{1}{r}\right)^{\frac{4\dx}{2s-\dx}} \left[ C_M \left(\frac{1}{r} \right)^{\frac{2\dx}{2s-\dx}}\log\left(4B^2\left(1 + C_L\left(\frac{1}{r}\right)^{\frac{4s-\dx}{2s-\dx}}\right) \right) + \log\left(\frac{\sigmaw}{T}\right)\right]\notag
\end{align}
The effective condition for the burn-in is obtained by letting $r^2 = \sigmaw/T$.
\end{proof}

\section{Details on the numerical experiment}\label{sec:numerical_experiment}

We now fully specify the details of the experiment presented in~\cref{sec:discussion}.

\paragraph{Model set-up.}

We consider a nonlinear dynamical system that describes the dynamics of a unicycle robot. The ground-truth model is given by
\begin{align*}
\dot{x}_1(t) &= \nu(t) \cos \vartheta(t), \notag \\  \dot{x}_2(t) &= \nu(t) \sin \vartheta(t) \notag \\ \dot{\vartheta}(t) &= \omega(t),   \notag
\end{align*}
where $(x_1,x_2) \in \mathbb{R}^2$ is the position of the robot on the plane, $\vartheta \in [0,\pi/2]$ is the orientation angle, and $(\nu,\omega)$ are the translational and angular velocities, respectively.\\ 
We simulate the continuous dynamics using a forward Euler discretization with step size $dt = 0.05$, yielding the discrete-time dynamical system
\begin{align*}
x_{1,t+1} &= x_{1,t} + \nu_t \cos(\vartheta_t)dt, \\
x_{2,t+1} &= x_{2,t} + {\nu}_t \sin(\vartheta_t)dt, \\
\vartheta_{t+1} &= \vartheta_t + \omega_t dt.   
\end{align*}

We generate training pairs $\{(s_t,u_t),s_{t+1}\}$ where $s_t = (x_{1,t},x_{2,t},\vartheta_t)$ and $u_t = (\nu_t,\omega_t)$, corrupted by an additive Gaussian noise on $s_{t+1}$ with variance $\sigmaw = 1$.

\paragraph{Physics-informed regularization.} The unicycle kinematics enforce that the velocity has no lateral component. This constraint takes the form
\begin{align*}
q(s_t,u_t) = \left(x_{2,t+1}-x_{2,t}\right)\cos\vartheta_t -\left(x_{1,t+1}-x_{1,t}\right)\sin\vartheta_t = 0.    
\end{align*}

To promote models consistent with the physics, we penalize the squared $\Elltwo$-norm of this residual over a compact domain of states and inputs, i.e., 
\begin{align*}
 \|q\|_{\Elltwo(\X)}^2 \;=\; \int_{\X} q(\xi)^2 d\xi,
\text{ with }\, \xi=(s,u).   
\end{align*}

We approximate the above integral with Monte Carlo sampling from the input domain $\X$. For each mini-batch, we evaluate the residuals under adopted predictor $g_\theta(s_t,u_t)$ (which we specify below) and compute
\begin{align*}
 \widehat{\|q\|}_{\Elltwo(\X)}^2 
= \frac{ \Lebmeas(\X)}{N}\sum_{i=1}^N q(\xi_i)^2,
\;\text{ with } \xi_i \text { uniformly sampled from }\X.   
\end{align*}
The total loss combines the data mean squared error and the physics-informed penalty according to~\eqref{eq:RERM}. This ensures the learned predictor both fits noisy data and respects the no-slip constraint.

\paragraph{Adopted estimator.} We use a feedforward neural network $g_\theta \colon (s_t,u_t) \in \R^5 \to \hat{s}_{t+1} \in \R^3$ to approximate the discrete-time dynamics. The architecture is a two-hidden-layer multilayer perceptron (MLP) with ReLU activation function and 64 inner nodes.\\
We train the estimator using the Adam optimizer with learning rate $0.5\times 10^{-3}$, batch size $N = 300$. We vary the effective number of training samples $T$ over the range $T \in [300,10^6]$, where $T$ denotes the total number of samples, focusing on the behavior after the burn-in.

\paragraph{Results.} Figure~\ref{fig:unicycle_rates} 
on page~\pageref{fig:unicycle_rates} presents the log-log plot of the empirical excess risk (estimation error) as a function of the number of samples $T$, comparing models trained with and without physics-informed regularization. Each curve is obtained by averaging over $20$ independent random realizations of the training data, with solid lines indicating the mean estimation error and shaded regions denoting $95\%$ confidence intervals. Consistently with the results in~\cref{sec:convergence_rates} (especially~\cref{thm:main_rate_exp}), the estimator without physical knowledge presents a slower rate of convergence, dictated empirically by $\mathcal{O}(T^{-0.681})$, while the physics-informed one performs empirically as $\mathcal{O}(T^{-1.086})$.
\end{document}